\documentclass{article}

\usepackage{microtype}
\usepackage{graphicx}
\usepackage{subfigure}
\usepackage{booktabs} % for professional tables
\usepackage[dvipsnames]{xcolor}

\usepackage{hyperref}

\usepackage[accepted]{icml2024}

\usepackage{amsmath}
\usepackage{amssymb}
\usepackage{mathtools}
\usepackage{amsthm}

\usepackage[capitalize,noabbrev]{cleveref}

\theoremstyle{plain}
\newtheorem{theorem}{Theorem}[section]

\newtheorem{lemma}[theorem]{Lemma}

\theoremstyle{definition}
\newtheorem{definition}[theorem]{Definition}
\newtheorem{assumption}[theorem]{Assumption}
\theoremstyle{remark}
\newtheorem{remark}[theorem]{Remark}
\DeclareMathOperator{\proj}{proj}

\usepackage[textsize=tiny]{todonotes}

\usepackage{multirow}
\usepackage{tabularray}
\usepackage{adjustbox}
\usepackage{etoolbox}
\usepackage[toc,page,header]{appendix}
\usepackage{minitoc}
\usepackage{array, tabularx, caption, boldline}

\usepackage{footnote}
\usepackage{dsfont}
\makesavenoteenv{tabular}
\usepackage{dcolumn}
\usepackage{makecell}

\newtheorem*{statement*}{Statement}
\definecolor{mydarkblue}{rgb}{0,0.08,0.45}

\usepackage{thmtools}
\usepackage{thm-restate}

\allowdisplaybreaks
\newcommand{\calL}{\mathcal{L}}
\newcommand{\calM}{\mathcal{M}}
\newcommand{\calD}{\mathcal{D}}
\newcommand{\calG}{\mathcal{G}}

\newcommand{\calX}{\mathcal{X}}
\newcommand{\calY}{\mathcal{Y}}

\newcommand{\calV}{\mathcal{V}}
\newcommand{\calE}{\mathcal{E}}
\newcommand{\calF}{\mathcal{F}}

\newcommand{\calB}{\mathcal{B}}

\newcommand{\calQ}{\mathcal{Q}}
\newcommand{\expp}{\mathbb{E}}
\newcommand{\norm}[1]{\left\lVert#1\right\rVert}

\icmltitlerunning{Reward Model Learning vs. Direct Policy Optimization}

\begin{document}

\twocolumn[
\icmltitle{Reward Model Learning vs. Direct Policy Optimization: \\ A Comparative Analysis of Learning from Human Preferences}

\begin{icmlauthorlist}
\icmlauthor{Andi Nika}{yyy}
\icmlauthor{Debmalya Mandal}{comp}
\icmlauthor{Parameswaran Kamalaruban}{sch}
\icmlauthor{Georgios Tzannetos}{yyy}\\
\icmlauthor{Goran Radanović}{yyy}
\icmlauthor{Adish Singla}{yyy}
\end{icmlauthorlist}

\icmlaffiliation{yyy}{Max Planck Institute for Software Systems, Saarbrücken, Germany}
\icmlaffiliation{comp}{University of Warwick, Coventry, UK}
\icmlaffiliation{sch}{Independent Researcher, London, UK}

\icmlcorrespondingauthor{Andi Nika}{andinika@mpi-sws.org}

\icmlkeywords{Machine Learning, ICML, Reinforcement Learning from Human Feedback, Direct Preference Optimization}

\vskip 0.3in
]

\printAffiliationsAndNotice{}

\begin{abstract}
    In this paper, we take a step towards a deeper understanding of learning from human preferences by systematically comparing the paradigm of reinforcement learning from human feedback (RLHF) with the recently proposed paradigm of direct preference optimization (DPO). We focus our attention on the class of loglinear policy parametrization and linear reward functions. In order to compare the two paradigms, we first derive minimax statistical bounds on the suboptimality gap induced by both RLHF and DPO, assuming access to an oracle that exactly solves the optimization problems. We provide a detailed discussion on the relative comparison between the two paradigms, simultaneously taking into account the sample size, policy and reward class dimensions, and the regularization temperature. Moreover, we extend our analysis to the approximate optimization setting and derive exponentially decaying convergence rates for both RLHF and DPO. Next, we analyze the setting where the ground-truth reward is not realizable and find that, while RLHF incurs a constant additional error, DPO retains its asymptotically decaying gap by just tuning the temperature accordingly. Finally, we extend our comparison to the Markov decision process setting, where we generalize our results with exact optimization. To the best of our knowledge, we are the first to provide such a comparative analysis for RLHF and DPO.
\end{abstract}

\section{Introduction}\label{sec:introduction}

\looseness-1Learning from human preferences has grown more prominent as we move closer to artificial general intelligence. One of the most effective ways to learn from preferences is through reinforcement learning from human feedback (RLHF), which involves a two-step process of reward learning and regularized policy optimization. The attractiveness of this paradigm lies in its ability to model the reward function based solely on preference data. This makes it highly applicable in numerous practical situations where rewards are not given \textit{a priori} or are challenging to define accurately. Once the reward is modeled, RLHF solves a regularized value function maximization problem to obtain a fine-tuned policy. This paradigm has enjoyed a lot of applications varying from game-playing \cite{christiano2017deep, warnell2018deep, knox2008tamer, macglashan2017interactive}, robotics \cite{shin2023benchmarks, brown2019learning}, and training large language models (LLMs) \cite{ziegler2019fine, nakano2021webgpt, wu2021recursively, ouyang2022training, stiennon2020learning, glaese2022improving, ramamurthy2023is, menick2022teaching, ganguli2022red, bai2022training, gao2023scaling}. 

\begin{table*}[ht]
    \centering
    % \sisetup{table-number-alignment=center, 
    % separate-uncertainty=true}
    \renewcommand{\arraystretch}{2.5}
    \newcommand{\boundscolor}{mydarkblue}
    \footnotesize
    % \captionsetup{font=footnotesize}
    \scalebox{0.85}{
    \begin{tabular}{cccc} 
    \toprule
          & & \thead{RLHF} & \thead{DPO}\\
    \hline 
        \multirow{2}{*}{\parbox{1.5cm}{\centering Realizable\\ Rewards}}  & \parbox{1.5cm}{\centering Exact \\ Optimization} & $O\left(\beta D_\textnormal{KL}(\pi^\textnormal{opt}_{r^*}||\mu)\right) + \widetilde{\Theta}\left(\sqrt{\frac{d_R}{n}}\right)$  & $O\left(\beta D_\textnormal{KL}(\pi^\textnormal{opt}_{r^*}||\mu)\right) + \Theta\left(\frac{d_P}{\beta n}\right)$  \\
        & \parbox{1.5cm}{\centering Approximate \\ Optimization} & $O\left(\beta D_\textnormal{KL}(\pi^\textnormal{opt}_{r^*}||\mu)\right) + \widetilde{\Theta}\left(\sqrt{\frac{d_R}{n}}\right) +$ \textcolor{\boundscolor}{$O\left((1-\frac{1}{n})^t+\frac{e^{-t}}{\beta} \right)$} &  $O\left(\beta D_\textnormal{KL}(\pi^\textnormal{opt}_{r^*}||\mu)\right) + \Theta\left(\frac{d_P}{\beta n}\right)+$  \textcolor{\boundscolor}{$O\left(\frac{1}{\beta}\left(1-\frac{\beta}{n}\right)^t\right)$}\\
        \hline
    \parbox{2cm}{\centering Non-realizable \\ Rewards} & \parbox{1.5cm}{\centering Exact \\ Optimization} &$O\left(\beta D_\textnormal{KL}(\pi^\textnormal{opt}_{r^*}||\mu)\right) + \widetilde{\Theta}\left(\sqrt{\frac{d_R}{n}}\right) +$ \textcolor{\boundscolor}{$O(\epsilon_\textnormal{app})$} &  $O\left(\beta D_\textnormal{KL}(\pi^\textnormal{opt}_{r^*}||\mu) \right) + \Theta\left(\frac{d_P}{\beta n}\right) +$ \textcolor{\boundscolor}{$O\left(\beta D_\textnormal{KL}\left( \pi_{\theta^*}||\pi^*_{r^*}\right)\right)$} \\
        \bottomrule
    \end{tabular}}
    \caption{A presentation of the bounds on the suboptimality gap for RLHF and DPO. The first two rows present bounds under the realizable reward assumption in the exact and approximate optimization frameworks; the last row presents the bounds when the ground-truth reward function is not realizable. Here, $\pi^\textnormal{opt}_{r^*}$ denotes an optimal policy with respect to the ground-truth reward function $r^*$, $\pi^*_{r^*}$ denotes an optimal regularized policy, and $\pi_{\theta^*}$ denotes an optimal loglinear regularized policy. Moreover, $\beta$ denotes the regularization temperature, $D_\textnormal{KL}$ denotes the KL divergence, $d_R$ denotes the reward dimension, and $d_P$ denotes the policy dimension. Finally, $n$ denotes the sample size, $t$ denotes the optimization steps for the approximate setting, $\epsilon_\textnormal{app}$ denotes the reward mismatch coefficient and  $\widetilde{\Theta}$ hides any $\log$ factors.
    }
    \label{fig:table}
\end{table*}

As an alternative to RLHF, \citet{rafailov2023direct} have recently proposed direct preference optimization (DPO), an RL-free paradigm to learning from preferences. DPO circumvents the reward modeling phase and directly optimizes the policy parameters based on the preference data. In certain LLM instances, DPO seems to be empirically superior to RLHF, due to its simple optimization framework. 

That said, a statistical analysis of the differences between these paradigms is lacking. The sample complexity of RLHF in various settings has already been studied \cite{zhu2023principled, zhan2023provable, xiong2023iterative}, and there have been some initial attempts at theoretically understanding DPO and its variants \cite{gheshlaghi2023general}. However, it is unclear when one of the paradigms is better and when these two paradigms are statistically comparable. Motivated by this observation, we initiate a thorough discussion of the theoretical comparison between RLHF and DPO. Specifically, the purpose of this paper is to address the following research questions:

\textit{What are the statistical guarantees of RLHF relative to those of DPO? What conditions benefit one as opposed to the other?}

As DPO does not learn a reward model, but directly optimizes over the policy space, a dependence on the policy dimensionality $d_P$ is expected. On the other hand, RLHF's performance evidently implies some dependence on reward dimensionality $d_R$ due to its reward learning phase. Does this imply a discrepancy in the statistical bounds of these paradigms when the reward and policy dimensions are different? Moreover, what can be said about the dependency of the bounds on the sample size $n$ or the regularization temperature $\beta$?

We address these questions in the following setting: finite spaces, Bradley-Terry preference model, linear rewards and loglinear policies. We first study the exact optimization setting and derive bounds on both RLHF and DPO. Then, we proceed to derive fast convergence rates for a modified version of the policy gradient for RLHF, and gradient descent for DPO.  Next, we discuss some implications of our bounds when the reward function is not fully realizable. We close our paper by extending our comparative analysis to deterministic Markov decision processes. Our contributions are summarized below; see Table~\ref{fig:table} for explicit bounds.
\begin{itemize}\setlength\itemsep{0.001em}
    \item First, we derive minimax bounds on the suboptimality gap induced by RLHF and DPO in the exact optimization setting by leveraging smoothness and strong convexity properties. We show that when the optimal regularized policy is loglinear and the reward function is linear, RLHF is $\widetilde{\Theta}(\sqrt{d_R/n})$-close to its objective, while DPO is $\widetilde{\Theta}(d_P/(\beta n))$-close. These results emphasize the comparison of the two paradigms in terms of the reward and policy dimensions when setting $\beta = {\Theta}(\sqrt{d_P/n})$ for DPO. 
    \item Furthermore, we study the convergence rates of a version of the natural policy gradient for RLHF and gradient descent for DPO. Motivated by recent fast convergence results for entropy regularized RL with tabular softmax policies, we derive $O\left(e^{-t}/\beta\right)$ convergence rates in $t$ iterations for a version of the natural policy gradient for RLHF.  Moreover, for gradient descent we are able to show $O\left((1/\beta)(1-\beta/n)^t\right)$ convergence rates by using the fact that the DPO loss function satisfies the \textit{PL condition} \cite{karimi2016linear}. These results replicate the implications from the exact optimization setting on the difference in terms of reward and policy dimensions. 
    \item \looseness-1We also consider the case where the ground-truth reward function is not realizable and its best linear fit is $\epsilon_\textnormal{app}$-close to it. We show that, while RLHF incurs an additional constant term on the suboptimality gap, DPO's dependence on the additional term can be controlled by setting the regularization temperature accordingly. 
    \item \looseness-1Finally, we extend our comparison to deterministic Markov decision processes by proposing a new formulation of the DPO objective for this setting and then generalizing our results. The main motivation for this extension is that, arguably, the difference in reward and policy dimensions in this setting is higher.
\end{itemize}

\section{Related Work}

\paragraph{Learning from pairwise comparisons.} \looseness-1In the context of RL, the problem of learning from pairwise comparisons has been studied thoroughly in the bandit setting, where the problem is known as the \textit{dueling bandit} problem \cite{yue2009the, faury2020improved, ailon2014reducing, gayane2015a, komiyama2015regret, zoghi2014relative, saha2019active, saha2022efficient}. For the case of dueling RL for linear MDPs, \citet{saha2023dueling} propose an algorithm that satisfies tight regret guarantees, while \citet{chen2022human} extend this formulation to the MDPs with general function approximation. Finally, \citet{chatterji2021theory} consider a more general setting in which the trajectory-based feedback is generated from a generalized linear model, and they propose variants of optimistic algorithms for online RL. 

In this paper, we consider the offline setting, where \citet{zhu2023principled} and \citet{zhan2023provable} have already provided statistical bounds on pessimistic RLHF for direct value maximization. Our focus, however, is on the regularized value maximization problem. While pessimism mitigates poor coverage for the setting considered by \citet{zhu2023principled} and \citet{zhan2023provable}, for regularized RLHF and, as a consequence, for DPO, this issue remains and is captured by the coverage coefficients which we define with respect to both reward and policy features. The aim of this paper, however, is not to mitigate these issues, but to thoroughly analyse the statistical guarantees of the existing prominent paradigms of learning from human preferences and derive insightful results that shed light on their relative performances.

\paragraph{Direct preference optimization.} In recent works, the RL-free fine-tuning paradigm of direct preference optimization (DPO) has gained popularity \cite{rafailov2023direct, an2023direct, gheshlaghi2023general, wang2023beyond}. Its original formulation was proposed for the contextual bandit setting. \citet{hejna2023contrastive} propose an extension of DPO to MDPs under the assumption that the preferences depend on the advantage function of the optimal policy. While we also provide an extension of the DPO formulation for MDPs in Section \ref{sec:extension_to_mdps}, our primary focus is on a comparative analysis between the two paradigms in the contextual bandit setting.

\paragraph{Offline RL.} In recent years, there has been a significant surge in interest towards offline RL, with an extensive literature both in the empirical front \citep{jaques2019way, laroche2019safe, fujimoto2019off, kumar2020conservative, agarwal2020optimistic, kidambi2020morel} and the theoretical one \citep{jin2021pessimism, xie2021bellman, rashidinejad2021bridging, uehara2021pessimistic, zanette2021provable}. While the focus of this line of work is on the traditional reward-based offline RL, our problem is derived from a combination of reward-learning from pairwise feedback and KL-regularized offline RL based on it.

\section{Formal Setting}\label{sec:preliminaries}

This section presents the background material that will be used throughout the paper. We will use a notation similar to \cite{rafailov2023direct} and \cite{gheshlaghi2023general}.

\paragraph{Notation.} Let $\langle u,v\rangle = u^\top v$ denote the inner product between vectors $u$ and $v$. The trace of a matrix $A$ is denoted by $tr(A)$ and its pseudo-inverse by $A^\dagger$. Moreover, $\Delta(\calX)$ denotes the set of distributions over the finite set $\calX$ and $\norm{v}_M = \sqrt{v^\top Mv}$ denotes the seminorm of vector $v$ with respect to $M$. Finally, $\proj_A (v)$ denotes the projection of vector $v$ onto set $A$ and $\widetilde{\Theta}(\cdot)$ hides any log factors.

\subsection{Preliminaries}\label{sec:preliminaries.pre}

Let $\calX$ be a finite set of contexts with cardinality $X$ and $\calY$ be a finite set of actions with cardinality $Y$. Fix $\rho\in\Delta(\calX)$ as an initial distribution over contexts and let $r:\calX\times\calY\rightarrow [0,1]$ be a reward function. We consider a class of linear reward functions defined below. 

\begin{definition}[\textit{Linear reward function class}]\label{def:linear_reward} 
Let $\phi$ be a $d_R$-dimensional feature mapping with $\max_{x,y}\norm{\phi(x,y)}_2\leq 1$ and let $F>0$. We consider the following class of linear reward functions:
    \begin{align*}
        & \calF = \Big\{ r_\omega\in[0,1]^{XY} : r_\omega (x,y)= \omega^\top\phi(x,y), \nonumber \\ & \forall (x,y)\in\calX\times\calY\; \text{where}\; \omega\in\mathbb{R}^{d_R}\; \text{and}\; \norm{\omega}_2\leq F\Big\}~. \nonumber
    \end{align*} 
\end{definition}
Given $x\in\calX$, a policy $\pi(\cdot|x)\in\Delta(\calY)$ is a distribution over actions. Throughout the paper, we will consider the loglinear class of policies, defined as follows.
\begin{definition}[\textit{Loglinear policy class}]\label{def:loglinear_policies}
    Let $\psi$ be a $d_P$-dimensional feature mapping with $\max_{x,y}\norm{\psi(x,y)}_2\leq 1$ and let $B>0$. We consider the following class of loglinear policies:
    \begin{align*}
        \Pi & = \Big\{ \pi_\theta: \pi_\theta(y|x) = \frac{\exp(\theta^\top \psi(x,y))}{\sum_{y'\in\mathcal{Y}} \exp(\theta^\top \psi(x,y'))},\\ &  \forall (x,y)\in\calX\times\calY\;\text{where}\; \theta \in\mathbb{R}^{d_P}\;\text{and}\; \norm{\theta}_2\leq B\Big\}~.
    \end{align*}
\end{definition}
Given policy $\pi$, the value function of $\pi$ with respect to reward function $r$ and context distribution $\rho$ is defined as $$V^\pi_r(\rho) = \sum_x\rho(x)\sum_{y}\pi(y|x)r(x,y).$$

\subsection{Offline Learning from Human Preferences}\label{sec:preliminaries.offline}
Let $\mu$ be a reference policy fixed throughout the paper, and $\beta >0$ be a regularization parameter. Let us define the KL-regularized objective as
$$\calV^\pi_r(\rho) =V^\pi_r(\rho) - \beta D_\textnormal{KL}\left(\pi ||\mu\right),$$
where $D_\textnormal{KL}\left(\pi ||\mu\right) = \sum_x\rho(x) \sum_y \pi(y|x) \log\frac{\pi(y|x)}{\mu(y|x)}$. 

We assume access to the dataset $\calD_n= \{ (x_i,y^w_i,y^l_i)\}^n_{i=1}$, where $y^w_i$ denotes the preferred action over $y^l_i$. In this paper, we will assume that the distribution of human preferences follows the Bradley-Terry (BT) model \cite{bradley1952rank}, which we formally state below.
\begin{definition}[\textit{Bradley-Terry preference model}]\label{def:btl_model}
    There exists a latent reward function $r^*$ and a probability law $P^*$ such that, for every tuple $(x,y^w,y^l)$, we have
    \begin{align*}
        P^*(y^w \succ y^l | x) = \frac{\exp\left( r^*(x,y^w)\right)}{\exp\left( r^*(x,y^w)\right) + \exp\left( r^*(x,y^l)\right)}~,
    \end{align*}
    where $y^w \succ y^l$ denotes $y^w$ being preferred over $y^l$.
\end{definition}
The latent reward function $r^*$ will be fixed throughout the paper as the ground-truth reward function.

\subsection{Reinforcement Learning from Human Feedback}\label{sec:preliminaries.RLHF}

We consider the reinforcement learning from human feedback (RLHF) paradigm as formulated in \cite{ziegler2019fine}. Having access to preference dataset $\calD_n$ and a fixed reference policy $\mu$, RLHF proceeds in two phases: the reward learning phase and the final KL-regularized reinforcement learning phase. 

For the reward learning phase, RLHF estimates $r^*$ by applying maximum likelihood estimation (MLE) to the dataset $\calD_n$. The MLE optimization problem can be written as
\begin{align}\label{op_mle_rlhf}
     \min_r \calL^r_\textnormal{RLHF}&({\calD_n}) := - \expp_{(x,y^w,y^l) \sim \calD_n} \nonumber \\ & \quad \Big[ \log \Big( \sigma \big( r(x^w,y^w)-r(x^l,y^l)\big)\Big)  \Big]~,\tag{P1.1}
\end{align}
where $\sigma(z)=1/(1+\exp(-z))$ denotes the sigmoid function.
Let $\widehat{r}$ denote the solution of Problem \eqref{op_mle_rlhf}. 

The final phase of RLHF consists of maximizing the KL-regularized objective with respect to $\widehat{r}$ by solving
\begin{align}\label{op_kl_regularized}
    \max_{\pi} \calV^{\pi}_{\widehat{r}} (\calD_n):=  \expp_{\substack{x\sim\calD_n \\ y\sim\pi(\cdot|x)}}\left[\widehat{r}(x,y)-\beta\log\frac{\pi(y|x)}{\mu(y|x)} \right]\tag{P1.2}~.
\end{align}

\subsection{Direct Preference Optimization}\label{sec:preliminaries.DPO}

\looseness-1Recently, alternative paradigms to RLHF have been studied. In particular, \citet{rafailov2023direct} introduced direct preference optimization (DPO), a new fine-tuning paradigm that directly optimizes the policy parameters instead of going through the reward modeling phase. Their key observation is that the latent reward can be expressed in terms of its optimal policy and the reference policy. This yields a loss function that is directly defined in terms of the preference data.

Formally, \citet{rafailov2023direct} show that there exists a policy $\pi$ that maximizes the KL-regularized objective, for which we have
\begin{align}\label{eq:reward_to_policy_mapping}
    r^*(x,y) = \beta\log\frac{\pi(y|x)}{\mu(y|x)} + \beta \log Z(x)
\end{align}
for every $(x,y)$, where $Z(x)$ denotes the partition function $\sum_y \mu(y|x)\exp(r^*(x,y)/\beta)$. A new objective is then derived, which directly depends on the policy. Given preference dataset $\calD_n$, this objective leads to the following optimization problem:
\begin{align}\label{op_dpo_obj}
      & \min_\pi \calL^\pi_\textnormal{DPO}({\calD_n}):= -\expp_{(x,y^w,y^l)\sim \calD_n} \nonumber \\ & \left[ \log \left( \sigma\Big( \beta \log \frac{\pi(y^w|x)}{\mu(y^w|x))} - \beta \log \frac{\pi(y^l|x)}{\mu(y^l|x))} \Big)\right)\right]~.\tag{P2}
\end{align}
As it turns out, this elegant approach yields practical benefits. However, it is unclear whether these benefits can be theoretically justified. In this paper, we will provide a comparative analysis of both RLHF and DPO in different settings. Next, we define a unified metric of performance to compare these paradigms. 

\subsection{Performance Metric}\label{sec:preliminaries.metric}

Given a reward function $r$, let $\pi^\textnormal{opt}_r\in \arg\max_\pi V^\pi_r(\rho)$ denote an optimal policy for $V^\pi_r(\rho)$ and $V^\textnormal{opt}_r(\rho)$ the optimal value. For a given policy $\pi$, we define the suboptimality gap of $\pi$ as 
\begin{align*}
    G(\pi) = V^\textnormal{opt}_{r^*}(\rho) - V^\pi_{r^*} (\rho)~.
\end{align*}
$G(\pi)$ captures how well a policy is performing w.r.t. the ground-truth reward function $r^*$. In this paper, we will use the suboptimality gap $G$ as our unified measure of performance when comparing the two paradigms.

We note that RLHF and DPO are designed to minimize regularlized objectives $\calV$ instead of optimizing the value function $V$ (see Sections~\ref{sec:preliminaries.RLHF}~and~\ref{sec:preliminaries.DPO}). In order to rigorously analyze the differences between RLHF and DPO, we need to establish some additional notation.  Let $\pi^*_r\in\arg\max_\pi\calV^\pi_r(\rho)$ denote a regularized optimal policy with respect to $r$ and $\calV^*_r(\rho)$ denote the optimal regularized value. Analogous to $G$, we define the regularized suboptimality gap of $\pi$  as 
\begin{align*}
     &  \calG(\pi)= \calV^*_{r^*}(\rho) - \calV^\pi_{r^*}(\rho) = \\
    & (V^{\pi^*_{r^*}}_{r^*}(\rho)-V^\pi_{r^*}(\rho))-\beta( D_{KL}(\pi^*_{r^*}||\mu)- D_{KL}(\pi||\mu))~.
\end{align*}
As RLHF and DPO are designed to minimize $\calG(\pi)$, both will incur an additional term on their bounds when comparing their bounds w.r.t. $G(\pi)$. For this reason, we will formally define this discrepancy term as $D(\pi) = G(\pi) - \calG(\pi)$ and discuss it further in Section \ref{sec:rlhf_betterthan_dpo}. Next, we proceed to provide a comparative analysis for these paradigms, starting with the exact optimization setting when the ground-truth reward is realizable. We note that any $\log$ terms are omitted for clarity of presentation. All proofs and related discussions can be found in the Appendix.

\section{Realizable Rewards: Exact Optimization}\label{sec:rlhf_betterthan_dpo}

In this section, we analyze the statistical differences in performance between RLHF and DPO in the exact optimization setting. We assume throughout the section that the ground-truth reward function is linear and realizable, i.e. $r^*\in \calF$. Moreover, we assume a loglinear regularized optimal policy exists, i.e. $\pi^*_{r^*}\in\Pi$. Note that, for linear reward function, $\calL^r_\textnormal{RLHF}(\calD_n)$ can be equivalently written as
\begin{align}
    -\expp_{(x,y^w,y^l)\sim{\calD_n}} \bigg[ \log\sigma \left( \omega^\top\left(\phi(x,y^w)-\phi(x,y^l)\right)\right) \bigg]~. \label{eq:linear_mle}
\end{align}
Moreover, for loglinear policies, $\calL^\pi_\textnormal{DPO}(\calD_n)$ can be equivalently written as
\begin{align}
    - & \expp_{(x,y^w,y^l)\sim{\calD_n}} \label{eq:loglinear_dpo}  \\ &\bigg[ \log\left(\sigma \left( \beta \theta^\top\left(\psi(x,y^w)-\psi(x,y^l)\right) - J(x,y^w,y^l)\right)\right) \bigg]~,\nonumber
\end{align}
where $J(x,y^w,y^l)=\log (\mu(y^w|x)/\mu(y^l|x))$. 

We will denote the losses for the RLHF reward learning phase and DPO as $\calL^\omega_\textnormal{RLHF}(\calD_n)$ and $\calL^\theta_\textnormal{DPO}(\calD_n)$, respectively. Let $r_{\widehat{\omega}}$ denote the reward estimate and let $\pi_{\widehat{\theta}}$ denote the policy learned by RLHF. Moreover, let $\pi_{\widetilde{\theta}}$ denote the policy learned by DPO. Formally, we assume that RLHF has access to an oracle that exactly solves both optimization problems and returns $r_{\widehat{\omega}}\in\arg\min\calL^\omega_{RLHF}(\calD_n)$ and $\pi_{\widehat{\theta}}\in\arg\max_\theta
\calV^{\pi_\theta}_{r_{\widehat{\omega}}}(\calD_n)$.  Similarly, for DPO we assume that the oracle returns $\pi_{\widetilde{\theta}}\in\arg\min_\theta \calL^\theta_\textnormal{DPO}({\calD_n})$.

\subsection{Theoretical Results}

\looseness-1Before stating our main results of this section, we will need to define the covering numbers with respect to ${\calD_n}$. Let the sample covariance matrix with respect to the reward features be defined as
\begin{align*}
    \Sigma_{\calD_n, R} = \frac{1}{n}\sum_{(x,y^w,y^l)\in\calD_n} \overline{\phi}(x,y^w,y^l)\overline{\phi}(x,y^w,y^l)^\top~,
\end{align*}
where $\overline{\phi}(x,y^w,y^l) = \left(\phi(x,y^w)-\phi(x,y^l)\right)$, for every $(x,y^w,y^l)\in\calD_n$. 
Fix $\lambda >0$ and let $\Lambda_R=\norm{(\Sigma_{\calD_n,R}+\lambda I)^{-1/2}}_2$ be the reward covering number.

Similarly, let the sample covariance matrix of the policy features be defined as 
\begin{align*}
    \Sigma_{\calD_n, P} = \frac{1}{n}\sum_{(x,y^w,y^l)\in\calD_n}\overline{\psi}(x,y^w,y^l)\overline{\psi}(x,y^w,y^l)^\top~,
\end{align*}
where $\overline{\psi}(x,y^w,y^l) = \left(  \psi\left(x,y^w\right)-\psi\left(x,y^l\right)\right)$,
for every $(x,y^w,y^l)\in\calD_n$. The policy covering number of ${\calD_n}$ is $\Lambda_P = \norm{\left(\Sigma_{\calD_n,P}+\lambda I\right)^{-1/2}}_2$.
Our bounds will depend on these quantities. We will start with minimax bounds on the suboptimality for RLHF in the above-mentioned setting.
\begin{restatable}{retheorem}{main_rlhf_result}\label{thm:main_rlhf_result}
    Let $\delta >0$. Assume that $r^*\in\calF$. Then, with probability at least $1-\delta$, the suboptimality gap incurred by RLHF is
    \begin{align*}
        G \left(\pi_{\widehat{\theta}}\right)
        & = \Theta\left( \Lambda_R \sqrt{\frac{d_R+\log(6/\delta)}{S^2_R n}+\lambda F^2}\right) +  D\left( \pi_{\widehat{\theta}}\right) ,
    \end{align*}
    where $S_R =1/\left(2+\exp\left( -2F\right)+\exp\left(2F\right)\right)$.
\end{restatable}
\textit{Proof sketch.} The proof follows from splitting the gap into sub-gaps which we can bound directly, and results from \citep{zhu2023principled}.\qed

\looseness-1Next, we will consider the suboptimality gap induced by DPO. First, note that Equation \eqref{eq:loglinear_dpo} for loglinear policies essentially becomes a logistic regression problem in $d_P+1$ dimensions, by adding a dummy variable to $\theta$ that corresponds to $J(x,y^w,y^l)$. Our goal is to use loglinearity to derive smoothness properties for logistic regression so that we can obtain minimax bounds -- without this assumption on the policy class, $\calL^\theta_\textnormal{DPO}(\calD_n)$ may not satisfy these properties.\footnote{Lemma \ref{lem:tabular_dpo} shows this loss is not smooth for tabular settings.}

Moreover, note that Equation \eqref{eq:reward_to_policy_mapping} relates $r^*$ to one of the regularized optimal policies $\pi^*_{r^*}$. Obviously, this does imply that $\pi^*_{r^*}$ is loglinear. Nevertheless, the following lemma states that that is the case for linear rewards. Let $\Phi\in\mathbb{R}^{d_R\times XY}=[\phi(x,y)]_{(x,y)\in\calX\times\calY}$ and $\Psi\in\mathbb{R}^{d_P\times XY}=[\psi(x,y)]_{(x,y)\in\calX\times\calY}$ denote the reward feature and policy feature matrices, respectively. 
\begin{restatable}{relemma}{lemmaoptimalisloglinear}
\label{lem:optimal_is_loglinear}
    Assume that $r^*\in\calF$ and $\mu\in\Pi$. Furthermore, assume that the column space of $\Phi$ is a subspace of the column space of $\Psi$. Then, there exists $\theta^*\in\Theta$, such that $\pi_{\theta^*}\in\arg\max_\pi\calV^{\pi}_{r^*}(\rho)$ and $r^*(x,y) = \beta\log (\pi_{\theta^*}(y|x) /\mu(y|x))+\beta\log Z(x)$.
\end{restatable}
With these observations in place, we are now ready to state the minimax bounds on the suboptimality for DPO.
\begin{restatable}{retheorem}{maindporesult}\label{thm:main_dpo_theorem}
    Let $\delta > 0$ and $\beta >0$. Assume that $r^*\in\calF$, $\mu\in\Pi$, and that the condition of Lemma \ref{lem:optimal_is_loglinear} is satisfied. Let $n\geq O\left(tr(\Sigma_{{\calD_n,P}}^\dagger)/(\beta B^2)\right)$. Then, with probability at least $1-\delta$, the suboptimality gap of DPO is
    \begin{align*}
        G \left(\pi_{\widetilde{\theta}}\right) & =  D\left(\pi_{\widetilde{\theta}}\right) + {\Theta}\left( \frac{\Lambda_P (d_P+1)}{\beta  n} + \beta\lambda\Lambda_P B^2\right)~.
    \end{align*}
\end{restatable}
\textit{Proof sketch.} We start by splitting the suboptimality gap and focus on the $\calG(\pi_{\widetilde{\theta}})$ term. Here, we utilize the expression for the ground-truth reward as  $r^*(x,y)=\beta\log\pi_{\theta^*}(y|x) - \beta\log\mu(y|x)+Z_{\widetilde{\theta}}(x)$, where $Z_{\widetilde{\theta}}(x)$ denotes the partition function with respect to $\pi_{\widetilde{\theta}}$. Then, using the fact that the policies are loglinear we further expand and reduce the whole gap in terms of the log differences. Next, we utilize the smoothness and Lipschitzness of the log-sum-exp function to finally obtain the upper bounds. For the lower bound, we construct an example where the policy feature matrix $\Psi$ is full rank, and show that the log-sum-exp function becomes strongly convex. This finally leads to the stated bounds.\qed

Before discussing the implications of our results, let us say a few words on the regularization gap $D(\pi)$ for RLHF and DPO. Let $\pi_{\theta^*}$ denote an optimal loglinear regularized policy. A characterization of $D(\pi)$ is given as follows. 
\begin{restatable}{relemma}{lemmaaboutd}
    For any $\theta$, we have that $\beta D_\textnormal{KL}\left(\pi_{\theta^*}||\mu\right) - \beta D_\textnormal{KL}  \left(\pi_{\theta}||\mu\right) \leq D(\pi_\theta)$ and $ D(\pi_\theta) \leq \beta D_\textnormal{KL}\left(\pi^\textnormal{opt}_{r^*}||\mu\right)- \beta D_\textnormal{KL} \left(\pi_{\theta}||\mu\right)$.
 \end{restatable}
Furthermore, for DPO, this quantity can be upper bounded in terms of $\pi^\textnormal{opt}_{r^*}$ and $\pi_{\theta^*}$ as follows.
\begin{restatable}{relemma}{lemmaaboutd2}\label{lem:kl_difference_bound}
    Given $\delta >0$, with probability at least $1-\delta$, we have $
        D(\pi_{\widetilde{\theta}}) \leq  \beta D_\textnormal{KL}\left(\pi^\textnormal{opt}_{r^*}||\mu\right) - \beta D_\textnormal{KL}\left(\pi_{\theta^*}||\mu\right)  + \widetilde{O}\left( d_P/n^{3/2}\right)$~.
\end{restatable}
\looseness-1In general, it is known that the KL divergence may not be upper-bounded. However, assuming that optimal policy $\pi^\textnormal{opt}_{r^*}$, optimal regularized policy $\pi_{\theta^*}$, and sampling policy $\mu$ are not far away from each other, then these $D_\textnormal{KL}(\cdot)$ quantities would not be too large.

\subsection{Comparative Analysis}\label{sec:comparative_analysis_exact}

In this section, we will provide some insights into the implications of our theoretical results. For the purpose of this section, we will focus our attention on the problem-dependent parameters and ignore the quantity $D(\pi)$.

\looseness-1\textbf{The role of dimensionality.} Note that RLHF has $\widetilde{\Theta}( \sqrt{d_R})$ dependence on the reward dimension, while DPO has $\Theta(d_P)$ dependence on the policy dimension. When $d_R = d_P$ and the sample size is small, RLHF seems to statistically outperform DPO. Any setting where $d_R \ll d_P$  makes this difference more apparent. In Section \ref{sec:extension_to_mdps}, we will discuss an extension of our analysis to a setting where the reward dimension can be much smaller than the policy dimension in practice.

\textbf{The role of sample size.} Next, we take into consideration the sample size. Note that DPO's bounds depend on $n$ being large enough (cf. Theorem \ref{thm:main_dpo_theorem}). Assume everything else constant and $d=d_R=d_P$. If the $D(\pi)$ terms are similar for both paradigms, then, for large sample sizes such that $n \gg d$, DPO seems to outperform RLHF asymptotically. Whenever $n < d$ (which is usually the case for large language models), RLHF has a smaller suboptimality gap. 

\textbf{The role of $\beta$.} Finally, we discuss the role of the temperature $\beta$ on the bounds. First, note that RLHF can effectively set $\beta =0$ to annihilate the effect that $D(\pi)$ has on its bounds. On the other hand, DPO cannot set $\beta$ to $0$ due to a disproportional dependence of its bounds on it. Thus, the optimal choice of $\beta$ for DPO is $\beta=\Theta(\sqrt{d_P/n})$, yielding $\Theta(\sqrt{d_P/n})$ bounds and matching the order of $n$ in the bounds of RLHF. For such a value of $\beta$, the same implications hold --  the main difference between both settings is in terms of the differences between the reward and policy parameter dimensions. 

\section{Realizable Rewards: Approximate Optimization}\label{sec:approximate_optimization}

In this section, we shift our focus to the approximate setting, where access to oracles is not given. Here, both paradigms have to approximately solve their estimation problems based on the given data. Similar to the previous section, we assume throughout this section that the ground-truth reward function $r^*$ is linear and realizable in $\calF$, and that there exists a loglinear regularized policy $\pi^*_{r^*}\in\Pi$. Moreover, we assume that, for every data tuple $(x,y^w,y^l)\in\calD_n$, we have $\phi(x,y^w)\neq\phi(x,y^l)$ and $\psi(x,y^w)\neq\psi(x,y^l)$.

\subsection{Theoretical Results}

Let us start with the reward learning phase. Recall the definition of the loss $\calL^\omega_\textnormal{RLHF}(\calD_n)$ for MLE, as defined in Section \ref{sec:rlhf_betterthan_dpo}. Let $\omega_0$ be initialized randomly, and let
\begin{align}\label{eq:gradient_descent_mle}
    \omega_{t+1} = \underset{\omega:\norm{\omega}_2\leq F}{\proj} \left(\omega_t - \eta \nabla_\omega \calL^\omega_\textnormal{RLHF}(\calD_n)\right)~,
\end{align}
for any iterate $t\geq 0$, where $\eta$ denotes the learning rate. Let $\omega^*_{\calD_n}\in\arg\max\calL^\omega_\textnormal{RLHF}(\calD_n)$.
The first result of this section provides fast convergence rates of gradient descent for the reward learning phase of RLHF. 
\begin{restatable}{retheorem}{rlhfmleconvergence}\label{thm:rlhf_mle_convergence}
    For every $t\geq 0$, the gradient descent procedure \eqref{eq:gradient_descent_mle} with learning rate $\eta=1/\exp(2F)$ satisfies
    \begin{align*}
        \norm{\omega_t -\omega^*_{\calD_n}}^2_{\Sigma_{\calD_n,R}} \leq O\left( 1-\frac{1}{n}\right)^t~.
    \end{align*}
\end{restatable}
\textit{Proof sketch.} We begin by showing Lipschitzness and smoothness of $\calL^\omega_\textnormal{RLHF}(\calD_n)$ with respect to $\omega$. Then, we show that the PL condition \citep{karimi2016linear} is enough to guarantee fast convergence of projected gradient descent by showing that such a condition implies that  $\calL^\omega_\textnormal{RLHF}(\calD_n)$ also satisfies the \textit{proximal} PL condition \citep{karimi2016linear} when its domain is restricted to a ball. The result follows by applying the convexity of $\calL^\omega_\textnormal{RLHF}(\calD_n)$.\qed

\looseness-1Next, we discuss the policy optimization phase of RLHF. Let $r_{\widehat{\omega}}$ denote the reward estimated from the previous phase. Initialize $\theta_0\in\mathbb{R}^{d_P}$ with $\norm{\theta_0}_2\leq B$. For any $t\geq 0$, let
\begin{align}\label{eq:natural_policy_gradient}
    \theta_{t+1}=\theta_t + \eta' \left(\Psi_n\Psi^\top_n\right)^{\dagger}\nabla_\theta \calV^{\pi_\theta}_{r_{\widehat{\omega}}}(\calD_n)~,
\end{align}
where $\eta' >0$ is the learning rate and $\Psi_n=[\psi(x,y)]_{x\in\calD_n,y\in\calY}$ denotes the sample feature matrix.\footnote{We only need $\norm{\theta_0}_2 \leq B$ for RLHF, since its bounds do not depend on $\norm{\theta_t}$, for $t\geq 1$. Thus, we do not need projection.} Then, the following result holds.
\begin{restatable}{retheorem}{rlhfpgconvergence}\label{thm:approx_rlhf}
    Let $\delta >0$. Assume that $\Psi_n$ has full column rank. Then, with probability at least $1-\delta$, for every $t\geq 1$, update rule \eqref{eq:natural_policy_gradient} with learning rate $\eta' \leq n/\beta$ satisfies
    \begin{align*}
        \calV^{*}_{r_{\widehat{\omega}}}\left(\calD_n\right) - \calV^{\pi_{\theta_t}}_{r_{\widehat{\omega}}}\left(\calD_n\right) \leq O\left(\frac{1}{\beta}\exp\left(-(t-1)\right)\right)~.
    \end{align*}
\end{restatable}
\textit{Proof sketch.} After deriving a naive gradient update rule in matrix notation, we examine the conditions needed to obtain fast convergence rates for our setting. As a consequence, we design our gradient update which resembles natural policy gradient for loglinear policies -- the matrix $(\Psi_n\Psi_n^\top)^\dagger$ captures the gradient information for this case.  Then, we use similar techniques to those in \citep{mei2020on} and use the fact that $\Psi_n$ is full rank, to finally obtain the desired bounds.  \qed
\begin{remark}
    It is important to emphasize that the assumption on $\Psi_n$ is not restrictive. Indeed, given a preference dataset, it is always possible to construct alternative feature representations that satisfy the assumption with respect to the given data (e.g. different LLMs use different encoding methods while being fine-tuned using the same data).
\end{remark}
\begin{remark}
    Before going to our next result, it is important to clarify the double usage of $\calD_n$ for both the reward learning and policy optimization phases. Our theoretical guarantees are based on the data being independently generated in these phases. Thus, the standard approach is to split the data into two batches for both purposes. Note that both batches would still be $O(n)$ in size and the dependence of the results on $n$ would not change. We use the same $\calD_n$ for both phases for simplicity of presentation.
\end{remark}
\looseness-1Next, we provide convergence results of gradient descent for DPO with loglinear policies. Let $\theta_0$ be initialized randomly, and let
\begin{align}\label{eq:gradient_descent_dpo}
    \theta_{t+1} = \underset{\theta:\norm{\theta}\leq B}{\proj}\left(\theta_t - \eta'' \nabla_\theta\calL^\theta_\textnormal{DPO}(\calD_n)\right)~,
\end{align}
for any iterate $t\geq 0$, where $\eta''$ denotes the learning rate. Let $\theta^*_{\calD_n}\in\arg\max\calL^\theta_\textnormal{DPO}(\calD_n)$. Then, we have the following result.
\begin{restatable}{retheorem}{dpomleconvergence}
    For every $t\geq 0$, the gradient descent procedure \eqref{eq:gradient_descent_dpo} with learning rate $\eta'' = O\left(1/\beta^2\right)$ satisfies
    \begin{align*}
        \norm{\theta_t-\theta^*_{\calD_n}}^2_{\Sigma_{\calD_n,P}} & \leq O\left( \frac{1}{\beta}\left(1-\frac{\beta}{n}\right)^t\right)~.
    \end{align*}
\end{restatable}
We have omitted the dependence on the absolute constants that are irrelevant to our discussion --  see Appendix \ref{sec:dpo_convergence} for a detailed expression of the hidden constants.

\subsection{Comparative Analysis}\label{sec:comparative_analysis_approx}

The regularized suboptimality gap for RLHF is $$\calG(\pi_{\widehat{\theta}})  \leq \Theta\left( \sqrt{\frac{d_R}{n}}\right) + O\left( \left( 1-\frac{1}{n}\right)^t +  \frac{\exp(-t)}{\beta}\right)$$
and the regularized suboptimality gap for DPO is
$$\calG\left(\pi_{\widetilde{\theta}}\right)  \leq \Theta\left( \frac{d_P}{\beta n}\right) + O\left( \frac{1}{\beta}\left(1-\frac{\beta}{n}\right)^t\right)~.$$
Both of these paradigms satisfy exponential convergence rates, thus, the main implications of the discussion of Section \ref{sec:comparative_analysis_exact} hold in this setting as well if $\beta$ is to be set as constant. 
If $\beta$ is to be tuned for DPO, it cannot be made arbitrarily small or large as observed in Section \ref{sec:comparative_analysis_exact} -- DPO's overall bounds disproportionately depend on the parameter $\beta$. Even so, setting $\beta$ to its optimal value of $\Theta(1/\sqrt{n})$ for the exact optimization setting would not affect the convergence rate of gradient descent for DPO.

\section{Non-realizable Rewards: Exact Optimization}\label{sec:nonrealizable_rewards}

In this section, we consider the case when the ground-truth reward function $r^*$ does not belong to the linear class $\calF$ (see Definition \ref{def:linear_reward}). We again assume that there exists an optimal $\pi^*_{r^*}$ regularized policy that belongs to the loglinear class $\Pi$ (see Definition \ref{def:loglinear_policies}) for some $\theta^*$. We will capture the mismatch between the reward function and its best linear approximation in $\calF$ by the following condition.

\begin{assumption}[\textit{Non-realizability of the ground-truth reward}]\label{asmp:non-linear_rewards}  
    There exists $r_{\omega^*}\in\calF$ with parameter $\omega^*$ such that $\norm{r^*-r_{\omega^*}}_\infty \leq \epsilon_\textnormal{app}$,
    where $\epsilon_\textnormal{app} >0$ denotes the mismatch coefficient.
\end{assumption}
As we will see, the mismatch coefficient will appear linearly in the RLHF bounds on the gap as an additional constant that cannot be improved by increasing the data size. 

\subsection{Theoretical Results}
We begin with the RLHF result, which can be derived from Theorem~\ref{thm:main_rlhf_result}.
\begin{restatable}{retheorem}{nonrealizablerlhf}\label{cor:non_realizable_rlhf}
    Let $\delta >0$. Suppose that Assumption \ref{asmp:non-linear_rewards} holds. Then, with probability at least $1-\delta$, we have $G\left(\pi_{\widehat{\theta}}\right) \leq D\left(\pi_{\widehat{\theta}}\right) + \widetilde{\Theta} \left( {\Lambda_R} \sqrt{d_R/n}\right)  + 2\epsilon_\textnormal{app}$.
\end{restatable}
 We can directly obtain a similar dependence on $\epsilon_\textnormal{app}$ for DPO. In addition to that, we can also obtain alternative bounds that can be controlled by $\beta$ as follows. 
\begin{restatable}{retheorem}{nonrealizabledpo}\label{cor:non_realizable_dpo}
    Let $\delta >0$. Suppose that Assumption \ref{asmp:non-linear_rewards} and the condition of Lemma \ref{lem:optimal_is_loglinear} hold. Then, with probability at least $1-\delta$, we have $G \left(\pi_{\widetilde{\theta}}\right) \leq D\left(\pi_{\widetilde{\theta}}\right) + \Theta\left( \Lambda_P d_P/(\beta n)\right) + \min\{ 2\epsilon_\textnormal{app}, O\left( \beta D_\textnormal{KL}\left( \pi_{\theta^*}||\pi^*\right)\right)\}$~.
\end{restatable}

\subsection{Comparative Analysis}

\looseness-1The key observation to be made in this section is the discrepancy of the bounds in terms of the reward mismatch coefficient.  RLHF does not use MLE for the policy parameter estimation, but first learns a reward model. Thus, it cannot bypass the error coming from the reward unrealizability. For DPO, note that, if we set $\beta = O(1/\sqrt{n})$, its bounds improves asymptotically with $n$, assuming that $D_\textnormal{KL}(\pi_{\theta^*}||\pi^*_{r^*})$ is bounded (recall that $\pi^*_{r^*}$ denotes the optimal regularized policy and $\pi_{\theta^*}$ denotes its best loglinear fit). This setting benefits DPO as it is designed to bypass the reward function and directly optimize over the policy space.

\section{Realizable Rewards: Exact Optimization -- An Extension to Deterministic MDPs}\label{sec:extension_to_mdps}

\looseness-1Up to this point, our discussion was concentrated on the contextual bandit setting, which has been used in the DPO literature for the KL-regularized problem \cite{rafailov2023direct}. Now, we focus on a generalization of our comparative analysis to Markov decision processes (MDPs), where contexts are related to each other through transition dynamics. 

As mentioned previously, the discrepancy between reward and policy dimensions plays a crucial role in the relative performances of RLHF and DPO. While these dimensions could arguably be similar (or have a small gap) for the contextual bandit setting, that is not necessarily the case in general when extending to MDPs, where the reward dimension can be smaller than the policy dimension. For this section, we assume that the ground-truth reward function $r^*$ is linear and realizable in $\calF$. 

\subsection{Preliminaries for Deterministic MDPs}\label{sec:mdp_preliminaries}

\looseness-1For an MDP, $\calX$ is the set of states and $\calY$ the set of actions. In particular, we consider deterministic MDPs, with a transition function $T:\calX\times\calY\rightarrow\calX$ that provides the next state, given the current state-action. 

The value function in infinite-horizon MDPs is given as $V^\pi_r(x)=\expp[\sum_{t\geq0}\gamma^t
r(x_t,y_t)|\rho,\pi]$, where $\rho$ denotes the initial state distribution. Given policy $\pi$, the occupancy measure of $\pi$ is given by $d^\pi_\rho(x,y)=(1-\gamma)\sum_{t\geq 0}\gamma^t\mathbb{P}\left( x_t=x,y_t=y|\rho,\pi\right)$. We will consider the class of loglinear occupancy measures, as defined next.
\begin{definition}[\textit{Loglinear occupancy measures class}]\label{def:loglinear_occupancy_measures}
    Let $\psi'(x,y)\in\mathbb{R}^{d_M}$ denote the feature vector of the pair $(x,y)$ with $\max_{x,y}\norm{\psi'(x,y)}_2\leq 1$, and $B'>0$. We consider the following class of loglinear occupancy measures:
    \begin{align*}
         \Pi' & = \Big\{ d^\theta_\rho: d^\theta_\rho(x,y) = \frac{\exp(\theta^\top \psi'(x,y))}{\sum_{x',y'} \exp(\theta^\top \psi'(x',y'))},\\ &  \forall (x,y)\in\calX\times\calY\;\text{where}\;\theta \in\mathbb{R}^{d_M}\;\text{and}\; \norm{\theta}_2\leq B'\Big\}~.
    \end{align*}
\end{definition}
In this section, we use the $\calV^{d_\rho}_r(\rho)$ notation, instead of $\calV^\pi_{r}(\rho)$. Similarly, we use $G(d_\rho)$ to denote the gap in terms of occupancy measure $d_\rho$, and $D(d_\rho)$ for the difference of the gaps. For a complete discussion, see Appendix~\ref{sec:dpo_for_mdp_formulation}. 

\looseness-1In this setting, we are given a dataset $\calD_n=\{ (x_{0,i},\tau^w_i,\tau^l_i)\}^n_{i=1}$, where $x_{0,i}$ denotes the initial state of the $i$th sample and $\tau^w_i$ denotes the preferred trajectory $(x_{0,i},y^w_{0,i},x^w_{1,i},\ldots)$ over $\tau^l_i$. Analogous to Section \ref{sec:preliminaries}, we define the Bradley-Terry preference model for two trajectories $\tau^w$ and $\tau^l$ as $P^*(\tau^w \succ\tau^l |x_0) = \sigma(R^*(\tau^w)-R^*(\tau^l))$,
where $R^*(\tau)=\sum_{t\geq 0}\gamma^t r^*(x_t,y_t)$ is the discounted return, and $\tau^w\succ \tau^l$ denotes $\tau^w$ being preferred over $\tau^l$.

\subsection{RLHF and DPO for MDPs}

Similar to the contextual bandit setting, the objective for the reward learning phase of RLHF in MDPs with linear rewards can be  written as follows: 
\begin{align}\label{op_mle_rlhf_mdp}
    & \min_\omega \calL^\omega_\textnormal{RLHF}(\calD_n):= -\expp_{(x_0,\tau^w,\tau^l)\sim\calD_n}\tag{P3.1} \\ & \left[ \log\sigma \left( \omega^\top\left(\sum_{t\geq0}\gamma^t\left( \phi(x^w_t,y^w_t)-\phi(x^l_t,y^l_t)\right)\right)\right) \right]\nonumber~.
\end{align}
Once we have the estimated reward function $r_{\widehat{\omega}}$, the objective is to solve the KL-regularized problem. Following previous literature on KL-regularized RL  \cite{nachum2019algaedice, lee2021optidice}, we formulate the objective in this setting as
\begin{align}\label{op_kl_reg_for_mdp}
    \max_\pi V^\pi_{r_{\widehat{\omega}}}(\rho) - \beta D_\textnormal{KL}\left(d^\pi_\rho||d^\mu_\rho\right)~.\tag{P3.2}
\end{align}
\looseness-1We will assume throughout this section that we are given access to oracles that exactly solve Problem \eqref{op_mle_rlhf_mdp} and \eqref{op_kl_reg_for_mdp}.
\begin{remark}
    Note that the objective in Problem \eqref{op_kl_reg_for_mdp} depends on $\rho$, while the objective in Problem \eqref{op_kl_regularized} depends on $\calD_n$. This is due to considering occupancy measures instead of policies. We keep our current formulation for ease of presentation and leave its extension to a sample objective formulation for future work.
\end{remark}
For the purposes of our comparative analysis, we also need an extension of DPO to MDPs, based on the preference model of Section \ref{sec:mdp_preliminaries}. The key difficulty of extending DPO to the MDP setting is that the gradient has a non-linear dependence on the policy. To bypass this issue, we leverage the fact that transitions are deterministic to simplify cumulative differences of the optimal Lagrange multipliers for Problem \eqref{op_kl_reg_for_mdp}, and obtain the following loss function for DPO:
\begin{align*}
    & \calL_\textnormal{DPO}^{d_\rho} (\calD_n) = -\expp_{(x_0,\tau^w,\tau^l)\sim\calD_n}\Bigg[ \\ & \log \sigma \left( \beta \sum^\infty_{t=0}\gamma^t\left( \log \frac{d_\rho (x^w_t,y^w_t)}{d^\mu_{\rho}(x^w_t,y^w_t)} - \log \frac{d_{\rho}(x^l_t,y^l_t)}{d^\mu_{\rho}(x^l_t,y^l_t)}\right)\right) \Bigg]~.
\end{align*}
All derivations are in Appendix \ref{sec:dpo_for_mdp_formulation}. 
Next, we generalize the bounds from Section~\ref{sec:rlhf_betterthan_dpo} for the above formulations. 

\subsection{Theoretical Results}
Analogous to the previous sections, let us define 
\begin{align*}
    \Sigma'_{\calD_n,R}=\frac{1}{n}\sum_{(x_0,\tau^w,\tau^l)\in\calD_n}\overline{\phi}'(x_0,\tau^w,\tau^l)\overline{\phi}'(x_0,\tau^w,\tau^l)^\top~,
\end{align*} 
where $\overline{\phi}'(x_0,\tau^w,\tau^l)=\sum_{t\geq 0}\gamma^t(\phi(x_t,y^w_t)-\phi(x_t,y^l_t))$. For $\lambda>0$, define $\Lambda'_R=\norm{(\Sigma'_{\calD_n,R}+\lambda I)^{-1/2}}_2$. Similarly, let the sample covariance matrix with respect to occupancy measure features be defined as 
\begin{align*}
    \Sigma'_{\calD_n ,M}=\frac{1}{n}\sum_{(x_0,\tau^w,\tau^l)\in\calD_n} \overline{\psi}'(x_0,\tau^w,\tau^l)\overline{\psi}'(x_0,\tau^w,\tau^l)^\top~,
\end{align*}
where $\overline{\psi}'(x_0,\tau^w,\tau^l)=\sum_{t\geq 0}\gamma^t ((\psi'(x_t,y^w_t)-\psi'(x_t,y^l_t))$. Let $\Lambda'_M=\norm{(\Sigma'_{\calD_n,M}+\lambda I)^{-1/2}}_2$.
For RLHF with exact optimization, it is straightforward to extend our previous bounds as follows. 
\begin{restatable}{retheorem}{cormainrlhf}\label{cor:rlhf_mdp_main}
    Let $\delta >0$. Assume that the policy learning phase yields ${\widehat{\theta}}\in\arg\max_\theta\calV^{d^{\theta}_\rho}_{r_{\widehat{\omega}}}(\rho)$, where $r_{\widehat{\omega}}$ is the estimated reward. Then, with probability at least $1-\delta$, the suboptimality gap incurred by RLHF is $ G (d^{\widehat{\theta}}_\rho) =  D(d^{\widehat{\theta}}_\rho) + \widetilde{\Theta} ( \Lambda'_R \sqrt{d_R/n})~.$ 
\end{restatable}
Now that we have a formulation for DPO, we can also extend the previous bounds for the MDP setting. Similar to Lemma \ref{lem:optimal_is_loglinear}, we now state an analogous result for this setting that guarantees the expression of the ground-truth reward in terms of optimal loglinear occupancy measures. Recall that $\Phi$ and $\Psi$ denote the reward and policy feature matrices, respectively. Let $\pi^*_{r^*}\in \arg\max_\pi \calV^{\pi}_{r^*}(\rho)$. Define $\Phi_{\pi^*_{r^*}}$ to be the $d_R\times XY$-dimensional matrix with columns defined as $\gamma\expp[\sum_t\gamma^t\phi(x_t,y_t)|x_0=x,\pi^*_{r^*}] - \expp[\sum_t\gamma^t\phi(x_t,y_t)|x_0=x,y_0=y,\pi^*_{r^*}]$. We have the following result.
\begin{lemma}\label{lem:optimal_occ_measure_is_loglinear}
    Assume that $r^*\in\calF$, $d^\mu_\rho\in\Pi'$ and $d^{\pi^*_{r^*}}_\rho\in\Pi'$, for some optimal $d^{\pi^*_{r^*}}_\rho$. Furthermore, assume that the column space of $\Phi+\Phi_{\pi^*_{r^*}}$ is contained in the column space of $\Psi$. Then, for finite MDPs with deterministic transitions, there exists $\theta^*$ such that $d^{\theta^*}_\rho\in\arg\max_d \calV^d_{r^*}(\rho)$ and $d^{\theta^*}_\rho(x,y)\propto d^\mu_\rho(x,y)\exp(A^{\pi^*_{r^*}}_{r^*}(x,y)/\beta)$, where $A^{\pi^*_{r^*}}_{r^*}(x,y)=Q^{\pi^*_{r^*}}_{r^*}(x,y)-V^{\pi^*_{r^*}}_{r^*}(x)$.
\end{lemma}
Now we are ready to state the DPO result for this section.
\begin{restatable}{retheorem}{dpomdpmain}\label{cor:dpo_mdp_main}
    Let $\delta > 0$. Let $d^{\widetilde{\theta}}_\rho$ denote the occupancy measure returned by DPO and assume that $d^{\pi^*}_\rho\in\Pi'$, for some $d^{\pi^*}_\rho\in\arg\max_{d_\rho}\calV^{d_\rho}_{r^*}(\rho)$, and that the condition of Lemma \ref{lem:optimal_occ_measure_is_loglinear} is satisfied. Then, for any $n\geq O\left(tr((\Sigma'_{{\calD_n,M}})^\dagger)/(\beta (B')^2)\right)$, with probability at least $1-\delta$, we have $G (d^{\widetilde{\theta}}_\rho) =   D(d^{\widetilde{\theta}}_\rho) + \Theta( \Lambda'_M (d_M+1)/(\beta n) + \beta\Lambda_M\lambda (B')^2)~.$
\end{restatable}
\textit{Proof sketch.} We start by expressing the optimal discounted reward in terms of an optimal occupancy measure, which is also loginear, by using Lemma \ref{lem:optimal_occ_measure_is_loglinear}. Then, we equivalently express the value function in terms of occupancy measures. This allows us to cancel out some terms and express the whole gap in terms of the $D_\textnormal{KL}(d^{\widetilde{\theta}}_\rho||d^{\theta^*}_\rho)$. Finally, similar to the proof of Theorem \ref{thm:main_dpo_theorem}, using loglinearity and properties of the log-sum-exp function, we obtain the desired bounds.\qed

\subsection{Comparative Analysis}

The main implication of the above results is that the observations made in Section \ref{sec:rlhf_betterthan_dpo} extend to deterministic MDPs, using our proposed formulation of RLHF and DPO. For the optimal value of $\beta$ for DPO as discussed in Section \ref{sec:comparative_analysis_exact}, the RLHF and DPO bounds become directly comparable in terms of the dimension differences for deterministic MDPs. In MDPs with simple reward models (e.g., low-dimensional linear reward models), typically there is still a necessity for high-dimensional policy parameters to represent the value function effectively. This suggests that the complexity of the policy class exceeds that of the reward class and that RLHF outperforms DPO in such instances.

\section{Concluding Discussion}\label{sec:conclusion}

In this paper, we provided a comparative analysis between reinforcement learning from human feedback (RLHF) and direct preference optimization (DPO). We performed a thorough analysis under different settings, where we derived sample complexity bounds for both paradigms and drew conclusions on their statistical comparison, based on sample size, regularization temperature, and the dimensionality of their respective parametrizations. We believe these results will initiate a larger discussion on the differences between these two paradigms. 

\looseness-1There are many interesting future directions to pursue. The first natural extension of this work is to relax the assumptions made on policy and reward classes and provide a comparative analysis of RLHF and DPO on more realistic settings (eg. general function approximation). Next, a systematic large-scale empirical investigation that would validate the theoretical insights of this paper would be of great importance. Finally, as our current extension of DPO for MDPs is limited to deterministic MDPs using a loglinear occupancy measure regularization, it would be interesting to see whether DPO can be extended to more general formulations.

\section*{Impact Statement}
This paper focuses on the theoretical aspects of machine learning, providing a comparative analysis of different paradigms of learning from human preferences. We do not foresee any direct negative outcomes from the findings of this paper. On the contrary, we believe that our results might initiate a larger discussion on the statistical properties of learning from human preferences. 

\section*{Acknowledgements}
The work of Andi Nika and Goran Radanovic was funded by the Deutsche Forschungsgemeinschaft (DFG, German Research Foundation) – project number 467367360.

\bibliography{bibliography}
\bibliographystyle{icml2024}

\newpage
\onecolumn
\appendix

\section*{{\LARGE Appendix}}

\begin{table*}[ht]
    \small
    \centering
    % \sisetup{table-number-alignment=center, 
    % separate-uncertainty=true}
    \renewcommand{\arraystretch}{2}
    % \newcommand{\boundscolor}{mydarkblue}
    % \footnotesize
    % \captionsetup{font=footnotesize}
    \resizebox{\textwidth}{!}{
    \begin{tabular}{l r} 
    % \toprule
        \textbf{\large Table of Contents} \\
    \hline 
    \textbf{\ref{sec:rlhf_main_bounds} \;\;\; \textcolor{mydarkblue}{Statistical Bounds for RLHF (Section \ref{sec:rlhf_betterthan_dpo} and \ref{sec:extension_to_mdps})}} & \pageref{sec:rlhf_main_bounds}\\
    \textbf{\ref{sec:main_dpo_proof} \;\;\; \textcolor{mydarkblue}{Statistical Bounds for DPO (Section \ref{sec:rlhf_betterthan_dpo})}} & \pageref{sec:main_dpo_proof}\\
    \textbf{\ref{sec:dpo_exact_bounds_mdp_proof} \;\;\; \textcolor{mydarkblue}{Statistical Bounds for DPO for MDPs (Section \ref{sec:extension_to_mdps})}} & \pageref{sec:dpo_exact_bounds_mdp_proof}\\
    \textbf{\ref{sec:rlhf_mle_convergence} \;\;\; \textcolor{mydarkblue}{Convergence of Gradient Descent for RLHF Reward Learning (Section \ref{sec:approximate_optimization})}} \hspace{3cm} & \pageref{sec:rlhf_mle_convergence}\\
    \textbf{\ref{sec:rlhf_convergence_proof} \;\;\; \textcolor{mydarkblue}{Convergence of Natural Policy Gradient for RLHF (Section \ref{sec:approximate_optimization})}} & \pageref{sec:rlhf_convergence_proof}\\
    \textbf{\ref{sec:dpo_convergence} \;\;\; \textcolor{mydarkblue}{Convergence of Gradient Descent for DPO (Section \ref{sec:approximate_optimization})}} & \pageref{sec:dpo_convergence}\\
    \textbf{\ref{sec:non_realizable_proofs} \;\;\; \textcolor{mydarkblue}{Non-realizable Rewards (Section \ref{sec:nonrealizable_rewards})}} & \pageref{sec:non_realizable_proofs}\\
    \textbf{\ref{sec:dpo_for_mdp_formulation} \;\;\; \textcolor{mydarkblue}{The DPO Extension to MDPs}} & \pageref{sec:dpo_for_mdp_formulation}\\
    \textbf{\ref{sec:mdp_gradient} \;\;\;\; \textcolor{mydarkblue}{Gradient Expression for KL-regularized Objective in MDPs}} & \pageref{sec:mdp_gradient}\\
    \textbf{\ref{sec:technical_lemmas} \;\;\;\; \textcolor{mydarkblue}{Technical Lemmas}} & \pageref{sec:technical_lemmas}\\
        \bottomrule
    \end{tabular}}
\end{table*}

\section{Statistical Bounds for RLHF (Section \ref{sec:rlhf_betterthan_dpo} and \ref{sec:extension_to_mdps})}\label{sec:rlhf_main_bounds}

In this section, we prove the main RLHF result, Theorem \ref{thm:main_rlhf_result}. We state the detailed version of it together with the necessary constants.

\begin{theorem}
\label{app-thm:rlhf-gap}
    Let $\delta >0$. Assume that the preference data satisfies the BT model, and $r^*\in\calF$. Denote by $\widehat{\omega}$ and $\widehat{\theta}$ the reward and policy parameters learned via RLHF, respectively. Furthermore, assume that $$\widehat{\omega} \in \arg\min_\omega \calL^\omega_\textnormal{RLHF}(\calD_n)$$ and $${\widehat{\theta}} \in \arg\max_\theta \calV^{\pi_\theta}_{r_{\widehat{\omega}}}(\calD_n)~.$$ Then, with probability at least $1-\delta$, for any $\lambda >0$, the suboptimality gap incurred by RLHF is
    \begin{align*}
        G \left(\pi_{\widehat{\theta}}\right)
        & = \Theta\left( \norm{\left(\Sigma_{\calD_n,R}+\lambda I\right)^{-1/2}}_2 \cdot  \sqrt{\frac{d_R+\log(6/\delta)}{S^2_R n}+\lambda F^2}\right) +  D\left( \pi_{\widehat{\theta}}\right) ,
    \end{align*}
    where $S_R =1/\left(2+\exp\left( -2F\right)+\exp\left(2F\right)\right)$.
\end{theorem}

\begin{proof}
Let $\Phi\in\mathbb{R}^{d_R\times XY}$ be the reward feature matrix.
Then, for any $\lambda >0$, with probability at least $1-\delta$, we have
\begin{align*}
    & G(\pi_{\widehat{\theta}}) 
        =  V^\textnormal{opt}_{r^*}(\rho) - V^{\pi_{\widehat{\theta}}}_{r^*}(\rho) \\
        & = D(\pi_{\widehat{\theta}}) + \left(  \calV^{\pi^*_{r^*}}_{r^*}(\rho)-\calV^{\pi_{\widehat{\theta}}}_{r^*}(\rho) \right)  \\
        & = D(\pi_{\widehat{\theta}}) + \left(  \calV^{\pi^*_{r^*}}_{r^*}(\rho) - \calV^{\pi^*_{r^*}}_{r_{\widehat{\omega}}}(\rho)\right) + \left( \calV^{\pi^*_{r^*}}_{r_{\widehat{\omega}}}(\rho) -\calV^{\pi_{\widehat{\theta}}}_{r_{\widehat{\omega}}}(\rho)\right) +\left( \calV^{\pi_{\widehat{\theta}}}_{r_{\widehat{\omega}}}(\rho) - \calV^{\pi_{\widehat{\theta}}}_{r^*}(\rho)\right) \\
        & \stackrel{(a)}{\leq} D(\pi_{\widehat{\theta}}) + \left(  \calV^{\pi^*_{r^*}}_{r^*}(\rho) - \calV^{\pi^*_{r^*}}_{r_{\widehat{\omega}}}(\rho)\right) + \left( \calV^{\pi^*_{r_{\widehat{\omega}}}}_{r_{\widehat{\omega}}}(\rho) -\calV^{\pi_{\widehat{\theta}}}_{r_{\widehat{\omega}}}(\rho)\right) +\left( \calV^{\pi_{\widehat{\theta}}}_{r_{\widehat{\omega}}}(\rho) - \calV^{\pi_{\widehat{\theta}}}_{r^*}(\rho)\right) \\
        & \stackrel{(b)}{\leq} D(\pi_{\widehat{\theta}}) + \left(  \calV^{\pi^*_{r^*}}_{r^*}(\rho) - \calV^{\pi^*_{r^*}}_{r_{\widehat{\omega}}}(\rho)\right) + \left( \calV^{\pi^*_{r_{\widehat{\omega}}}}_{r_{\widehat{\omega}}}(\calD_n) -\calV^{\pi_{\widehat{\theta}}}_{r_{\widehat{\omega}}}(\calD_n)\right) + O\left(\sqrt{\frac{\log(6/\delta)}{n}}\right) +\left( \calV^{\pi_{\widehat{\theta}}}_{r_{\widehat{\omega}}}(\rho) - \calV^{\pi_{\widehat{\theta}}}_{r^*}(\rho)\right) \\
        & \stackrel{(c)}{\leq} D(\pi_{\widehat{\theta}}) + \left(  \calV^{\pi^*_{r^*}}_{r^*}(\rho) - \calV^{\pi^*_{r^*}}_{r_{\widehat{\omega}}}(\rho)\right) + \left( \calV^{\pi_{\widehat{\theta}}}_{r_{\widehat{\omega}}}(\calD_n) -\calV^{\pi_{\widehat{\theta}}}_{r_{\widehat{\omega}}}(\calD_n)\right) + O\left(\sqrt{\frac{\log(6/\delta)}{n}}\right) +\left( \calV^{\pi_{\widehat{\theta}}}_{r_{\widehat{\omega}}}(\rho) - \calV^{\pi_{\widehat{\theta}}}_{r^*}(\rho)\right) \\
        & = D(\pi_{\widehat{\theta}}) +\sum_{x,y} \rho(x) \cdot (\pi^*_{r^*}(y|x) - \pi_{\widehat{\theta}}(y|x)) \cdot (r^*(x,y) -r_{\widehat{\omega}}(x,y)) + O\left(\sqrt{\frac{\log(6/\delta)}{n}}\right) \\
        & \stackrel{(d)}{=} D(\pi_{\widehat{\theta}}) + \left(d^*_\rho - d^{\pi_{\widehat{\theta}}}_\rho\right)^\top (\omega^*-\widehat{\omega}) \Phi + O\left(\sqrt{\frac{\log(6/\delta)}{n}}\right)\\
        & \stackrel{(e)}{\leq} D(\pi_{\widehat{\theta}}) + \norm{ \Phi\left(d^*_\rho - d^{\pi_{\widehat{\theta}}}_\rho\right)}_{\left(\Sigma_{\calD_n,R} +\lambda I\right)^{-1}} \norm{\omega^*-\widehat{\omega}}_{\Sigma_{\calD_n,R} +\lambda I}  + O\left(\sqrt{\frac{\log(6/\delta)}{n}}\right)\\
        & \stackrel{(f)}{\leq} D(\pi_{\widehat{\theta}}) + O\left( \norm{\left(\Sigma_{\calD_n,R}+\lambda I\right)^{-1/2}}_2 \cdot  \sqrt{\frac{d_R+\log(6/\delta)}{S^2_R n} +\lambda F^2}\right)~,
\end{align*}
where $(a)$ is due to the fact that $\pi^*_{r_{\widehat{\omega}}} \in \arg\max_\pi\calV_{r_{\widehat{\omega}}}^\pi(\rho)$; $(b)$ follows from Lemma~\ref{lem:training_value_close} and the union bound -- if we have $\mathbb{P}(\calE^c_i)\leq \delta_i$, for $i=1,2,3$, where $\calE^c$ denotes the complement of event $\calE$, letting $\delta_i=\delta/3$, for all $i$, we have
\begin{align*}
    \mathbb{P}\left(\calE_1\cup\calE_2\cup\calE_3\right) & = 1-\mathbb{P}\left(\calE^c_1\cup\calE^c_2\cup\calE^c_3\right) \\
    & \geq 1 - \left( \mathbb{P}(\calE_1)+\mathbb{P}(\calE_2)+\mathbb{P}(\calE_3)\right)\\
    & \geq 1 - \left(\delta_1+\delta_2+\delta_3\right)\\
    & = 1 -\delta~.
\end{align*}
Next, $(c)$ is due to the fact that $\pi_{\widehat{\theta}} \in \arg\max_{\pi \in \Pi}\calV_{r_{\widehat{\omega}}}^\pi(\calD_n)$ and $\pi^*_{r_{\widehat{\omega}}} \in \Pi$ (as per Lemma \ref{lem:optimal_is_loglinear}); $(d)$ is due to $d^*_\rho(x,y)=\rho(x,y)\pi^*_{r^*}(y|x)$ and $d^{\pi_{\widehat{\theta}}}_\rho(x,y)=\rho(x)\pi_{\widehat{\theta}}(y|x)$; $(e)$ is an application of the Cauchy-Schwarz inequality with respect to the semi-norm induced by matrix $\Sigma_{\calD_n,R}+\lambda I$; and $(f)$ is a direct application of Lemma 3.1 of \cite{zhu2023principled} for the discounted infinite-horizon setting.

The lower bound is an immediate application of Theorem 3.10 of \cite{zhu2023principled}. Note that we are under the same conditions; our reward function is assumed to be linear, and we also assume a bounded covering number. For the lower bound construction, let $\mu = \pi^\textnormal{opt}_{r^*}$, i.e., the reference policy is the actually optimal one. Let $\textnormal{CB}(\Lambda)$ denote the set of bandit instances coupled with datasets with a covering number no more than $\Lambda$. Let $\calQ$ denote such an instance. Under these assumptions, Theorem 3.10 of \cite{zhu2023principled} implies an information-theoretic lower bound of
\begin{align*}
    \inf_{\pi} \sup_{\calQ \in \textnormal{CB}(\Lambda_R)} \left( V^\textnormal{opt}_{\calQ}(\rho) - V^{\pi}_{\calQ}(\rho)\right) \geq O\left(\Lambda_R \sqrt{\frac{d_R}{n}}\right)~.
\end{align*}
\end{proof}

\cormainrlhf*

\begin{proof}
    The proof of this result is an immediate application of the previous result with instead an application of Lemma 5.1 of \cite{zhu2023principled} with $$ S_R =1/\left(2+\exp\left( -2F(1-\gamma)\right)+\exp\left(2F(1-\gamma)\right)\right)~.$$
\end{proof}

Next, we connect the difference between the sample regularized gap and the expected regularized gap with respect to the context distribution, and obtain the following result.

\begin{lemma}\label{lem:training_value_close}
    Let $\delta >0$ and assume that the conditions of Theorem~\ref{app-thm:rlhf-gap} are satisfied. Then, we have that
    \begin{align*}
        \left| \expp_{x\sim\rho}\left[\calV^{\pi^*_{r_{\widehat{\omega}}}}_{r_{\widehat{\omega}}}(x)\right] - \frac{1}{n}\sum_{x\in\calD_n} \calV^{\pi^*_{r_{\widehat{\omega}}}}_{r_{\widehat{\omega}}}(x) \right| \leq \sqrt{\frac{\log(4/\delta)}{n}}~,
    \end{align*}
    and 
    \begin{align*}
        \left| \frac{1}{n}\sum_{x\in\calD_n} \calV^{\pi_{\widehat{\theta}}}_{r_{\widehat{\omega}}}(x) - \expp_{x\sim\rho}\left[\calV^{\pi_{\widehat{\theta}}}_{r_{\widehat{\omega}}}(x)\right]\right| \leq \sqrt{\frac{(1+\beta(2B+\log Y))\log(4/\delta)}{n}}~.
    \end{align*}
    with probability at least $1-\delta$. 
\end{lemma}

\begin{proof}
Using the reward-to-policy mapping of Equation \eqref{eq:reward_to_policy_mapping}, we have that, for every $(x,y)$,
\begin{align*}
    \mu(y|x) = \pi^*_{r_{\widehat{\omega}}}(y|x)\exp\left(-\frac{1}{\beta}r_{\widehat{\omega}}(x,y)\right)~.
\end{align*}
Thus, note that, for every $x\in\calX$,
\begin{align*}
    \left|{\calV^{\pi^*_{r_{\widehat{\omega}}}}_{r_{\widehat{\omega}}}(x)}\right| & = \left| \sum_y \pi^*_{r_{\widehat{\omega}}}(y|x)r_{\widehat{\omega}}(x,y) - \beta\pi^*_{r_{\widehat{\omega}}}(y|x)\log\frac{\pi^*_{r_{\widehat{\omega}}}(y|x)}{\mu(y|x)} \right| \\
        & \leq 1 + \beta \left| \sum_y \pi^*_{r_{\widehat{\omega}}}(y|x) \log\frac{\pi^*_{r_{\widehat{\omega}}}(y|x)}{\mu(y|x)}\right| \\
    & = 1 + \beta \left| \sum_y \pi^*_{r_{\widehat{\omega}}}(y|x)\frac{1}{\beta}r_{\widehat{\omega}}(x,y) \right|\\
        & \leq 2~,
\end{align*}
where we have used that the reward lies in $[0,1]$. On the other hand, we have
\begin{align*}
    \left| {\calV^{\pi_{\widehat{\theta}}}_{r_{\widehat{\omega}}}(x)}\right| & = \left| \sum_y \pi_{\widehat{\theta}}(y|x)r_{\widehat{\omega}}(x,y) - \beta\pi_{\widehat{\theta}}(y|x)\log\frac{\pi_{\widehat{\theta}}(y|x)}{\mu(y|x)}\right|\\
        & \leq 1 + \beta \left| \sum_y \pi_{\widehat{\theta}}(y|x)\left( \log\frac{\pi_{\widehat{\theta}}(y|x)}{\pi^*_{r_{\widehat{\omega}}}(y|x)} + \frac{1}{\beta}r_{\widehat{\omega}}(x,y)\right)\right|\\
    & \leq 2 + \beta\max_y \left| \log\frac{\pi_{\widehat{\theta}}(y|x)}{\pi_{\theta^*}(y|x)}\right|\\
        & \leq 2 + \beta\max_{x,y}\left( \left| \log  \frac{\exp(\psi(x,y)^\top\widehat{\theta})}{\sum_{y'} \exp(\psi(x,y')^\top\widehat{\theta})} \right|+ \left|\log  \frac{\exp(\psi(x,y)^\top\theta^*)}{\sum{y'}\exp(\psi(x,y')^\top\theta^*)}\right|\right)\\
    & \leq 2 +\beta \max_{x,y}\left( \left|\log\exp(\psi(x,y)^\top\widehat{\theta})\right| + \left| \log \sum_{y'}\exp(\psi(x,y')^\top\widehat{\theta})\right| \right.\\ & \quad\quad \left. + \left|\log\exp(\psi(x,y)^\top\theta^*)\right| + \left|\log\sum_{y'}\exp(\psi(x,y')^\top\theta^*)\right|\right) \\
        & \leq 2 +\beta \left( 2B + 2 \log\left( Y\exp(B)\right)\right)\\
    & \leq 2 + 2\beta (2B + \log Y)~,
\end{align*}
where we have used Lemma \ref{lem:loglinear_is_optimal} and the fact that 
\begin{align*}
    -B \leq \langle \psi(x,y) ,\theta \rangle \leq B~.
\end{align*}
The result then follows from Hoeffding's inequality. 
\end{proof}

\section{Statistical Bounds for DPO (Section \ref{sec:rlhf_betterthan_dpo})}\label{sec:main_dpo_proof}

In this section, we prove the main DPO result for Section \ref{sec:rlhf_betterthan_dpo}.

\maindporesult*

\begin{proof}
    Lemma \ref{lem:optimal_is_loglinear} implies that there exists $\theta^*\in\mathbb{R}^d$, such that, for every $(x,y)$, 
    \begin{align*}
        r^*(x,y)= \beta \log\frac{\pi_{\theta^*}(y|x))}{\mu(y|x)} + \beta Z(x)
    \end{align*}
    and $\calV^*_{r^*}(\rho) = \calV^{\pi_{\theta^*}}_{r^*}(\rho)$. Now, observe that
    \begin{align*}
        G \left(\pi_{\widetilde{\theta}}\right) & = V^\textnormal{opt}_{r^*}(\rho) - V^{\pi_{\widetilde{\theta}}}_{r^*}(\rho) \\
        & = D(\pi_{\widetilde{\theta}}) + \left( \calV^{\pi_{\theta^*}}_{r^*}(\rho)-\calV^{\pi_{\widetilde{\theta}}}_{r^*}(\rho)\right)\\
        & = D(\pi_{\widetilde{\theta}}) + \expp_{\substack{x\sim\rho \\ y\sim \pi_{\theta^*}(\cdot|x)}}\left[ r^*(x,y) - \beta\log\frac{\pi_{\theta^*}(y|x)}{\mu(y|x)} \right]  - \expp_{\substack{x\sim\rho \\ y\sim \pi_{\widetilde{\theta}}(\cdot|x)}}\left[r^*(x,y) - \beta\log\frac{\pi_{\widetilde{\theta}}(y|x)}{\mu(y|x)}\right]\\
        & = D(\pi_{\widetilde{\theta}}) + \expp_{\substack{x\sim\rho \\ y\sim \pi_{\theta^*}(\cdot|x)}}\bigg[ \beta \log\frac{\pi_{\theta^*}(y|x)}{\mu(y|x)} + \beta\log Z(x)  - \beta\log\frac{\pi_{\theta^*}(y|x)}{\mu(y|x)} \bigg] \\ & \quad\quad\quad - \expp_{\substack{x\sim\rho \\ y\sim \pi_{\widetilde{\theta}}(\cdot|x)}}\bigg[  \beta \log\frac{\pi_{\theta^*}(y|x)}{\mu(y|x)}   + \beta \log Z(x) - \beta\log\frac{\pi_{\widetilde{\theta}}(y|x))}{\mu(y|x)}\bigg]\\
        & = D(\pi_{\widetilde{\theta}}) + \expp_{\substack{x\sim \rho \\ y\sim \pi_{\widetilde{\theta}}(\cdot|x)}}\left[ \beta\log \pi_{\widetilde{\theta}}(y|x) -  \beta \log \pi_{\theta^*}(y|x) \right]\\
        & = D(\pi_{\widetilde{\theta}})  + \expp_{\substack{x\sim\rho \\ y\sim \pi_{\widetilde{\theta}}}} \Bigg[ \beta \left\langle \psi(x,y), \widetilde{\theta} - \theta^*\right\rangle  + \beta \log\frac{\sum_{y'}\exp(\psi(x,y')^\top\theta^*)}{\sum_{y'}\exp(\psi(x,y')^\top\widetilde{\theta})}\Bigg]\\
        & = D(\pi_{\widetilde{\theta}}) + \beta \expp_{\substack{x\sim\rho \\ y\sim \pi_{\widetilde{\theta}}}} \left[ \left\langle \psi(x,y), \widetilde{\theta} - \theta^*\right\rangle \right] + \beta (A(\theta^*)-A(\widetilde{\theta}))~,
    \end{align*}
    where we have denoted by $A(\theta)$ the log-sum-exp function
    \begin{align*}
        A(\theta) = \sum_x \rho(x) \log\sum_{y'} \exp\left( \psi(x,y')^\top\theta\right)~,
    \end{align*}
    At this point, some properties of the log-exp-sum function will be useful. The proof of the following result can be found in Appendix \ref{sec:technical_lemmas}. 
    \begin{lemma}\label{lem:log_exp_sum_properties}
    The function $A(\theta)$
    is $1$-Lipschitz and $2$-smooth. Moreover, if the features are sampled from a $0$-mean distribution and span $R^{d_P}$, then there exists $\kappa >0$, such that $A(\theta)$ is $\kappa$-strongly convex.
    \end{lemma}
    Since $A(\theta)$ is $2$-smooth, we have
    \begin{align*}
        A(\theta^*)- A(\widetilde{\theta}) & \leq \left\langle \nabla_\theta A(\widetilde{\theta}), \theta^*-\widetilde{\theta}\right\rangle + \norm{\theta^*-\widetilde{\theta}}^2_2\\
        & = \expp_{x\sim\rho, y\sim \pi_{\widetilde{\theta}}(\cdot|x)}\left[ \left\langle\psi(x,y),\theta^*-\widetilde{\theta}\right\rangle\right]  + \norm{\theta^*-\widetilde{\theta}}^2_2~.
    \end{align*}
    Substituting to the suboptimality gap equalities, we obtain
    \begin{align*}
        G \left(\pi_{\widetilde{\theta}}\right) & \leq D(\pi_{\widetilde{\theta}})  + \norm{\theta^*-\widetilde{\theta}}^2_2\\
        & \stackrel{(a)}{\leq} D(\pi_{\widetilde{\theta}})  + \beta \norm{(\Sigma_{{\calD_n,P}}+\lambda I)^{-1}}_2\norm{\theta^*-\widetilde{\theta}}^2_{\Sigma_{{\calD_n,P}}+\lambda I}\\
        & \stackrel{(b)}{\leq} D(\pi_{\widetilde{\theta}})  + \beta \norm{(\Sigma_{{\calD_n,P}}+\lambda I)^{-1}}_2\norm{\theta^*-\widetilde{\theta}}^2_{\Sigma_{{\calD_n,P}}} + 4\beta\lambda\Lambda_P B^2\\
        & \stackrel{(c)}{\leq} D(\pi_{\widetilde{\theta}}) +O\left( \frac{ \Lambda_P (d_P+1)}{ \beta n}\right) + 4\beta\lambda\Lambda_P B^2~,
    \end{align*}
    where for $(a)$ we have used that $\langle x, Ax\rangle \leq \norm{A}_2\norm{x}^2_2$; $(b)$ is due to the fact that $\norm{\theta}_2\leq B$; $(c)$ follows from Theorem \ref{thm:btl_mle_minimax} and the assumption that $\widetilde{\theta}\in\arg\min_\theta\calL^\theta_\textnormal{DPO}({\calD_n})$. The value of $\lambda$ can be tuned accordingly. Note that, for fixed $\beta$, letting $\lambda =\Theta(1/n)$ yields the desired bound. If $\beta = \Theta(1/\sqrt{n})$, then any small value of $\lambda$ works. 
    
    On the other hand, consider the following feature construction. Let $\Psi$ be full rank with zero-mean columns. Then, there exists $\kappa >0$ such that $A(\theta)$ is $\kappa$-strongly convex. This, in turn, implies that
    \begin{align*}
        A(\theta^*) - A(\widetilde{\theta}) & \geq \left\langle \nabla_\theta A(\widetilde{\theta}), \theta^* -\widetilde{\theta} \right\rangle + \frac{\kappa}{2}\norm{\widetilde{\theta}-\theta^*}^2_2\\
         & = - \expp_{x\sim\rho, y\sim \pi_{\widetilde{\theta}}(\cdot|x)}\left[ \left\langle \psi(x,y),\widetilde{\theta}-\theta^*\right\rangle \right] + \frac{\kappa}{2}\norm{\widetilde{\theta}-\theta^*}^2_2~.
    \end{align*}
    Thus, substituting in the original gap expression, we obtain
    \begin{align*}
        G \left(\pi_{\widetilde{\theta}}\right)  & \geq D(\pi_{\widetilde{\theta}}) + \frac{\beta \kappa}{2}\norm{\widetilde{\theta}-\theta^*}^2_2 \\
        & \geq D(\pi_{\widetilde{\theta}}) + \frac{\beta \kappa}{2} \norm{\widetilde{\theta}-\theta^*}^2_2\norm{\Sigma_{{\calD_n,P}}}_2 \\
        & \geq D(\pi_{\widetilde{\theta}}) + \frac{\beta \kappa}{2} \left\langle \widetilde{\theta}-\theta^*, \Sigma_{{\calD_n,P}}(\widetilde{\theta}-\theta^*)\right\rangle\\
        & \geq D(\pi_{\widetilde{\theta}}) + \frac{\beta \kappa}{2} \norm{\widetilde{\theta}-\theta^*}_{\Sigma_{{\calD_n,P}}}\\
        & \geq D(\pi_{\widetilde{\theta}}) + \frac{\beta \kappa}{2} \norm{\widetilde{\theta}-\theta^*}_{\Sigma_{{\calD_n,P}}}\\
        & \geq D(\pi_{\widetilde{\theta}}) + \Omega \left(\frac{ (d_P+1)}{\beta n}\right)~,
    \end{align*}
    for any $n\geq O\left(tr(\Sigma_{{\calD_n,P}}^\dagger)/(\beta B^2)\right)$, where the second inequality uses the fact that $\norm{\Sigma_{{\calD_n,P}}}_2 \leq 1$; the third inequality follows from  Cauchy-Schwarz; the last inequality follows by Theorem \ref{thm:btl_mle_minimax}.
\end{proof}

\section{Statistical Bounds for DPO for MDPs (Section \ref{sec:extension_to_mdps})}\label{sec:dpo_exact_bounds_mdp_proof}

In this section, we will prove Corollary \ref{cor:dpo_mdp_main}, the statistical convergence rate of the DPO method in the MDP setting. We restate the result with all the quantities appearing in the bounds. 

\begin{theorem}
    Assume that $r^*\in\calF$, the data $\calD$ satisfies the BT preference model for trajectories, and that the feature matrix is full rank. Furthermore, assume that $d^*_\rho\in\Pi'$ and assume that we have $0$ optimization error from gradient descent on the MLE loss. Let $\beta >0$, and
    \begin{align*}
        S_M & = (\exp(-B') + \exp(B') + 2)^{-1}~,\\
        U' & = \exp(-2B') + \exp(2B') + 2~, \\
        \Lambda_M & = \norm{(\Sigma'_{\calD_n,M}+\lambda I)^{-1/2}}_2.
    \end{align*}
    Then, DPO incurs the following minimax bounds on the suboptimality gap:
    \begin{align*}
        G\left(d^{\widetilde{\theta}}_\rho\right) & =  D\left(d^{\widetilde{\theta}}_\rho\right) + \Theta\left( \frac{\Lambda_M U'(d_M+1)}{\beta S_P n}\right)~.
    \end{align*}
\end{theorem}
\begin{proof}
    First, note that, under some assumptions on the feature space, Lemma \ref{lem:loglinear_is_optimal_v2} implies that there exists $\theta^*$ for which we can have $\calV^*_{r^*}(\rho) = \calV^{d^{\theta^*}_\rho}_{r^*}(\rho)$, with respect to some policy $\pi^*_{r^*}$ and that
    \begin{align}\label{eq_dpo_for_mdp_analysis_eq_01}
        \sum_{t\geq 0}\gamma^t r^*(x_t,y_t) = \sum_{t\geq 0}\gamma^t\beta\log\frac{d^{\theta^*}_\rho(x_t,y_t)}{d^\mu_\rho(x_t,y_t)} + \beta + \alpha^*(x_0)~,
    \end{align}
    for any trajectory $\tau$, where $\alpha^*$ is the optimal dual variable of Problem \eqref{op_dpo_for_mdp}, and we have used Equation \eqref{eq:reward_via_occupancy_measures} to express the ground-truth reward in terms of an optimal regularized policy. Now, let us denote by $\widetilde{\pi}$ the policy corresponding to $d^{\widetilde{\theta}}_\rho$. 
    Observe that
    \begin{align}
        & G\left(d^{\widetilde{\theta}}_\rho\right) = V^{\textnormal{opt}}_{r^*}(\rho) - V^{d^{\widetilde{\theta}}_\rho}_{r^*}(\rho) \nonumber\\
        & = D\left(d^{\widetilde{\theta}}_\rho\right) + \left( \calV^{d^{\theta^*}_\rho}_{r^*}(\rho)-\calV^{d^{\widetilde{\theta}}_\rho}_{r^*}(\rho)\right)\nonumber\\
        & =  D\left(d^{\widetilde{\theta}}_\rho\right) + \expp_{x_0\sim\rho,y_t\sim\pi^*_{r^*}(\cdot|x_t)}\left[\sum_{t\geq0} \gamma^t r^*(x_t,y_t)\right] - \beta \expp_{(x,y)\sim d^{\theta^*}_\rho}\left[\log\frac{d^{\theta^*}_\rho(x,y)}{d^\mu_\rho(x,y)}\right] - \calV^{d^{\widetilde{\theta}}_\rho}_{r^*}(\rho)\nonumber \\
        & = D\left(d^{\widetilde{\theta}}_\rho\right) + \expp_{x_0\sim\rho,y_t\sim\pi^*_{r^*}(\cdot|x_t)}\left[\sum_{t\geq 0}\gamma^t\beta\log\frac{d^{\theta^*}_\rho(x_t,y_t)}{d^\mu_\rho(x_t,y_t)} + \beta + \alpha^*(x_0)\right] - \beta \expp_{(x,y)\sim d^{\theta^*}_\rho}\left[\log\frac{d^{\theta^*}_\rho(x,y)}{d^\mu_\rho(x,y)}\right] - \calV^{d^{\widetilde{\theta}}_\rho}_{r^*}(\rho) \label{eq_dpo_for_mdp_analysis_eq_02}\\
        & =  D\left(d^{\widetilde{\theta}}_\rho\right) + \expp_{(x,y)\sim d^{\theta^*}_\mu}\left[\beta\log\frac{d^{\theta^*}_\rho(x,y)}{d^\mu_\rho(x,y)} \right] + \frac{\beta}{1-\gamma} + \sum_x\rho(x)\alpha^*(x) - \beta \expp_{(x,y)\sim d^{\theta^*}_\rho}\left[\log\frac{d^{\theta^*}_\rho(x,y)}{d^\mu_\rho(x,y)}\right] - \calV^{d^{\widetilde{\theta}}_\rho}_{r^*}(\rho) \label{eq_dpo_for_mdp_analysis_eq_03}\\
        & = D\left(d^{\widetilde{\theta}}_\rho\right) + \frac{\beta}{1-\gamma} \nonumber\\ & \quad + \sum_x\rho(x)\alpha^*(x) -  \expp_{x_0\sim\rho,y_t\sim\widetilde{\pi}(\cdot|x_t)}\left[\sum_{t\geq 0}\gamma^t\beta\log\frac{d^{\theta^*}_\rho(x_t,y_t)}{d^\mu_\rho(x_t,y_t)} + \beta + \alpha^*(x_0)\right] + \beta \expp_{(x,y)\sim d^{\widetilde{\theta}}_\rho}\left[\log\frac{d^{\widetilde{\theta}}_\rho(x,y)}{d^\mu_\rho(x,y)}\right]\nonumber\\
        & = D\left(d^{\widetilde{\theta}}_\rho\right) + \beta \expp_{(x,y)\sim d^{\widetilde{\theta}}_\rho}\left[\log\frac{d^{\widetilde{\theta}}_\rho(x,y)}{d^\mu_\rho(x,y)} - \log\frac{d^{\theta^*}_\rho(x,y)}{d^\mu_\rho(x,y)}\right]\nonumber\\
        & = D\left(d^{\widetilde{\theta}}_\rho\right) + \expp_{(x,y)\sim d^{\widetilde{\theta}}_{\rho}}\left[ \beta\log d^{\widetilde{\theta}}_{\rho}(x,y) -  \beta \log d^{\theta^*}_{\rho}(x,y) \right]\nonumber\\
        & = D\left(d^{\widetilde{\theta}}_\rho\right)  + \expp_{(x,y)\sim d^{\widetilde{\theta}}_{\rho}} \Bigg[ \beta \left\langle \psi'(x,y), \widetilde{\theta} - \theta^*\right\rangle   + \beta \log\frac{\sum_{x',y'}\exp(\psi'(x',y')^\top\theta^*)}{\sum_{x',y'}\exp(\psi'(x',y')^\top\widetilde{\theta})}\Bigg]\label{eq_dpo_for_mdp_analysis_eq_04}\\
        & = D\left(d^{\widetilde{\theta}}_\rho\right) + \beta \expp_{(x,y)\sim d^{\widetilde{\theta}}_\rho} \left[ \left\langle \psi'(x,y), \widetilde{\theta} - \theta^*\right\rangle + A(\theta^*)-A(\widetilde{\theta})\right]~,\nonumber
    \end{align}
    where we have denoted by $$A(\theta)=\log\sum_{x',y'}\exp(\psi'(x',y')^\top\theta)$$ the log-sum-exp function. Above, in Equation \eqref{eq_dpo_for_mdp_analysis_eq_02} we have applied Equation \eqref{eq_dpo_for_mdp_analysis_eq_01}; Equation \eqref{eq_dpo_for_mdp_analysis_eq_03} uses the fact that, for any policy $\pi$ and function $f:\calX\times\calY\rightarrow \mathbb{R}$, we have
    \begin{align*}
        \expp_{x\sim\rho, y_t\sim \pi(\cdot|x_t)}\left[\sum_{t\geq 0}\gamma^t f(x_t,y_t)\right] = \expp_{(x,y)\sim d^\pi}\left[f(x,y)\right]~.
    \end{align*}
    Finally, for Equation \eqref{eq_dpo_for_mdp_analysis_eq_04} we have used loglinearity. Now, given $\theta\in\mathbb{R}^{d_P}$, note that
    \begin{align*}
        \nabla_\theta A(\theta) & = \frac{\sum_{x,y}\exp(\psi'(x,y)^\top\theta) \cdot \psi'(x,y)}{\sum_{x',y'}\exp(\psi'(x',y')^\top\theta)} = \sum_{x,y}d^\theta_\rho(x,y) \psi'(x,y)~.
    \end{align*}
    On the other hand, the Hessian of $A(\theta)$ is
    \begin{align*}
        \nabla^2_\theta A(\theta) & = \sum_{x,y} \nabla_\theta d^\theta_\rho(x,y) \psi'(x,y) \\
        & = \sum_{x,y}d^\theta_\rho(x,y) \left( \psi'(x,y) - \expp_{(x',y')\sim d^\theta_\rho}[\psi'(x',y')]\right)\psi'(x,y)^\top \\
        & = \expp_{(x,y)\sim d^\theta_\rho}\left[ \psi'(x,y)\psi'(x,y)^\top \right] - \expp_{(x,y)\sim d^\theta_\rho}[\psi'(x,y)] \expp_{(x,y)\sim d_\theta}[\psi'(x,y)]^\top \\
        & = \expp_{(x,y)\sim d^\theta_\rho}  \left[ \left(\psi'(x,y) - \expp_\theta\left[\psi'(x,y)\right]\right)\left(\psi'(x,y) - \expp_\theta\left[\psi'(x,y)\right]\right)^\top\right]~.
    \end{align*}
    By assumption on the feature mapping, we have that
    \begin{align*}
        \norm{\nabla^2_\theta A(\theta)}_2 & \leq \max_{x,y} \norm{\left(\psi'(x,y) - \expp_\theta\left[\psi'(x,y)\right]\right)\left(\psi'(x,y) - \expp_\theta\left[\psi'(x,y)\right]\right)^\top}_2\\
        & \leq \max_{x,y}\norm{\psi'(x,y)-\expp_\theta[\psi'(x,y)]}_2\\
        & \leq 2\max_{x,y}\norm{\psi'(x,y)}_2 = 2~.
    \end{align*}
    Therefore, the function $A(\theta)$ is $2$-smooth in $\theta$, which implies that
    \begin{align*}
        A(\theta^*)- A(\widetilde{\theta}) & \leq \left\langle \nabla_\theta A(\widetilde{\theta}), \theta^*-\widetilde{\theta}\right\rangle + \norm{\theta^*-\widetilde{\theta}}^2_2\\
        &  = \expp_{(x,y)\sim d^{\widetilde{\theta}}_\rho}\left[ \left\langle\psi'(x,y),\theta^*-\widetilde{\theta}\right\rangle\right]  + \norm{\theta^*-\widetilde{\theta}}^2_2~.
    \end{align*}
    Substituting to the suboptimality gap equalities, we obtain
    \begin{align*}
        G \left(d^{\widetilde{\theta}}_\rho\right) & \leq D\left(d^{\widetilde{\theta}}_\rho\right)  + \norm{\theta^*-\widetilde{\theta}}^2_2\\
        & \leq D\left(d^{\widetilde{\theta}}_\rho\right)  + \beta \norm{(\Sigma'_{\calD_n,M}+\lambda I)^{-1}}_2\norm{\theta^*-\widetilde{\theta}}^2_{\Sigma'_{\calD_n,M}+\lambda I}\\
        & \leq D\left(d^{\widetilde{\theta}}_\rho\right) +\frac{ \Lambda_M U' d_M}{ S_M\beta n} + 4\beta\Lambda_M\lambda (B')^2~,
    \end{align*}
    where the last inequality follows from Theorem \ref{thm:btl_mle_minimax} and the assumption on exact optimization.
    
    For the lower bound, let $\psi'$ be sampled from a $0$-mean bounded distribution. Note that, for any non-zero vector in $\mathbb{R}^{d_M}$, we have
    \begin{align*}
        z^\top \nabla^2_\theta A(\theta)z & = \expp_{(x,y)\sim d^\theta_\rho}\left[ z^\top\psi'(x,y)\psi'(x,y)^\top z \right]\\
            & \geq \min_{\theta, x,y} d^\theta_\rho(x,y) \sum_{x,y} (\psi'(x,y)^\top z)^2  \\
            & \geq C_3 \sum_{x,y} (\psi'(x,y)^\top z)^2~,
    \end{align*}
    for a positive $C_3$, since $d^\theta_\rho$ is in the loglinear class, for every $\theta$. Now, note that, if $z$ can be expressed as a linear combination of $\{ \psi'(x,y)\}_{x,y}$, the summation cannot be zero for non-zero $z$. Thus, if $\{ \psi'(x,y)\}_{x,y}$ spans $\mathbb{R}^{d_M}$, that is, the feature matrix is full rank, then there exists an absolute positive constant $\kappa$, such that we have
    \begin{align*}
        \norm{\nabla^2_\theta A(\theta)}_2 \geq \kappa > 0~.
    \end{align*}
    Thus, the function $A(\theta)$ is $\kappa$-strongly convex. This, in turn, implies that
    \begin{align*}
        A(\theta^*) - A(\widetilde{\theta}) & \geq \left\langle \nabla_\theta A(\widetilde{\theta}), \theta^* -\widetilde{\theta} \right\rangle + \frac{\kappa}{2}\norm{\widetilde{\theta}-\theta^*}^2_2\\
        & \geq \left\langle \nabla_\theta A(\widetilde{\theta}), \theta^* -\widetilde{\theta} \right\rangle + \frac{\kappa}{2}\norm{\widetilde{\theta}-\theta^*}^2_{\Sigma'_{\calD_n,M}}\\
         & = - \expp_{(x,y)\sim d^{\widetilde{\theta}}_\rho}\left[ \left\langle \psi'(x,y),\widetilde{\theta}-\theta^*\right\rangle \right] + \frac{\kappa}{2}\norm{\widetilde{\theta}-\theta^*}^2_{\Sigma'_{\calD_n,M}}~,
    \end{align*}
    using a similar argument as in the proof of Theorem \ref{thm:main_dpo_theorem}. Thus, substituting in the original gap expression, we obtain
    \begin{align*}
        G\left(d^{\widetilde{\theta}}_\rho\right)  & \geq D\left(d^{\widetilde{\theta}}_\rho\right) + \frac{\beta \kappa}{2}\norm{\widetilde{\theta}-\theta^*}^2_2 \\
        & \geq D\left(d^{\widetilde{\theta}}_\rho\right) + \Omega\left(\frac{ d_M+1}{\beta n}\right)~,
    \end{align*}
    for any $n\geq O\left(tr((\Sigma'_{\calD})^\dagger)/(S_M(B')^2)\right)$, by Theorem \ref{thm:btl_mle_minimax}.
\end{proof}

\section{Convergence of Gradient Descent for RLHF Reward Learning (Section \ref{sec:approximate_optimization})}\label{sec:rlhf_mle_convergence}

In this section, we will prove convergence bounds for projected gradient descent for the RLHF reward learning phase. Recall that, the projected gradient update rule for the reward learning phase is given as
\begin{align*}
    \omega_{t+1} = \underset{\omega:\norm{\omega}_2\leq F}{\proj}\left(\omega_t - \eta \nabla_\omega \calL^\omega_\textnormal{RLHF}(\calD_n)\right)~.
\end{align*}
First, we show Lipschitzness and smoothness of the loss function.

\begin{lemma}\label{lem:mle_rlhf_lipshcitz}
    The RLHF reward learning objective $\calL^\omega_\textnormal{RLHF}(\calD_n)$ is $2\exp(2F)$-Lipschitz and $2\exp(2F)$-smooth. 
\end{lemma}
\begin{proof}
    Note that the gradient of $\calL^\omega_\textnormal{RLHF}(\calD_n)$ satisfies
    \begin{align*}
        \norm{\nabla_\omega \calL^\omega_\textnormal{RLHF}(\calD_n)}_2 & = \norm{\frac{1}{n}\sum_{(x,y^w,y^l)\sim\calD_n} \nabla_\omega \log\left( 1+ \exp\left( \omega^\top\left( \phi(x,y^w)-\phi(x,y^l)\right)\right)\right)}_2 \\
            & \leq \frac{1}{n}\sum_{(x,y^w,y^l)\sim\calD_n} \frac{\exp\left( \omega^\top\left( \phi(x,y^w)-\phi(x,y^l)\right)\right)}{1+ \exp\left( \omega^\top\left( \phi(x,y^w)-\phi(x,y^l)\right)\right)} \norm{\left( \phi(x,y^w)-\phi(x,y^l)\right)}_2\\
        & \leq \exp (2F) \norm{\phi(x,y^w)-\phi(x,y^l)}_2\\
            & \leq 2\exp(2F)~.
    \end{align*}
    Moreover, the Hessian of $\calL^\omega_\textnormal{RLHF}(\calD_n)$ satisfies
    \begin{align*}
        & \norm{\nabla^2_\omega\calL^\omega_\textnormal{RLHF}(\calD_n)}_2  = \norm{ \frac{1}{n}\sum_{(x,y^w,y^l)\sim\calD_n} \nabla_\omega \frac{\exp\left( \omega^\top\left( \phi(x,y^w)-\phi(x,y^l)\right)\right)}{1+ \exp\left( \omega^\top\left( \phi(x,y^w)-\phi(x,y^l)\right)\right)}\left( \phi(x,y^w)-\phi(x,y^l)\right) }_2 \\
            & \leq \frac{1}{n}\sum_{(x,y^w,y^l)\sim\calD_n}  \frac{\exp\left( \omega^\top\left( \phi(x,y^w)-\phi(x,y^l)\right)\right)}{\left(1+ \exp\left( \omega^\top\left( \phi(x,y^w)-\phi(x,y^l)\right)\right)\right)^2}\norm{\left( \phi(x,y^w)-\phi(x,y^l)\right)\left( \phi(x,y^w)-\phi(x,y^l)\right)^\top }_2 \\ 
        & \leq \exp(2F) \norm{\left( \phi(x,y^w)-\phi(x,y^l)\right)}_2\\
            & \leq 2\exp(2F)~.
    \end{align*}
    The result follows.
\end{proof}

Next, we show that $\calL^\omega_\textnormal{RLHF}(\calD_n)$ satisfies the PL condition \citep{karimi2016linear} defined below.
\begin{definition}\label{def:pl_condition}
    A function $\calL$ is said to satisfy the PL condition with coefficient $C_{PL}>0$ if, for every $\omega$ in the domain of $\calL$, we have
    \begin{align*}
        \norm{\nabla_\omega \calL(\omega)}^2_2 \geq C_{PL}\left(\calL(\omega) - \calL^*\right)~,
    \end{align*}
    where $\calL^*$ is the minimum value of $\calL$ in its domain.
\end{definition}

\begin{lemma}\label{lem:pl_mle_rlhf}
    Let $L_2=2\exp(2F)$ and 
    \begin{align*}
        C_{PL} =\frac{\exp(-2F) \xi(1+\exp(-2F))}{n(1+\exp(2F)^2}~,
    \end{align*}
    where 
    \begin{align*}
        0 < \xi = \min_{(x,y^w,y^l)\sim\calD_n} \norm{\phi(x,y^w)-\phi(x,y^l)}^2_2~.
    \end{align*}
    Then, we have
    \begin{align*}
        \frac{1}{2}\norm{\nabla_\omega\calL^\omega_\textnormal{RLHF}(\calD_n)}^2 \geq C_{PL}\left(\calL^\omega_\textnormal{RLHF}(\calD_n)-\calL^*_\textnormal{RLHF}(\calD_n)\right)~.
    \end{align*}
\end{lemma}
\begin{proof}
    Due to the assumption on the features and parameter vectors, we have
    \begin{align*}
        \frac{1}{2}\norm{\nabla_\omega\calL^\omega_\textnormal{RLHF}(\calD_n)}^2 & \geq \frac{1}{2n^2}\sum_{(x,y^w,y^l)\sim\calD_n} \frac{\exp(-2F)}{1+\exp(2F)}\min_{(x,y^w,y^l)\sim\calD_n}\norm{\phi(x,y^w)-\phi(x,y^l)}^2_2\\
        & \geq \frac{\exp(-2F) \xi}{n(1+\exp(2F))}~.
    \end{align*}
    On the other hand, note that, for some $\omega^*$ such that $\norm{\omega^*}_2\leq F$, we have
    \begin{align*}
        \calL^\omega_\textnormal{RLHF}(\calD_n) -\calL^*_\textnormal{RLHF}(\calD_n) & = \expp_{(x,y^w,y^l)\sim\calD_n}\Big[ \log\left(1+\exp\left( \omega^\top \left( \phi(x,y^w)-\phi(x,y^l)\right)\right)\right) \\& \quad\quad - \log\left(1+\exp\left( (\omega^*)^\top \left( \phi(x,y^w)-\phi(x,y^l)\right)\right)\right) \Big] \\
            & = \frac{1}{n}\sum_{(x,y^w,y^l)\sim\calD_n} \log\frac{1+\exp\left( \omega^\top \left( \phi(x,y^w)-\phi(x,y^l)\right)\right)}{1+\exp\left( (\omega^*)^\top \left( \phi(x,y^w)-\phi(x,y^l)\right)\right)} \\
        & \leq \frac{1}{n}\sum_{(x,y^w,y^l)\sim\calD_n}\left( \frac{1+\exp\left( \omega^\top \left( \phi(x,y^w)-\phi(x,y^l)\right)\right)}{1+\exp\left( (\omega^*)^\top \left( \phi(x,y^w)-\phi(x,y^l)\right)\right)} - 1\right) \\
        & \leq \frac{1}{n}\sum_{(x,y^w,y^l)\sim\calD_n} \frac{1+\exp(2F)}{1+\exp(-2F)}\\
            & \leq \frac{1+\exp(2F)}{1+\exp(-2F)}~,
    \end{align*}
    where the third inequality follows from $\log x \leq x-1$, for $x>0$. 
    Solving for $C_{PL}$ the equation 
    \begin{align*}
        C_{PL} \frac{1+\exp(2F)}{1+\exp(-2F)} = \frac{\exp(-2F) \xi}{n(1+\exp(2F)}~,
    \end{align*}
    we obtain
    \begin{align*}
        C_{PL}= \frac{\exp(-2F) \xi(1+\exp(-2F))}{n(1+\exp(2F)^2}~.
    \end{align*}
\end{proof}

Now, we are ready to state the convergence result for gradient descent.
\rlhfmleconvergence*
\begin{proof}
    The projected gradient descent rule is equivalent to the proximal gradient update \citep{karimi2016linear}, given as
    \begin{align*}
        \omega_{t+1}= \arg\min_\omega \left( \left\langle \nabla_\omega\calL^{\omega_t}_\textnormal{RLHF}(\calD_n),\omega-\omega_t\right\rangle + \frac{L_2}{2}\norm{\omega-\omega_t}^2_2 + g(\omega)-g(\omega_t)\right)~,
    \end{align*}
    where $g(\omega)=0$, if $\norm{\omega}_2\leq F$ and $\infty$ otherwise. To see that, note that we can equivalently write the above as
    \begin{align*}
        \omega_{t+1} & = \arg\min_\omega \norm{\omega-\left(\omega_t-\frac{1}{L_2}\nabla_\omega\calL^{\omega_t}_\textnormal{RLHF}(\calD_n)\right)}_2, \; \textnormal{s.t.}\; \norm{\omega}_2\leq F \\
            & = \proj_{\omega:\norm{\omega}_2\leq F}\left(\omega_t-\frac{1}{L_2}\nabla_\omega \calL^{\omega_t}_\textnormal{RLHF}(\calD_n)\right)~.
    \end{align*}
    Theorem 5 of \citep{karimi2016linear} gives us linear rates of convergence for projected gradient descent under the \textit{proximal PL} condition. This condition is shown in Appendix G of \citep{karimi2016linear} to be equivalent to the following condition. 
    \begin{definition}
        A function $F$ is said to satisfy the Kurdyka-Lojasiewicz condition with exponent $1/2$ if there exists $C >0$ such that
        \begin{align*}
            \min_{s\in\partial F(\omega)}\norm{s}^2_2 \geq C(F(\omega)-F_*)~,
        \end{align*}
        where $\partial F(\omega)$ is the Frechet subdifferential of $F$ at $\omega$ and $F_*$ denotes the minimum value of $F$. 
    \end{definition}
    Note that, in our case we have $F(\omega)=\calL^\omega_\textnormal{RLHF}(\calD_n)-g(\omega)$ and the Frechet subdifferential of this function in the domain $\{ \omega: \norm{\omega}_2\leq F\}$ only contains $\nabla_\omega \calL^\omega_\textnormal{RLHF}(\calD_n)$. Thus, the above condition is equivalent to the PL condition. As a consequence Theorem 5 of \citep{karimi2016linear} implies that
    \begin{align*}
        \calL^{\omega_t}_\textnormal{RLHF}(\calD_n)-\calL^*_\textnormal{RLHF}(\calD_n) \leq \left(1-\frac{C_{PL}}{L_2}\right)^t\left(\calL^{\omega_0}_\textnormal{RLHF}(\calD_n)-\calL^*_\textnormal{RLHF}(\calD_n)\right)~.
    \end{align*}

    Now, recalling the Hessian of our loss, note that for any non-zero vector $v\in\mathbb{R}^{d_P}$, we have
    \begin{align*}
        & v^\top\nabla^2_\omega\calL^\omega_\textnormal{RLHF}(\calD_n) v\\ & = v^\top\left( \frac{1}{n}\sum_{(x,y^w,y^l)\sim\calD_n}  \frac{\exp\left( \omega^\top\left( \phi(x,y^w)-\phi(x,y^l)\right)\right)}{(1+ \exp\left( \omega^\top\left( \phi(x,y^w)-\phi(x,y^l)\right)\right))^2}\left( \phi(x,y^w)-\phi(x,y^l)\right)\left( \phi(x,y^w)-\phi(x,y^l)\right)^\top\right)v \\
        & \geq \frac{\exp(-F)}{(1+\exp(2F))^2}\norm{v}^2_{\Sigma_{\calD_n,R}}~.
    \end{align*}
    Thus, $\calL^\omega_\textnormal{RLHF}(\calD_n)$ is $\frac{\exp(-F)}{(1+\exp(2F))^2}$-strongly convex with respect to the semi-norm $\norm{\cdot}_{\Sigma_{\calD_n,R}}$ around $\omega^*$. Therefore, for any $\omega$ we have
    \begin{align*}
        \calL^\omega_\textnormal{RLHF}(\calD_n)-\calL^*_\textnormal{RLHF}(\calD_n) & \geq \left\langle \nabla_\omega \calL^*_\textnormal{RLHF}(\calD_n),\omega-\omega^*\right\rangle + \frac{\exp(-2F)}{(1+\exp(2F))^2}\norm{\omega-\omega^*}_{\Sigma_{\calD_n,R}}^2\\
            &\geq \frac{\exp(-2F)}{(1+\exp(2F))^2}\norm{\omega-\omega^*}_{\Sigma_{\calD_n,R}}^2~.
    \end{align*}
    Putting everything together, we obtain
    \begin{align*}
        \norm{\omega-\omega^*}_{\Sigma_{\calD_n,R}}^2 \leq O\left( \left(1-\frac{1}{n}\right)^t\right)~.
    \end{align*}
\end{proof}

\section{Convergence of Natural Policy Gradient for RLHF (Section \ref{sec:approximate_optimization})}\label{sec:rlhf_convergence_proof}

In this section, we prove fast convergence rates for the a version of the natural policy gradient algorithm with loglinear policy class.
We begin by deriving the gradient of the KL-regularized objective. Throughout, let us fix reward function $r$ and dataset ${\calD_n}$.

\begin{lemma}\label{lem:gradient_of_generative}
    Given reward $r$ and policy $\pi_\theta$, the gradient of $\calV^{\pi_\theta}_r({\calD_n})$, for the softmax policy class can be written as
    \begin{align*}
        \nabla_\theta\calV^{\pi_\theta}_r({\calD_n}) = \frac{1}{n}\sum_{x\in{\calD_n}}\sum_{y\in\calY} \pi_\theta(y|x)\left( r(x,y) -\beta \log \left( \frac{\pi_\theta(y|x)}{\mu(y|x)}\right)\right) \overline{\psi}_\theta (x,y)~,
    \end{align*}
    where 
    \begin{align*}
        \overline{\psi}_\theta (x,y)= \psi(x,y)-\sum_{y'\in\calY}\pi_\theta(y'|x)\psi(x,y')~.
    \end{align*}
\end{lemma}
\begin{proof}
    First, note that, for softmax policies with linear action preferences, we have, for any given $(x,y)$,
    \begin{align*}
        & \nabla_\theta \pi_\theta(y|x)  = \nabla_\theta \frac{\exp (\theta^\top \psi(x,y))}{\sum_{y'\in\calY}\exp (\theta^\top \psi(x,y'))}\\
            & = \frac{\exp (\theta^\top \psi(x,y))\sum_{y'\in\calY}\exp (\theta^\top \psi(x,y'))}{(\sum_{y'\in\calY}\exp (\theta^\top \psi(x,y')))^2}\psi(x,y)  - \frac{\exp (\theta^\top \psi(x,y))\sum_{y'\in\calY}\exp(\theta^\top\psi(x,y'))\psi(x,y')}{(\sum_{y'\in\calY}\exp (\theta^\top \psi(x,y')))^2}\psi(x,y) \\
        & = \pi_\theta(y|x)\left( \psi(x,y)-\expp_{y'\sim \pi_\theta(\cdot|x)}[\psi(x,y')]\right)~.
    \end{align*}
    On the other hand, for the regularizer, we have
    \begin{align*}
        \nabla_\theta \log\left(\frac{\pi_\theta(y|x)}{\mu(y|x)}\right) & = \nabla_\theta \left( \theta^\top\psi(x,y) - \log \overline{Z}_\theta(x) - \log \mu(y|x)\right) \\
            & = \psi(x,y) - \frac{1}{\overline{Z}_\theta(x)}\sum_{y'\in\calY}\exp(\theta^\top\psi(x,y))\psi(x,y)\\
        & = \psi(x,y) - \expp_{y'\sim \pi_\theta(\cdot|x)}[\psi(x,y')]~,
    \end{align*}
    where
    \begin{align*}
        \overline{Z}_\theta(x) = \sum_{y\in\calY}\exp(\theta^\top\psi(x,y))~.
    \end{align*}
    Using the definition of $\calV^{\pi_\theta}_r({\calD_n})$ and the above derivations, we have
    \begin{align*}
        \nabla_\theta \calV^{\pi_\theta}_r({\calD_n}) & =  \frac{1}{n}\sum_{x\in{\calD_n}}\nabla_\theta\left(\sum_{y\in\calY}\pi_\theta(y|x)\left( r(x,y) -\beta \log \left( \frac{\pi_\theta(y|x)}{\mu(y|x)}\right)\right) \right)\\
            & = \frac{1}{n}\sum_{x\in{\calD_n}}\sum_{y\in\calY} \pi_\theta(y|x)\left( r(x,y) -\beta \log \left( \frac{\pi_\theta(y|x)}{\mu(y|x)}\right)\right)\left( \psi(x,y)-\sum_{y'\in\calY}\pi_\theta(y'|x)\psi(x,y')\right) \\ & \quad \quad \quad - \frac{1}{n} \sum_{x\in{\calD_n}}\sum_{y\in\calY}\pi_\theta(y|x) \left( \psi(x,y)  - \sum_{y'\in\calY}\pi_\theta(y',x)\psi(x,y')\right) \\
        & =  \frac{1}{n}\sum_{x\in{\calD_n}}\sum_{y\in\calY} \pi_\theta(y|x)\left( r(x,y) -\beta \log \left( \frac{\pi_\theta(y|x)}{\mu(y|x)}\right)\right)\left( \psi(x,y)-\sum_{y'\in\calY}\pi_\theta(y'|x)\psi(x,y')\right)  \\ & \quad \quad \quad - \frac{1}{n} \sum_{x\in{\calD_n}}\sum_{y\in\calY}\pi_\theta(y|x) \psi(x,y)  + \frac{1}{n}\sum_{x\in{\calD_n}} \sum_{y'\in\calY}\pi_\theta(y'|x)\psi(x,y') \sum_{y\in\calY}\pi_\theta(y|x) \\
            & =\frac{1}{n}\sum_{x\in{\calD_n}}\sum_{y\in\calY} \pi_\theta(y|x)\left( r(x,y) -\beta \log \left( \frac{\pi_\theta(y|x)}{\mu(y|x)}\right)\right)\left( \psi(x,y)-\sum_{y'\in\calY}\pi_\theta(y'|x)\psi(x,y')\right) \\
        & = \frac{1}{n}\sum_{x\in{\calD_n}}\sum_{y\in\calY} \pi_\theta(y|x)\left( r(x,y) -\beta \log \left( \frac{\pi_\theta(y|x)}{\mu(y|x)}\right)\right) \overline{\psi}_\theta (x,y)~.
    \end{align*}
\end{proof}

Now we consider the following gradient update rule. Let $\Psi_n\in\mathbb{R}^{d_P\times nY}$ denote the feature matrix corresponding to ${\calD_n}$ with columns $\psi(x,y)$, for every $(x,y)\in{\calD_n}\times\calY$. We will assume that $\Psi_n$ is full column rank.  For every $t\geq 0$, let 
\begin{align}
    \theta_{t+1}=\theta_t + \eta' \left(\Psi_n \Psi_n^\top\right)^\dagger\nabla_\theta \calV^{\pi_\theta}_r({\calD_n})~. \label{eq:regularized_gradient_update}
\end{align}
First, we will expand this gradient expression in the following lemma. 

\begin{lemma}\label{lem:alternative_gradient_update}
    For every $t\geq0$, the gradient update can be written as 
    \begin{align*}
        \theta_{t+1}= \theta_t + \frac{\eta'}{n}\left( \Psi_n \Psi_n^\top\right)^{\dagger}\Psi_n H(\pi_{\theta_t})\alpha_t~,
    \end{align*}
    where $\mathbb{R}^{nY\times nY} \ni H(\pi)=diag(\pi)-M(\pi)$, for any policy $\pi$, with $M(\pi)$ being a blog-diagonal matrix composed of $n$ blocks $\pi(x)\pi(x)^\top\in\mathbb{R}^{Y\times Y}$, with $\pi(x)=[\pi(y|x)]_{y\in\calY}$, for every $x\in{\calD_n}$, and 
    \begin{align*}
        \alpha_t = \beta\Psi_n^\top\theta_t - r - \beta\log\mu - \frac{\left(  \beta\Psi_n^\top\theta_t - r -\beta\log\mu\right)^\top \mathbf{1}}{Y}\cdot \mathbf{1}~.
    \end{align*}
    Here, $r=[r(x,y)]_{(x,y)\in{\calD_n}\times\calY}$ denotes the reward vector, and $\log\mu=[\log\mu(y|x)]_{(x,y)\in{\calD_n}\times\calY}$ denotes the vector of log values for te reference policy.
\end{lemma}
\begin{proof}
    Using the gradient update and Lemma \ref{lem:gradient_of_generative}, we have
    \begin{align*}
        \theta_{t+1} & = \theta_t + \eta'\left(\Psi_n \Psi_n^\top\right)^\dagger\nabla_\theta\calV^{\pi_\theta}_r({\calD_n})\\
            & = \theta_t + \eta'\left(\Psi_n \Psi_n^\top\right)^\dagger\frac{1}{n}\sum_{x\in{\calD_n}}\sum_{y\in\calY} \pi_{\theta_t}(y|x)\left( r(x,y) -\beta \log \left( \frac{\pi_{\theta_t}(y|x)}{\mu(y|x)}\right)\right)\overline{\psi}_{\theta_t} (x,y)\\
        & = \theta_t + \frac{\eta'}{n}\left(\Psi_n \Psi_n^\top\right)^\dagger\sum_{x\in{\calD_n}}\sum_{y\in\calY}\pi_{\theta_t}(y|x)\Big( r(x,y) -  \beta\theta_t^\top \psi(x,y) + \beta\log \overline{Z}_{\theta_t}(x) + \beta\log \mu(y|x)\Big)\overline{\psi}_{\theta_t}(x,y)\\
        & = \theta_t + \frac{\eta'}{n}\left(\Psi_n \Psi_n^\top\right)^\dagger\sum_{x\in{\calD_n}}\sum_{y\in\calY}\pi_{\theta_t}(y|x)\Big( r(x,y) - \beta \theta_t^\top\psi(x,y) + \beta \log \mu(y|x)\Big) \overline{\psi}_{\theta_t}(x,y)\\
            & = \theta_t + \frac{\eta'}{n} \left(\Psi_n \Psi_n^\top\right)^\dagger\Psi_n H(\pi_{\theta_t})\left( r -\beta \Psi_n^\top\theta_t + \beta \log \mu \right)\\
        & = \theta_t - \frac{\eta'}{n} \left(\Psi_n \Psi_n^\top\right)^\dagger\Psi_n H(\pi_{\theta_t}) \left(  \beta\Psi_n^\top\theta_t - r - \beta\log\mu - \frac{\left(  \beta\Psi_n^\top\theta_t - r -\beta\log\mu\right)^\top \mathbf{1}}{Y}\cdot \mathbf{1}\right)~,
    \end{align*}
    where the fourth equality follows from the observation that 
    \begin{align*}
        \beta\sum_{y}\pi_\theta(y|x)\log\overline{Z}_\theta(x)\overline{\psi}_\theta(x,y) = \beta\log\overline{Z}_\theta(x)\sum_y \pi_\theta(x,y)\left(\psi(x,y) -\expp_{y'\sim\pi_\theta(\cdot|x)}[\psi(x,y')]\right) = 0~,
    \end{align*}
    while the last equality follows from the fact that $H(\pi_\theta) c\mathbf{1}=\mathbf{0}$, for any constant $c$.
\end{proof}

Now we will express $\alpha_{t+1}$ in a different way using the above derivation. 
\begin{lemma}\label{lem:contraction_of_alpha}
    For every $t\geq 0$, we have
    \begin{align*}
        \alpha_{t+1} = \left( I - (\eta'\beta/n) H(\pi_{\theta_t})\right) \alpha_t
    \end{align*}
\end{lemma}
\begin{proof}
    Observe that
    \begin{align*}
         & \alpha_{t+1} = \beta\Psi_n^\top\theta_{t+1} - r - \beta\log\mu - \frac{\left(  \beta\Psi_n^\top\theta_{t+1} - r -\beta\log\mu\right)^\top \mathbf{1}}{Y}\cdot \mathbf{1} \\
            & = \beta\Psi_n^\top\theta_t - r - \beta\log\mu - \frac{\left(  \beta\Psi_n^\top\theta_{t} - r -\beta\log\mu\right)^\top \mathbf{1}}{Y}\cdot \mathbf{1} + \beta \Psi_n^\top (\theta_{t+1}-\theta_t) -\frac{\beta(\Psi_n^\top(\theta_t-\theta_{t+1}))^\top \mathbf{1}}{Y}\cdot\mathbf{1}\\
            & = \beta\Psi_n^\top\theta_t - r - \beta\log\mu - \frac{\left(  \beta\Psi_n^\top\theta_{t} - r -\beta\log\mu\right)^\top \mathbf{1}}{Y}\cdot \mathbf{1} -\frac{\beta(\Psi_n^\top(\theta_t-\theta_{t+1}))^\top \mathbf{1}}{Y}\cdot\mathbf{1} \\ & \quad \quad - \beta \Psi_n^\top \left(\frac{\eta'}{n} \left( \Psi_n \Psi_n^\top\right)^{\dagger} \Psi_n H(\pi_{\theta_t}) \left(  \beta\Psi_n^\top\theta_t - r - \beta\log\mu - \frac{\left(  \beta\Psi_n^\top\theta_t - r -\beta\log\mu\right)^\top \mathbf{1}}{Y}\cdot \mathbf{1}\right)\right) \\
            & = \left( I - (\beta\eta'/n) \Psi_n^\top\left(\Psi_n\Psi_n^\top\right)^{\dagger}\Psi_n H(\pi_{\theta_t})\right)  \left(  \beta\Psi_n^\top\theta_t - r - \beta\log\mu - \frac{\left(  \beta\Psi_n^\top\theta_t - r -\beta\log\mu\right)^\top \mathbf{1}}{Y}\cdot \mathbf{1}\right) \\& \quad\quad\quad -\frac{\beta(\Psi_n^\top(\theta_t-\theta_{t+1}))^\top \mathbf{1}}{Y}\cdot\mathbf{1}\\
            & =  \left( I - (\beta\eta'/n) H(\pi_{\theta_t})\right)  \left(  \beta\Psi_n^\top\theta_t - r - \beta\log\mu - \frac{\left(  \beta\Psi_n^\top\theta_t - r -\beta\log\mu\right)^\top \mathbf{1}}{Y}\cdot \mathbf{1}\right) \\ & \quad\quad\quad -\frac{\beta(\Psi_n^\top(\theta_t-\theta_{t+1}))^\top \mathbf{1}}{Y}\cdot\mathbf{1}~,
    \end{align*}
    where the third equality uses Lemma \ref{lem:alternative_gradient_update} and the last equality follows from the fact that 
    \begin{align*}
        \Psi_n^\top\left(\Psi_n\Psi_n^\top\right)^{\dagger}\Psi_n = \Psi_n^\top (\Psi_n^\dagger)^\top \Psi_n^\dagger\Psi_n =(\Psi_n^\dagger\Psi_n)^\top\Psi_n^\dagger\Psi_n = I~,
    \end{align*}
    since we assume $\Psi_n$ to be full column rank. For the last term in the derivation above, we have
    \begin{align*}
        & \frac{\beta(\Psi_n^\top(\theta_t-\theta_{t+1}))^\top \mathbf{1}}{Y}\cdot\mathbf{1} \\ &  = \frac{\beta\eta'}{nY}\left( \Psi_n^\top\left(\Psi_n\Psi_n^\top\right)^{\dagger}\Psi_n H(\pi_{\theta_t}) \left(  \beta\Psi_n^\top\theta_t - r - \beta\log\mu - \frac{\left(  \beta\Psi_n^\top\theta_t - r -\beta\log\mu\right)^\top \mathbf{1}}{Y}\cdot \mathbf{1}\right)\right)^\top \mathbf{1}\cdot \mathbf{1}\\
            & = \frac{\beta\eta'}{nY}\left( H(\pi_{\theta_t}) \left(  \beta\Psi_n^\top\theta_t - r - \beta\log\mu - \frac{\left(  \beta\Psi_n^\top\theta_t - r -\beta\log\mu\right)^\top \mathbf{1}}{Y}\cdot \mathbf{1}\right)\right)^\top \mathbf{1}\cdot \mathbf{1}\\
        & = \frac{\beta\eta'}{nY} \left(  \beta\Psi_n^\top\theta_t - r - \beta\log\mu - \frac{\left(  \beta\Psi_n^\top\theta_t - r -\beta\log\mu\right)^\top \mathbf{1}}{Y}\cdot \mathbf{1}\right)^\top H(\pi_{\theta_t})^\top \mathbf{1}\cdot\mathbf{1}\\
            & = \mathbf{0}~,
    \end{align*}
    where we have used the fact that $H(\pi_\theta)^\top\mathbf{1}=\mathbf{0}$. The result follows. 
\end{proof}

Next, we will decompose the matrix $H(\pi_\theta)$ into simpler pieces and explore its structure. 

\begin{lemma}\label{lem:spectrum_of_H}
    The eigenvalues of $H(\pi_\theta)$ satisfy the following. The lowest eigenvalue is $\lambda_1=0$ with multiplicity $n$ with corresponding eigenvectors $e_i$, for each $i\in[n]$, where $e_i$ are the vectors of ones in indices $Y(i-1)$ to $Yi$ and zeros everywhere else. Furthermore, we have that $\min_{x\in{\calD_n},y\in\calY}\pi(y|x)\leq \lambda_2$ and $\lambda_{\max}\leq \max_{x\in{\calD_n},y\in\calY}\pi(y|x)$.
\end{lemma}
\begin{proof}
    Again, given $x\in{\calD_n}$, let us denote by $\pi(x)$ the vector $[\pi(y|x)]_{y\in\calY}$. Note that, for any policy $\pi$, we can write $H(\pi)$ as
    \begin{align*}
        H(\pi) & = 
        \begin{bmatrix}
            diag(\pi(x_1)) & 0 & \ldots & 0 \\
            0 & diag(\pi(x_2)) & \ldots & 0 \\
            \vdots & \vdots & \ddots & \vdots \\
            0 & \ldots & 0 & diag(\pi(x_n)) 
        \end{bmatrix}
        \\ & \quad\quad -
        \begin{bmatrix}
            \pi(x_1)\pi(x_1)^\top & 0 & \ldots & 0 \\
            0 & \pi(x_2)\pi(x_2)^\top & \ldots & 0 \\
            \vdots & \vdots & \ddots & \vdots \\
            0 & \ldots & 0 & \pi(x_n)\pi(x_n)^\top 
        \end{bmatrix} \\
        & = \begin{bmatrix}
            H(\pi(x_1)) & 0 & \ldots & 0 \\
            0 & H(\pi(x_2)) & \ldots & 0 \\
            \vdots & \vdots & \ddots & \vdots \\
            0 & \ldots & 0 & H(\pi(x_n)) 
        \end{bmatrix}
    \end{align*}
    Now, given $x\in{\calD_n}$, Lemma 22 of \cite{mei2020on} states that the spectrum of $H(\pi(x))$ satisfies $\lambda_1=0$ with corresponding eigenvector $\mathbf{1}\in\mathbb{R}^Y$, and 
    \begin{align*}
        \pi(y_{i-1}|x) \leq \lambda_i \leq \pi(y_i|x)~,
    \end{align*}
    for each $2\leq i\leq Y$, where $\lambda_1\leq\ldots\leq\lambda_Y$ and $\pi(y_1|x)\leq\ldots\leq\pi(y_Y|x)$. Furthermore, it is known that the spectrum of a block diagonal matrix is composed of the eigenvalues of each block, counting multiplicities. Thus, we have that $0$ is the lowest eigenvalue of $H(\pi)$ occurring with multiplicity $n$. The rest follows. 
\end{proof}

\begin{lemma}\label{lem:norm_of_v'}
    Let $v\in\mathbb{R}^{nY}$ be any given vector. Then, we have that 
    \begin{align*}
        \norm{\left( I -  H(\pi)\right)\left( v - \frac{v^\top\mathbf{1}}{Y}\mathbf{1}\right)}_2\leq \left(1-\min_{x\in{\calD_n},y\in\calY}\pi(y|x)\right)\norm{v - \frac{v^\top\mathbf{1}}{Y}\mathbf{1}}_2
    \end{align*}
\end{lemma}
\begin{proof}
    First, for every $i\in[n]$, let $e_i\in\mathbb{R}^{nY}$ denote a vector with entries $1$ at the indices $Y(i-1)$ to $Yi$ and $0$ everywhere else. Note that 
    \begin{align*}
        \sum_{i\leq n}e_i = \mathbf{1}\in\mathbb{R}^{nY}~.
    \end{align*}
    Next, let $v(j)$ denote an $nY$-dimensional vector with entries $v_k$ for each $Y(j-1)\leq k\leq Yj$. 
    Since $H(\pi)$ is diagonalizable, as a symmetric matrix, any vector can be represented as a linear combination of its eigenvectors. Since $H(\pi)$ is symmetric, this representation is unique. Now, by Lemma \ref{lem:spectrum_of_H}, note that
    \begin{align*}
        v=\sum_{j\leq nY} a_ju_j = \sum_{j\leq n}a_je_j + \sum^{nY}_{k=n+1}a_ku_k = \sum_{j\leq n}\frac{v(j)^\top\mathbf{1}_Y}{Y}e_j + \sum^{nY}_{k=n+1}a_ku_k = \frac{v^\top \mathbf{1}}{Y}\mathbf{1} + \sum^{nY}_{k=n+1}a_ku_k~,
    \end{align*}
    where $(u_{i,j})_{i\leq n,j\leq Y}$ is the eigenvector basis, with the first $n$ eigenvectors being $e_i$, for $i\leq n$. Thus, we have that
    \begin{align*}
       v' =  v- \frac{v^\top \mathbf{1}}{Y}\mathbf{1} =  \sum^{nY}_{k=n+1}a_k u_k~,
    \end{align*}
    with $a_{n+1} >0$, and that
    \begin{align*}
        \norm{v'}_2 = \sum^{nY}_{j=n+1}a_j^2~.
    \end{align*}
    From the above, we obtain
    \begin{align*}
        \left( I -H(\pi)\right)v'= \sum^{nY}_{j=n+1}a_j(1-\lambda_j)u_j~,
    \end{align*}
    and thus, by Lemma \ref{lem:spectrum_of_H},
    \begin{align*}
        \norm{ \left( I -H(\pi)\right)v'}_2 & = \sqrt{\sum^{nY}_{j=n+1}a^2_j(1-\lambda_j)^2}\\
                & \leq \sqrt{(1-\lambda_{n+1})\left(\sum^{nY}_{j=n+1}a^2_j\right)}\\
                & = (1-\lambda_{n+1})\norm{v'}_2\\
                & \leq \left( 1-\min_{x\in{\calD_n},y\in\calY}\pi(y|x)\right)\norm{v'}_2~.
    \end{align*}
\end{proof}

\begin{lemma}\label{lem:exponential_contraction}
    Suppose $\eta'$ and $\beta$ are such that $\eta'\beta/n \leq 1$. For the loglinear policy class, for every $t\geq 1$, 
    \begin{align*}
        \norm{\alpha_t}_2 \leq \frac{2\left( \beta B + 1 \right)\sqrt{Y}}{\exp\left(\eta'\beta\sum^{t-1}_{s=1}\min_{x\in{\calD_n},y\in\calY}\pi_\theta(y|x)\right)}~.
    \end{align*}
\end{lemma}
\begin{proof}
    By Lemma \ref{lem:contraction_of_alpha} and Lemma \ref{lem:norm_of_v'}, for all $t\geq 1$,
    \begin{align*}
        & \norm{\left( I - (\eta'\beta/n) H(\pi_{\theta_{t+1}})\right)\alpha_{t+1}}_2 \leq \left(1-(\eta'\beta/n)\min_{x\in{\calD_n},y\in\calY}\pi_{\theta_t}(y|x)\right)\norm{\alpha_t}_2 \\
                & \leq \frac{1}{\exp\left((\eta'\beta/n)\min_{x\in{\calD_n},y\in\calY}\pi_{\theta_t}(y|x)\right)}\norm{\alpha_t}_2 \\
                & \leq \frac{1}{\exp\left((\eta'\beta/n)\min_{x\in{\calD_n},y\in\calY}\pi_{\theta_t}(y|x)\right)} \left( 1-\min_{x\in{\calD_n},y\in\calY}\pi_{\theta_{t-1}}(y|x)\right) \norm{\alpha_{t-1}}_2\\
                & \leq \frac{1}{\exp\left((\eta'\beta/n)\sum^{t}_{s=t-1}\min_{x\in{\calD_n},y\in\calY}\pi_{\theta_s}(y|x)\right)}\norm{\alpha_{t-1}}_2\\
                & \leq \frac{1}{\exp\left((\eta'\beta/n)\sum^{t}_{s=1}\min_{x\in{\calD_n},y\in\calY}\pi_{\theta_s}(y|x)\right)}\norm{\alpha_1}_2~.
    \end{align*}
    For the first iteration, observe that
    \begin{align*}
        \norm{\alpha_1}_2 & = \norm{\beta\Psi_n^\top\theta_0-r-\beta\log\mu - \frac{\left(\beta\Psi_n^\top\theta_0-r-\beta\log\mu\right)^\top\mathbf{1}}{Y}\mathbf{1}}_2 \\
            & \leq \norm{\beta\Psi_n^\top\theta_0-r-\beta\log\mu}_2 + \frac{1}{\sqrt{Y}}\norm{\beta\Psi_n^\top\theta_0-r-\beta\log\mu}_2\norm{\mathbf{1}}_2 \\
            & = 2 \norm{\beta\Psi_n^\top\theta_0-r-\beta\log\mu}_2 \\
            & \leq 2\left( \beta \norm{\Psi_n^\top\theta_0}_\infty + 1\right)\sqrt{Y}\\
            & \leq 2\left( \beta B + 1 \right)\sqrt{Y}~,
    \end{align*}
    where the second inequality follows from the triangle inequality, the Cauchy-Schwarz inequality and the fact that rewards lie in the unit ball, while the last follows from the fact that the features lie in a unit subspace of $\mathbb{R}^{d_P}$, while $\norm{\theta_0}_\infty \leq B$. The result follows. 
\end{proof}

\begin{lemma}\label{lem:lower_bound_pi}
    There exists a constant $C= C(\beta,Y,B)>0$, such that, for all $t\geq 1$, we have $\min_{x\in{\calD_n},y\in\calY}\pi_{\theta_t}(y|x)\geq C$.
\end{lemma}
\begin{proof}
    First, by Lemma \ref{lem:norm_of_v'}, note that, for any $t\geq 1$,
    \begin{align*}
        \norm{\alpha_{t+1}}_2 \leq \left(1-(\eta'\beta/n)\min_{x\in{\calD_n},y\in\calY}\pi_{\theta_t}(y|x)\right) \norm{\alpha_t}_2 \leq \norm{\alpha_t}_2 \leq \ldots \leq \norm{\alpha_1}_2 \leq 2\left(\beta B+1\right)\sqrt{Y}~,
    \end{align*}
    where the second inequality follows from the fact that policies are probability distributions, and the last inequality follows from Lemma \ref{lem:exponential_contraction}. Next, observe that, for any $(x,y)\in{\calD_n}\times\calY$, 
    \begin{align*}
        \Bigg| \psi(x,y)^\top\theta_t-\frac{1}{\beta}r(x,y) & -\log\mu(y|x)-\frac{\left(\Psi_n^\top\theta_t-r/\beta -\log\mu\right)^\top\mathbf{1}}{Y}\Bigg| \\ & \leq \frac{1}{\beta}\left| \beta\psi(x,y)^\top\theta_t-r(x,y)-\beta\log\mu(y|x)-\frac{\beta\left(\Psi_n^\top\theta_t-r -\beta\log\mu\right)^\top\mathbf{1}}{Y}\right|  \\
            & \leq \frac{1}{\beta}\norm{\beta\Psi_n^\top\theta_t -r-\beta\log\mu - \frac{\left( \beta\Psi_n^\top\theta_t -r-\beta\log\mu\right)^\top\mathbf{1}}{Y}\mathbf{1}}_2\\
        & \leq \frac{1}{\beta}\norm{\alpha_t}_2 \\
        & \leq 2(B+1/\beta)\sqrt{Y}~.
    \end{align*}
    Now, define $(x_1,y_1)=\arg\min_{x\in{\calD_n},y\in\calY}\psi(x,y)^\top\theta_t$ and $(x_2,y_2)=\arg\max_{x\in{\calD_n},y\in\calY}\psi(x,y)^\top\theta_t$. By the above, we have
    \begin{align*}
        \Psi_n(x_1,y_1)^\top\theta_t & \geq \frac{1}{\beta}r(x_1,y_1) + \log\mu(y_1|x_1) + \frac{\left(\Psi_n^\top\theta_t-r/\beta -\log\mu\right)^\top\mathbf{1}}{Y} - 2(B+1/\beta)\sqrt{Y}~,\\
        -\Psi_n(x_2,y_2)^\top\theta_t & \geq - \frac{1}{\beta}r(x_2,y_2) -\log\mu(y_2|x_2)-\frac{\left(\Psi_n^\top\theta_t-r/\beta -\log\mu\right)^\top\mathbf{1}}{Y} - 2(B+1/\beta)\sqrt{Y}~,
    \end{align*}
    which imply
    \begin{align*}
        \min_{x\in{\calD_n},y\in\calY}\pi_{\theta_t}(y|x) & \geq \min_{x\in{\calD_n},y\in\calY} \frac{\exp(\psi(x,y)^\top\theta_t)}{\sum_{y'\in\calY}\exp(\psi(x,y)\top\theta_t)} \geq \frac{1}{Y}\exp\left( \left( \Psi_n(x_1,y_1)-\Psi_n(x_2,y_2)\right)^\top\theta_t\right) \\   & \geq \frac{1}{Y}\exp \left( \frac{1}{\beta}\left( r(x_1,y_1)-r(x_2,y_2)\right) + \log \frac{\mu(y_1|x_1)}{\mu(y_2|x_2)} -4\left( B+1/\beta \right)\sqrt{Y}\right) \\
        & \geq \frac{1}{Y} \exp \left( -\frac{1}{\beta} - 4\left( B+1/\beta \right)\sqrt{Y}\right)=C~.
    \end{align*}
\end{proof}
Let us denote by $softmax(\Psi_n^\top v)$ the policy $\exp(\psi(x,y)^\top v)/\sum_{y'} \exp(\psi(x,y')^\top v)$, for any parameter vector $v$ and pair $(x,y)\in\calD_n\times\calY$. Now, we are ready to prove the main result of this section.
\begin{theorem}
    Let $\pi_{\theta_t}=\text{softmax}(\Psi^\top_n\theta_t)$. Using update rule \eqref{eq:regularized_gradient_update} with $\eta'\leq n/\beta$, for all $t\geq 1$, 
    \begin{align*}
         \calV^{\pi_{\theta^*}}_r(\calD_n) - \calV^{\pi_{\theta_t}}_r(\calD_n) \leq \frac{2\sqrt{Y}(B+1/\beta)}{\exp((\beta\eta' / n)\cdot C\cdot (t-1))}~,
    \end{align*}
    where 
    \begin{align*}
        C = \frac{1}{Y} \exp \left( -\frac{1}{\beta} - 4\left( B+1/\beta \right)\sqrt{Y}\right)~.
    \end{align*}
\end{theorem}
\begin{proof}
    Observe that, since $\pi_{\theta^*}\propto \mu(y|x)\exp(r(x,y)/\beta)$, we have
    \begin{align*}
        & \calV^{\pi_{\theta^*}}_r(\calD_n) - \calV^{\pi_{\theta_t}}_r(\calD_n) = \frac{1}{n}\sum_{x\in{\calD_n}} \sum_{y\in\calY}\left( \pi_{\theta^*_n}(y|x)r(x,y) - \beta D_\textnormal{KL}(\pi_{\theta^*}||\mu) - \pi_{\theta_t}(y|x)r(x,y) + \beta D_\textnormal{KL}(\pi_{\theta_t}||\mu) \right)\\
            & \leq \frac{1}{n}\sum_{x\in\calD_n} \norm{\pi_{\theta^*}(\cdot|x) -\pi_{\theta_t}(\cdot|x)}_1 \\ & \quad\quad + \frac{1}{n}\sum_{x\in\calD_n} -\beta D_\textnormal{KL}(\pi_{\theta^*}||\pi_{\theta^*}) + \pi_{\theta^*}(y|x) r(x,y) + \beta D_\textnormal{KL}(\pi_{\theta_t}||\pi_{\theta^*}) - \pi_{\theta_t}(y|x)r(x,y) \\
        & \leq \frac{2}{n}\sum_{x\in\calD_n} \norm{\pi_{\theta^*}(\cdot|x) -\pi_{\theta_t}(\cdot|x)}_1 + \beta D_\textnormal{KL}(\pi_{\theta_t}||\pi_{\theta^*})\\
        & \leq (2Y + \beta) D_\textnormal{KL}(\pi_{\theta_t}||\pi_{\theta^*})\\
        & \leq (2Y+\beta) \norm{\Psi_n^\top\theta^* -\Psi_n\theta_t + \frac{(\Psi_n^\top(\theta_t -\theta^*))^\top\mathbf{1}}{Y}\mathbf{1}}^2_\infty~,
    \end{align*}
    where the third inequality uses Pinsker's inequality and the last one follows from Lemma \ref{lem:kl_and_parameter_lemma}. 
    Now, note that the optimal softmax policy parameter $\theta^*$ satisfies, for each $(x,y)\in\calD_n$, 
    \begin{align*}
        \psi(x,y)^\top\theta^* = \frac{1}{\beta}\left( r(x,y)+\log\mu(y|x)\right)~,
    \end{align*}
    by setting the gradient at $(x,y)$ to $0$. Its existence is guaranteed by the assumption that $r^*\in\calF$ and Lemma \ref{lem:loglinear_is_optimal}. Thus, we have
    \begin{align*}
        \calV^{\pi_{\theta^*}}_r(\calD_n) - \calV^{\pi_{\theta_t}}_r(\calD_n)  & \leq  (2Y+\beta) \norm{\Psi_n^\top\theta^* -\Psi_n\theta_t + \frac{(\Psi_n^\top(\theta_t -\theta^*))^\top\mathbf{1}}{Y}\mathbf{1}}^2_\infty \\
            & = (2Y+\beta)\norm{\frac{1}{\beta}\left( r+\log\mu\right) - \Psi_n^\top\theta_t +  \frac{(\beta\Psi^\top\theta_t - r-\beta\log\mu)^\top\mathbf{1}}{\beta Y}\mathbf{1}}^2_\infty \\
            & = \frac{(2Y+\beta)}{\beta}\norm{\beta\Psi_n^\top\theta_t - r -\beta\log\mu -  \frac{(\beta\Psi^\top\theta_t - r-\beta\log\mu)^\top\mathbf{1}}{Y}\mathbf{1}}^2_\infty \\
            & \leq \frac{2}{\beta} \frac{(2Y+\beta)2\sqrt{Y}(\beta B+1)}{\exp((\eta'\beta/n) C (t-1))}~,
    \end{align*}
    where the last inequality follows from Lemma \ref{lem:exponential_contraction} and Lemma \ref{lem:lower_bound_pi}.
\end{proof}

\section{Convergence of Gradient Descent for DPO (Section \ref{sec:approximate_optimization})}\label{sec:dpo_convergence}

In this section, we will prove convergence bounds for the projected gradient descent procedure for DPO. Recall that the projected gradient descent is defined as
\begin{align*}
    \theta_{t+1} = \underset{\theta:\norm{\theta}_2\leq B}{\proj}\left(\theta_t - \eta'' \nabla_\theta\calL^\theta_\textnormal{DPO}(\calD_n)\right)~,
\end{align*}
We begin by showing that the DPO objective satisfies the PL condition \cite{karimi2016linear} stated in Definition \ref{def:pl_condition}.
We will show that the DPO objective satisfies this condition for the loglinear parametrization. First, we need to show that such an objective has Lipschitz gradients, which holds under the assumption that the parameter vectors $\theta$ have a length of no more than $B$.
\begin{lemma}\label{lem:lipschitz_loss}
    The DPO objective $\calL^\theta_\textnormal{DPO}(\calD_n)$ is Lipschitz continuous with parameter $L'_1=\beta\exp(2\beta (B+J))$ and has Lipschitz gradients with parameter $L'_2 = \beta^2\exp \left( 2\beta (B+J)\right)$, where $$J = \max_{(x,y^w,y^l)\in {\calD_n}} \beta \left|\log \frac{\mu(y^w|x)}{\mu(y^l|x)}\right|~.$$
\end{lemma}

\begin{proof}
    Note that, in order to show $L$-Lipschitzness, it suffices to prove that the Hessian of $\calL^\theta_\textnormal{DPO}(\calD_n)$ has bounded eigenvalues. Let us first compute the Hessian. Before doing that, we first simplify the gradient expression, when instantiated for the softmax parametrization. First, given parameter vector $\theta$, corresponding to $\pi_\theta$, we have
    \begin{align*}
        & \calL^\theta_\textnormal{DPO}(\calD_n) = - \expp_{(y^w,y^l,x)\sim {\calD_n}}\left[ \log \sigma \left( \beta \log \frac{\pi_\theta(y^w|x)}{\mu(y^w|x)} - \beta \log \frac{\pi_\theta(y^l|x)}{\mu(y^l|x)}\right)  \right] \\
            & = - \expp_{(y^w,y^l,x)\sim {\calD_n}}\left[ \log \sigma \left( \beta \log \frac{\exp (\theta^\top\psi(x,y^w))}{\sum_{y\in\mathcal{Y}}\exp(\theta^\top \psi(x,y))} - \beta \log \frac{\exp (\theta^\top\psi(x,y^l))}{\sum_{y\in\mathcal{Y}}\exp(\theta^\top \psi(x,y))}  - \beta \log \frac{\mu(y^w|x)}{\mu(y^l|x)}  \right) \right] \\
        & = - \expp_{(y^w,y^l,x)\sim {\calD_n}}\left[ \log \sigma \left( \beta \theta^\top (\psi(x,y^w)-\psi(x,y^l)) - \beta \log \frac{\mu(y^w|x)}{\mu(y^l|x)}  \right) \right] \\
            & = \expp_{(y^w,y^l,x)\sim {\calD_n}}\left[ \log \left( 1 + \exp \left( \beta \theta^\top \left( \psi(x,y^w)-\psi(x,y^l) \right) -  J(x,y^w,y^l) \right) \right) \right]~,
    \end{align*}
    where we let $$J(x,y^w,y^l) = \beta\log \frac{\mu(y^w|x)}{\mu(y^l|x)}~.$$
    Based on the above, we have
    \begin{align*}
        & \nabla_{\theta}\calL^\theta_\textnormal{DPO}(\calD_n) = \nabla_{\theta} \expp_{(y^w,y^l,x)\sim {\calD_n}}\left[ \log \left( 1 + \exp \left( \beta \theta^\top\left( \psi(x,y^w) - \psi(x,y^l) \right) -J(x,y^w,y^l) \right) \right) \right] \\
             & =\frac{1}{n}\sum_{(x,y^w,y^l)\in {\calD_n}} \frac{\beta \exp \left( \beta \theta^\top( \psi(x,y^w) -  \psi(x,y^l)) -J(x,y^w,y^l) \right)}{\left( 1 +\exp \left( \beta \theta^\top (\psi(x,y^w) - \psi(x,y^l)) - J(x,y^w,y^l)\right)\right) }\left( \psi(x,y^w)-\psi(x,y^l)\right)~,
    \end{align*}
    and 
    \begin{align*}
        & \nabla^2_{\theta}\calL^\theta_\textnormal{DPO}(\calD_n)  = \frac{1}{n}\sum_{(x,y^w,y^l)\in {\calD_n}} \\ & \quad\quad \nabla_{\theta} \frac{\beta \exp \left( \beta \theta^\top( \psi(x,y^w) -  \psi(x,y^l)) - J(x,y^w,y^l) \right)}{\left( 1 +\exp \left( \beta \theta^\top (\psi(x,y^w) - \psi(x,y^l)) -J(x,y^w,y^l)\right)\right) }\left( \psi(x,y^w)-\psi(x,y^l)\right) \\
            & = \frac{1}{n}\sum_{(x,y^w,y^l)\in {\calD_n}} \\ & \frac{\beta^2 \exp \left( \beta \theta^\top( \psi(x,y^w) -  \psi(x,y^l)) - J(x,y^w,y^l) \right)}{\left( 1 +\exp \left( \beta \theta^\top (\psi(x,y^w) - \psi(x,y^l)) - J(x,y^w,y^l)\right)\right)^2 }\left( \psi(x,y^w)-\psi(x,y^l)\right)\left( \psi(x,y^w)-\psi(x,y^l)\right)^\top~.
    \end{align*}
    Now, define
    \begin{align*}
        E(\theta, x,y) = \exp \left( \beta \theta^\top(\psi(x,y^w)- \psi(x,y^l)) - J(x,y^w,y^l) \right)~.
    \end{align*}
    Note that we have
    \begin{align*}
        \norm{\nabla\calL^\theta_\textnormal{DPO}(\calD_n)}_2 \leq \beta \exp(2\beta(B+J))\norm{\psi(x,y^w)-\psi(x,y^l)}_2 \leq \beta \exp(2\beta(B+J))~,
    \end{align*}
    and
    \begin{align*}
        \nabla^2_{\theta} \calL^\theta_\textnormal{DPO}(\calD_n) & = \beta^2\sum_{(x,y^w,y^l)\in{\calD_n}} \frac{E(\theta, x,y)}{n\left( 1+E(\theta, x,y)\right)^2}\psi(x)\psi(x)^\top \\ &
        \preceq \frac{\beta^2\exp \left( 2\beta(B+J)\right)}{n}  \sum_{(x,y^w,y^l)\in{\calD_n}}\psi(x)\psi(x)^\top \\ & \preceq \beta^2\exp \left( 2\beta(B+J)\right) I_d~,
    \end{align*}
    where the last inequality follows from the fact that the feature norms are bounded by $1$, and thus the maximum eigenvalue of the sample covariance matrix is no more than $1$.
\end{proof}

Next, we show that the DPO objective satisfies the PL condition under some mild assumption on the data.

\begin{lemma}\label{lem:pl_condition}
    Assume that, for each triple $(x,y^w,y^l)\in{\calD_n}$, we have that $\psi(x,y^w)\neq \psi(x,y^l)$.
    Then, if we let
    \begin{align*}
        C'_{PL} = \frac{\beta \exp (-2\beta(B+J))^3 \left( 1+ \exp (-2\beta(B+J))\right)}{n\left( 1+\exp(2\beta(B+J))\right)^2}\min_{(x,y^w,y^l)\in{\calD_n}}\norm{\psi(x,y^w)-\psi(x,y^l)}^2~,
    \end{align*}
    we have
    \begin{align*}
        \frac{1}{2}\norm{\nabla \calL^\theta_\textnormal{DPO}(\calD_n)}^2 \geq C'_{PL}\left( \calL^\theta_\textnormal{DPO}(\calD_n) - \calL^*_\textnormal{DPO}(\calD_n) \right)
    \end{align*}
    where $\calL^*_\textnormal{DPO}(\calD_n)=\min_\theta\calL^\theta_\textnormal{DPO}(\calD_n)$ denotes the optimal loss value. 
\end{lemma}

\begin{proof}
    Using the notation $E(\theta, x,y) = \exp \left( \beta \theta^\top(\psi(x,y^w) - \psi(x,y^l)) - E(x) \right)$, and noting that every quantity in the expression below is non-negative, we have
    \begin{align*}
        \frac{1}{2}\norm{\nabla_\theta \calL^\theta_\textnormal{DPO}(\calD_n)}^2 & \geq \frac{\beta^2}{2n^2}\sum_{(x,y)\in {\calD_n}} \frac{ E(\theta,x,y)^2}{\left( 1 +E(\theta,x,y)\right)^2 }\norm{ \psi(x,y^w)-\psi(x,y^l)}^2\\ &  \geq \frac{\beta^2\exp(-2\beta(B+J))^2}{n\left(1+\exp(2\beta(B+J))\right)^2}\min_{(x,y)\in{\calD_n}}\norm{\psi(x,y^w)-\psi(x,y^l)}^2
    \end{align*}
    since $-J \leq E(x) \leq J$, and thus, $\exp (-2\beta(B+J)) \leq E(\theta, x, y) \leq \exp(2\beta(B+J))$.  On the other hand, observe that
    \begin{align*}
        \calL^\theta_\textnormal{DPO}(\calD_n) - \calL^*_\textnormal{DPO}(\calD_n) & = \expp_{(x,y)\sim {\calD_n}}\Big[ \log \left( 1 + E\left(\theta,x,y\right) \right)  -  \log \left( 1 + E(\theta^*,x,y) \right) \Big] \\
            & = \frac{1}{n}\sum^n_{i=1} \log \frac{1+E(\theta,x_i,y_i)}{1+E(\theta^*,x_i,y_i)} \\
        & \leq \frac{1}{n}\sum^n_{i=1}\left( \frac{1+E(\theta,x_i,y_i)}{1+E(\theta^*,x_i,y_i)} - 1 \right) \\
            & = \frac{1}{n}\sum^n_{i=1} \frac{E(\theta,x_i,y_i)-E(\theta^*,x_i,y_i)}{1+E(\theta^*,x_i,y_i)} \\
        & \leq \frac{1}{n}\sum^n_{i=1} \frac{E(\theta,x_i,y_i)}{1+E(\theta^*,x_i,y_i)} \\
            & \leq \frac{\exp (2\beta(B+J))}{1+ \exp(-2\beta(B+J))}~.
    \end{align*}
    Now, the assumption on the data implies that, there exists $\xi$ such that
    \begin{align*}
        0 < \xi' = \min_{(x,y^w,y^l)\in{\calD_n}}\norm{\psi(x,y^w)-\psi(x,y^l)}^2~.
    \end{align*}
    Using $\xi'$ and solving the equation
    \begin{align*}
        C'_{PL} \cdot \frac{\exp (2\beta(B+J))}{1+ \exp(-2\beta(B+J))} = \frac{\beta\exp\left(-2\beta(B+J)\right)^2}{n\left( 1+\exp \left(2\beta(B+J)\right)\right)^2} \xi'
    \end{align*}
    for $C'_{PL}$, we obtain 
    \begin{align}\label{eq:pl_constant}
        C'_{PL} = \frac{\beta \exp (-2\beta(B+J))^3 \left( 1+ \exp (-2\beta(B+J))\right)}{n\left( 1+\exp(2\beta(B+J))\right)^2}\xi'~.
    \end{align}
    
\end{proof}

These two conditions are enough to obtain the following result.
\dpomleconvergence*
\begin{proof}
    A similar argument as the one in the proof of Theorem \ref{thm:rlhf_mle_convergence} implies that, for every $t\geq 1$ we have:
    \begin{align*}
        \calL^{\theta_t}_\textnormal{DPO}(\calD_n) - \calL^*_\textnormal{DPO}(\calD_n) \leq  \left( 1- \frac{C'_{PL}}{L'_2}\right)^t \left( \calL^{\theta_0}_\textnormal{DPO}(\calD_n)-\calL^*_\textnormal{DPO}(\calD_n)\right)~,
    \end{align*}
    where $L'_1=\beta \exp(2\beta(B+J))$ is the Lipschitz constant and $C'_{LPL}=BC_{PL}$.
    Using the expression for the Hessian derived in the proof of Lemma \ref{lem:lipschitz_loss} we have, for any non-zero vector $v$, that
    \begin{align*}
        v^\top\nabla^2_\theta\calL^\theta_\textnormal{DPO}(\calD_n) v \geq \beta^2 \frac{\exp\left(-\beta\left(B+J\right)\right)}{1+\exp(\beta(B+J))}\norm{v}^2_{\Sigma_{\calD_n,P}}~.
    \end{align*}
    Thus, $\calL^*_\textnormal{DPO}(\calD_n)$ is $\beta^2 \frac{\exp\left(-\beta\left(B+J\right)\right)}{1+\exp(\beta(B+J))}$-strongly convex with respect to the semi-norm $\norm{\cdot}_{\Sigma_{\calD}}$ around $\theta^*_{\calD_n}$, where $\theta^*_{\calD_n}$ is a parameter vector that achieves $\calL^*_\textnormal{DPO}(\calD_n)$. for any $\theta$, we have
    \begin{align*}
        \calL^\theta_\textnormal{DPO}(\calD_n)-\calL^*_{{\calD_n}} & \geq \left\langle \nabla_\theta \calL^*_\textnormal{DPO}(\calD_n), \theta-\theta^*_{\calD_n}\right\rangle +\beta^2 \frac{\exp\left(-\beta\left(B+J\right)\right)}{2(1+\exp(\beta(B+J)))}\norm{\theta-\theta^*}^2_{\Sigma_{\calD_n,P}}\\
        & \geq \beta^2 \frac{\exp\left(-\beta\left(B+J\right)\right)}{2(1+\exp(\beta(B+J)))}\norm{\theta-\theta^*}^2_{\Sigma_{\calD_n,P}}
    \end{align*}
    Therefore, using the upper bound on the loss, we finally obtain, for any iterate $\theta_t$ of GD,
    \begin{align*}
        \norm{\theta_t-\theta^*_{\calD_n}}^2_{\Sigma_{\calD_n,P}} & \leq O\left(\frac{\calL^{\theta_0}_\textnormal{DPO}(\calD_n)-\calL^*_{{\calD_n}}}{\beta^2}\left( 1- \frac{\beta}{n}\right)^t\right) \\
        & \leq O\left( \frac{1}{\beta}\left(1-\frac{\beta}{n}\right)^t\right)
    \end{align*}
\end{proof}

\section{Non-realizable Rewards (Section \ref{sec:nonrealizable_rewards})}\label{sec:non_realizable_proofs}

In this section, we will derive the proofs of the two results from Section \ref{sec:nonrealizable_rewards}. We restate them for convenience. 

\nonrealizablerlhf*

\begin{proof}
    From Theorem \ref{thm:main_rlhf_result}, we have
    \begin{align*}
        G(\pi_{\widehat{\theta}}) 
        & = D\left(\pi_{\widehat{\theta}}\right) + \left\langle d^*_\rho - d^{\pi_{\widehat{\theta}}}_\rho, r^*-r_{\widehat{\omega}}\right\rangle \\ 
            & = D\left(\pi_{\widehat{\theta}}\right) + \left\langle d^*_\rho - d^{\pi_{\widehat{\theta}}}_\rho, r^* -r_{\omega^*}\right\rangle + \left\langle  d^*_\rho - d^{\pi_{\widehat{\theta}}}_\rho, r_{\omega^*}-r_{\widehat{\omega}}\right\rangle \\
        & \leq D\left(\pi_{\widehat{\theta}}\right) + 2\max_{x,y}| r^*(x,y)-r_{\omega^*}(x,y)| + O\left(\Lambda_R \sqrt{\frac{d_R}{n}}\right)\\
            & \leq D\left(\pi_{\widehat{\theta}}\right) + O\left(\Lambda_R \sqrt{\frac{d_R}{n}}\right) + 2\epsilon_\textnormal{app}~,
    \end{align*}
    where for the first inequality we have used Cauchy-Schwarz, while for the last inequality we have used Theorem \ref{thm:main_rlhf_result} and  Condition \ref{asmp:non-linear_rewards}.
\end{proof}
Next, we prove the analogous result for DPO.

\nonrealizabledpo*
\begin{proof}
    Since the ground-truth reward function is not linear, we are not guaranteed that the optimal policy representable in terms of the reward is loglinear. Let $\pi^*$ denote the optimal policy for the KL-regularized problem with respect to $r^*$, and let $\pi_{\theta^*}$ be the loglinear approximation of $\pi^*$.
    \begin{align*}
        G\left(\pi_{\widetilde{\theta}}\right) & = V^\textnormal{opt}_{r^*}(\rho) - V^{\pi_{\widetilde{\theta}}}_{r^*}(\rho) \\
            & = D\left(\pi_{\widetilde{\theta}}\right) + \left( \calV^{\pi^*}_{r^*}(\rho)-\calV^{\pi_{\widetilde{\theta}}}_{r^*}(\rho)\right) \\
        & = D\left(\pi_{\widetilde{\theta}}\right) + \expp_{x\sim\rho, y\sim\pi^*(\cdot|x)}\left[r^*(x,y)-\beta\log\frac{\pi^*(y|x)}{\mu(y|x)}\right] - \expp_{x\sim\rho, y\sim\pi_{\widetilde{\theta}}(\cdot|x)}\left[r^*(x,y)-\beta\log\frac{\pi_{\widetilde{\theta}}(y|x)}{\mu(y|x)}\right]\\
            & = D\left(\pi_{\widetilde{\theta}}\right) + \expp_{x\sim\rho, y\sim\pi^*(\cdot|x)}\left[\beta\log\frac{\pi^*(y|x)}{\mu(y|x)} + \beta\log Z(x) -\beta\log\frac{\pi^*(y|x)}{\mu(y|x)}\right] \\ & \quad\quad - \expp_{x\sim\rho, y\sim\pi_{\widetilde{\theta}}(\cdot|x)}\left[\beta\log\frac{\pi^*(y|x)}{\mu(y|x)} + \beta\log Z(x) -\beta\log\frac{\pi_{\widetilde{\theta}}(y|x)}{\mu(y|x)}\right] \\
        & = D\left(\pi_{\widetilde{\theta}}\right) + \expp_{x\sim\rho, y\sim\pi_{\widetilde{\theta}}(\cdot|x)}\left[ \beta\log\frac{\pi_{\widetilde{\theta}}(y|x)}{\mu(y|x)} - \beta\log\frac{\pi^*(y|x)}{\mu(y|x)} \right] \\
            & = D\left(\pi_{\widetilde{\theta}}\right) + \expp_{x\sim\rho, y\sim\pi_{\widetilde{\theta}}(\cdot|x)}\left[ \beta\log\frac{\pi_{\widetilde{\theta}}(y|x)}{\mu(y|x)} - \beta\log\frac{\pi_{\theta^*}(y|x)}{\mu(y|x)}\right] \\ & \quad\quad\quad +\expp_{x\sim\rho, y\sim\pi_{\widetilde{\theta}}(\cdot|x)}\left[ \beta\log\frac{\pi_{\theta^*}(y|x)}{\mu(y|x)} - \beta\log\frac{\pi^*(y|x)}{\mu(y|x)} \right] \\
        & = D\left(\pi_{\widetilde{\theta}}\right) + \Theta\left( \frac{\Lambda (d_P+1)}{\beta n}\right) +  \expp_{x\sim\rho, y\sim\pi_{\widetilde{\theta}}(\cdot|x)}\left[ \beta\log \pi_{\theta^*}(y|x) - \beta\log \pi^*(y|x) \right]\\
            &  = D\left(\pi_{\widetilde{\theta}}\right) + \Theta\left( \frac{\Lambda (d_P+1)}{\beta n}\right) + \beta\sum_{x}\rho(x)\sum_y \frac{\pi_{\widetilde{\theta}}(y|x)}{\pi_{\theta^*}(y|x)}\log\frac{\pi_{\theta^*}(y|x)}{\pi^*(y|x)}\\
        & \leq  D\left(\pi_{\widetilde{\theta}}\right) + \Theta\left( \frac{\Lambda (d_P+1)}{\beta n}\right) + \beta Y\exp(2B) D_\textnormal{KL}\left(\pi_{\theta^*}||\pi^*\right)~.
    \end{align*}
    
    On the other hand, using the same idea as in the proof of Theorem \ref{thm:main_dpo_theorem}, we have
    \begin{align*}
        G\left(\pi_{\widetilde{\theta}}\right) & = V^\textnormal{opt}_{r^*}(\rho) - V^{\pi_{\widetilde{\theta}}}_{r^*}(\rho) \\
        & = D\left(\pi_{\widetilde{\theta}}\right) + \left( \calV^{\pi_{\theta^*}}_{r^*}(\rho)-\calV^{\pi_{\widetilde{\theta}}}_{r^*}(\rho)\right)\\
            & = D\left(\pi_{\widetilde{\theta}}\right) + \left( \calV^{\pi_{\theta^*}}_{r^*}(\rho) - \calV^{\pi_{\theta^*}}_{r_{\omega^*}}(\rho)\right) + \left( \calV^{\pi_{\theta^*}}_{r_{\omega^*}}(\rho) - \calV^{\pi_{\widetilde{\theta}}}_{r_{\omega^*}}(\rho)\right) + \left( \calV^{\pi_{\widetilde{\theta}}}_{r_{\omega^*}}(\rho) - \calV^{\pi_{\widetilde{\theta}}}_{r^*}(\rho)\right)\\
        & = D\left(\pi_{\widetilde{\theta}}\right) + \Theta\left( \frac{\Lambda (d_P+1)}{\beta n}\right) + \expp_{(x,y)\sim d^{\theta^*}_\rho}\left[ r^*(x,y)-r_{\omega^*}(x,y)\right] + \expp_{(x,y)\sim d^{\widetilde{\theta}}_\rho}\left[ r_{\omega^*}(x,y) - r^*(x,y)\right] \\
            & \leq D\left(\pi_{\widetilde{\theta}}\right) + \Theta\left( \frac{\Lambda (d_P+1)}{\beta n}\right) + \expp_{(x,y)\sim d^{\theta^*}_\rho}\left[ r^*(x,y)-r_{\omega^*}(x,y)\right] + 2\max_{x,y}|r^*(x,y)-r_{\omega^*}(x,y)|\\
        & \leq D\left(\pi_{\widetilde{\theta}}\right) + \Theta\left( \frac{\Lambda (d_P+1)}{\beta n}\right)  + 2\epsilon_\textnormal{app}~,
    \end{align*}
    where the fourth equality follows from Theorem \ref{thm:main_dpo_theorem} and the last inequality from Condition \ref{asmp:non-linear_rewards}.
\end{proof}

\section{The DPO Extension to MDPs (Section \ref{sec:extension_to_mdps})}\label{sec:dpo_for_mdp_formulation}

First, we start with MDP setting preliminaries. 

\subsection{Deterministic Markov Decision Processes}

An infinite-horizon discounted deterministic Markov decision process (MDP) is a mathematical object $\calM = \left( \calX, \calY, T, r^*, \gamma, \rho\right)$, where $\calX$ denotes the state space, $\calY$ denotes the action space, both of which are assumed to be finite with cardinalities $X$ and $Y$, respectively. $T:\calX\times\calY\rightarrow\calX$ denotes the deterministic transition function, where $T(x,y)$ denotes the next state after taking action $y$ in state $x$. The reward function is denoted by $r^*:\calX\times\calY\rightarrow [0,1]$. Finally, $\gamma\in [0,1)$ denotes the discount factor, while $\rho\in\Delta(\calX)$ denotes the initial state distribution. 

Policies $\pi$ are mappings from states to distributions over actions, that is, $\pi:\calX\rightarrow\Delta(\calY)$. 
Given policy $\pi$, the state occupancy measure of state $x$ with respect to initial state $x_0$ is given as 
\begin{align*}
    d^\pi_{x_0}(x)=(1-\gamma)\sum_{t\geq 0}\gamma^t\mathbb{P}\left( x_t=x|x_0,\pi\right)~,
\end{align*}
while the state-action occupancy measure is given as $d^\pi_{x_0}(x,y)=d^\pi_{x_0}(x)\pi(y|x)$. We also write $d^\pi_\rho(x,y)=\expp_{x_0\sim\rho}[d_{x_0}(x,y)]$.
Furthermore, given policy $\pi$ and an arbitrary reward function $r$, the value function of policy $\pi$ with respect to reward $r$ is defined as
\begin{align*}
    V^\pi_r(x) = \expp\left[ \sum^\infty_{t=0}\gamma^t r(x_t,y_t)\Big| x_0=x,\pi \right]~,
\end{align*}
and the action-value function is defined as
\begin{align*}
    Q^\pi_r(x,y)= \expp\left[ \sum^\infty_{t=0}\gamma^t r(x_t,y_t)\Big| x_0=x,y_0=y,\pi \right]~,
\end{align*}
for every state-action pair $(x,y)$. We denote by $V^\pi_r(\rho)=\expp_{x\sim\rho}[V^\pi_r(x)]$ the expected value function over the initial distribution.

\subsection{DPO for MDPs}\label{appendix:dpo_for_mdps}

A direct extension of DPO to the MDP setting is not straightforward. To understand this, it is enough to see that the optimal policy-to-reward mapping in this case is not linear. Fix a reward function $r$. The gradient of the KL-regularized objective with respect to $r$ is given as 
\begin{align*}
\nabla_\theta \calV^\theta_r (\rho) = \frac{1}{1-\gamma}\sum_x d^{\pi_\theta}_\rho(x)\sum_y \pi_\theta(y|x) \left(  r(x,y) + \gamma\calV^{\pi_\theta}(T(x,y)) - \beta \log\frac{\pi_\theta(y|x)}{\mu(y|x)} \right) \overline{\psi}_\theta(x,y)~,
\end{align*}
where $\overline{\psi}_\theta(x,y)=\psi(x,y)-\expp_{y'\sim\pi_\theta(\cdot|x)}[\psi(x,y')]$. See Appendix \ref{sec:mdp_gradient} for derivations. What complicates things is the occupancy measure $d^\pi_\rho$, which is non-linearly dependent on policy $\pi$, and the gradient of the occupancy measure. To allow for the change of variables to carry through in this case, we utilize the dual formulation of Problem \eqref{op_kl_reg_for_mdp}:
\begin{align}\label{op_dpo_for_mdp}
    \max_{d_{\rho}}\;\; & \sum_{x,y}d_{\rho}(x,y)r(x,y) \nonumber  - \beta \sum_{x,y}d_{\rho}(x,y)\log \frac{d_{\rho}(x,y)}{d^\mu_{\rho}(x,y)} \tag{P3.2'}\\
    \text{s.t.}\;\; & \sum_{y} d_{\rho}(x,y) = (1-\gamma)\rho(x) \nonumber  + \gamma \sum_{x',y'}\mathds{1}\left(x=T\left(x',y'\right)\right) d_{\rho}(x',y'), \forall x\in\calX\nonumber ~,
\end{align}
where we have used that 
\begin{align*}
    V^\pi_r(\rho)=\sum_{x,y}d^\pi_{\rho}(x,y)r(x,y)~,
\end{align*}
and also taken the KL-divergence of the occupancy measures, instead of the actual policies. 
This is a convex program and thus any stationary points are optimal. The Lagrangian of the above problem can be written as
\begin{align*}
    L(d_\rho,\alpha) & = \sum_{x,y}d_\rho(x,y)\left( r^*(x,y) -\beta\log\frac{d_\rho(x,y)}{d^\mu_\rho(x,y)}\right) \\ & \quad + \sum_{x}\alpha(x)\left( \sum_{y}d_\rho(x,y) - (1-\gamma)\rho(x) - \gamma\sum_{x',y'}\mathds{1}(x=T(x',y'))d_\rho(x',y')\right) \\
        & = -\beta \sum_{x,y}d_\rho(x,y)\log\frac{d_\rho(x,y)}{d^\mu_\rho(x,y)} - (1-\gamma)\sum_x \rho(x)\alpha(x) \\ & \quad + \sum_{x,y}d_\rho(x,y)\left(\underbrace{r^*(x,y) - \gamma\sum_{x'}\mathds{1}(x=T(x',y'))\alpha(x') + \alpha(x)}_{e_\alpha(x,y)}\right)~,
\end{align*}
Then, given $(x,y)$, the gradient of the Lagrangian with respect to $d_\rho(x,y)$ is
\begin{align*}
    \nabla_{d_\rho(x,y)}L(d_\rho,\alpha) = -\beta \left(\log\frac{d_\rho(x,y)}{d^\mu_\rho(x,y)} - \mathbf{1}\right) + e_\alpha(x,y)~,
\end{align*}
which, when set to zero, yields
\begin{align*}
    d_\rho(x,y) = d^\mu_\rho(x,y)\exp\left(\frac{1}{\beta}e_\alpha(x,y)\right)\exp(-1)~.
\end{align*}
Primal feasibility implies that $\sum_{x,y}d_\rho(x,y)=1$, thus, our choice of $\alpha$ should satisfy such condition. Letting $Z=\exp(1)$, and $\alpha^*$ be the optimal Lagrange multiplier, we have
\begin{align}\label{eq:optimal_occupancy_measure_via_reward}
    d^*_\rho(x,y) = \frac{1}{Z} d^\mu_\rho(x,y)\exp\left(\frac{1}{\beta}e_{\alpha^*}(x,y)\right)~.
\end{align}
Writing the expression for the reward function, we get
\begin{align*}
    r^*(x,y) = \beta\log\frac{d^*_\rho(x,y)}{d^\mu_\rho(x,y)} + \beta + \gamma\sum_{x'}\mathds{1}(x=T(x',y'))\alpha^*(x') - \alpha^*(x)~.
\end{align*}
Now, observe that, given a trajectory $\tau=(x_0,y_0,x_1,\ldots)$, we can write the discounted return using the above expression and obtain
\begin{align}
    \sum^\infty_{t=0}\gamma^t r^*(x_t,y_t) & = \sum^\infty_{t=0}\gamma^t \left( \beta\log\frac{d^*_\rho(x_t,y_t)}{d^\mu_\rho(x_t,y_t)} + \beta + \gamma\sum_{x}\mathds{1}(x=T(x',y'))\alpha^*(x) - \alpha^*(x_t)\right)\nonumber\\
        & = \sum^\infty_{t=0}\gamma^t \left( \beta\log\frac{d^*_\rho(x_t,y_t)}{d^\mu_\rho(x_t,y_t)} + \beta + \gamma\alpha^*(x_{t+1}) - \alpha^*(x_t)\right)\label{eq:dpo-for-mdp-eq0001}\\
        & = \sum^\infty_{t=0}\gamma^t \left( \beta\log\frac{d^*_\rho(x_t,y_t)}{d^\mu_\rho(x_t,y_t)} +  \beta - \alpha^*(x_0)\right)~,\label{eq:reward_via_occupancy_measures}
\end{align}
where for Equation \eqref{eq:dpo-for-mdp-eq0001} we have used the fact that the transitions are deterministic, and for Equation \eqref{eq:reward_via_occupancy_measures} note that the terms $\alpha^*(x)$ cancel each other out. 

Now, let us get back to the BT preference model for MDPs. Given a dataset $\calD_n$ of pairs of trajectories, each pair of which starts from the same initial state, we can express the MLE loss directly in terms of the occupancy measures using the above derivation as follows:
\begin{align*}
    \calL_\textnormal{DPO}(d_\rho) = -\expp_{(\tau^w,\tau^l)\sim\calD_n}\left[ \log\sigma \left( \sum^\infty_{t=0}\gamma^t\beta\log\frac{d_\rho(x^w_t,y^w_t)}{d^\mu_\rho(x^w_t,y^w_t)} - \sum^\infty_{t=0}\gamma^t\beta\log\frac{d_\rho(x^l_t,y^l_t)}{d^\mu_\rho(x^l_t,y^l_t)} \right) \right]~,
\end{align*}
where we have used the fact that the terms $\beta$ and $\alpha^*(x_0)$ cancel out.

Now, note that the minimizer to the above loss may not satisfy the Bellman flow constraints of Problem \eqref{op_dpo_for_mdp}. Thus, we need to restrict the domain of the problem to the following set
\begin{align*}
    & \mathcal{B} = \bigg\{ d\in\Delta(\calX\times\calY): \sum_{y} d(x,y) = (1-\gamma)\rho(x) \nonumber  + \gamma \sum_{x',y'}\mathds{1}\left(x=T(x',y')\right) d(x',y'), \forall x\in\calX \bigg\}
\end{align*}
\subsection{DPO for MDPs with Loglinear Occupancy Measures}
Similar to the contextual bandit setting, we want to write the DPO loss such that it resembles logistic regression. For loglinear occupancy measures, as defined in Definition \ref{def:loglinear_occupancy_measures}, with parameter set restricted to
\begin{align*}
    \Theta' := \{ \theta\in\mathbb{R}^{d_M}: d^{\pi_\theta}_{\rho}\in\calB\}~,
\end{align*}
we can write the loss so that it resembles logistic regression. 
Note that the domain of $\theta$ is restricted only to those parameters which imply that $d^{\pi_\theta}_{\rho}$ is an occupancy measure with respect to the underlying MDP. 
For this case, the DPO loss becomes
\begin{align*}
    & \calL_{\calD_n}(\theta) = -\expp_{(\tau^w,\tau^l)\sim\calD_n}\Bigg[  \log \sigma \left( \beta\theta^\top\left( \sum^\infty_{t=0}\gamma^t\left( \psi(x^w_t,y^w_t)-\psi(x^l_t,y^l_t)\right)\right)  + K(\tau^w,\tau^l) \right) \Bigg]
\end{align*}
where
\begin{align*}
    K(\tau^w,\tau^l) = \sum^\infty_{t=0}\gamma^t \log \frac{d^\mu_{\rho}(x^l_t,y^l_t)}{d^\mu_{\rho}(x^w_t,y^w_t)}~.
\end{align*}
Given the learned occupancy measure $d^{\pi_\theta}_\rho$, one can finally compute an optimal policy, for each state-action pair, as
\begin{align*}
    \pi_\theta(y|x) = \frac{d^{\pi_\theta}_\rho (x,y)}{\sum_yd^{\pi_\theta}_\rho(x,y)}~.
\end{align*}
Note that, in general, the quantity $K(\tau^w,\tau^l)$ is not easy to compute as it requires access to the occupancy measure with respect to $\mu$. However, in practice, $K(\tau^w,\tau^l)$ can be treated as a hyperparameter of the problem and tuned accordingly.

\section{Gradient Expression for KL-regularized Objective in MDPs}\label{sec:mdp_gradient}

In this section, we derive the gradient for the loglinear policy class. We rewrite the problem below for convenience.
\begin{align*}
    \max_\theta \; & \expp_{x\sim\rho}\left[ \sum^\infty_{t=0}\gamma^t \left(r(x_t,y_t)-\beta D_\textnormal{KL}\left(\pi_\theta(\cdot|x_t)||\mu(\cdot|x_t)\right)\right) \Big| y_t\sim \pi_\theta(\cdot|x_t) \right]
\end{align*}

\begin{lemma}\label{lem:gradient_rlhf_mdp}
    Let 
    \begin{align*}
        \calV^{\pi_\theta}_r(x) = \expp_{x\sim\rho, y_t\sim\pi_\theta(\cdot|x_t)}\left[ \sum_{t\geq 0}\gamma^t \left(r(x_t,y_t) -\beta\log\frac{\pi_\theta(y_t|x_t)}{\mu(y_t|x_t)}\right)  \right]
    \end{align*}
    and
    \begin{align*}
        \calQ^{\pi_\theta}_r(x,y) = r(x,y) + \gamma \calV^{\pi_\theta}_r(T(x,y))~.
    \end{align*}
    The gradient expression for $\calV^{\pi_\theta}(\rho)$ is given by 
    \begin{align*}
        \nabla_\theta \calV^\theta_r (\rho) = \frac{1}{1-\gamma}\sum_x d^{\pi_\theta}_\rho(x)\sum_y \pi_\theta(y|x) \left(  \calQ^{\pi_\theta}_r(x,y) - \beta \log\frac{\pi_\theta(y|x)}{\mu(y|x)} \right) \overline{\psi}_\theta(x,y)~.
    \end{align*}
\end{lemma}
\begin{proof}
    Note that we have
    \begin{align*}
        \calV^{\pi_\theta}_r(\rho) = \expp_{x\sim\rho}\left[ \sum_y \pi_\theta(y|x) \left( \calQ^{\pi_\theta}_r(x,y) - \beta \log\frac{\pi_\theta(y|x)}{\mu(y|x)} \right)  \right]~.
    \end{align*}
    Thus, we can write
    \begin{align*}
        \nabla_\theta \calV^{\pi_\theta}_r(\rho) & = \sum_{x,y}\rho(x)\Bigg( \nabla_\theta\pi_\theta(y|x) \left( \calQ^{\pi_\theta}_r(x,y) - \beta \log\frac{\pi_\theta(y|x)}{\mu(y|x)} \right) \\ & \quad\quad  + \pi_\theta(y|x)\left( \nabla_\theta \calQ^{\pi_\theta}_r(x,y) - \frac{\mu(y|x)}{\pi_\theta(y|x)}\nabla_\theta \pi_\theta(y|x)\right)\Bigg)\\
            & = \sum_{x,y}\rho(x) \left( \pi_\theta(y|x) \left(  \calQ^{\pi_\theta}_r(x,y) - \beta \log\frac{\pi_\theta(y|x)}{\mu(y|x)} - 1\right) \overline{\psi}_\theta(x,y) + \pi_\theta(y|x) \nabla_\theta \calQ^{\pi_\theta}_r(x,y) \right)\\
        & = \sum_{x,y}\rho(x) \left( \pi_\theta(y|x) \left(  \calQ^{\pi_\theta}_r(x,y) - \beta \log\frac{\pi_\theta(y|x)}{\mu(y|x)} \right) \overline{\psi}_\theta(x,y)\right)  \\ & \quad\quad + \gamma \sum_{x,y}\rho(x)\pi_\theta(y|x) \nabla_\theta \calV^{\pi_\theta}_r(T(x,y)) \\
            & = \frac{1}{1-\gamma}\sum_x d^{\pi_\theta}_\rho(x)\sum_y \pi_\theta(y|x) \left(  \calQ^{\pi_\theta}_r(x,y) - \beta \log\frac{\pi_\theta(y|x)}{\mu(y|x)} \right) \overline{\psi}_\theta(x,y)~,
    \end{align*}
    where the second equality follows from the derivation of the gradient of loglinear policies (see the proof of Lemma \ref{lem:gradient_of_generative}, while the third equality follows from the fact that $\expp_{y\sim\pi_\theta(\cdot|x)}[\overline{\psi}_\theta(x,y)] = 0$, for each $x\in\calX$. 
\end{proof}

\section{Technical Lemmas}\label{sec:technical_lemmas}

The purpose of this section is to present various technical results that are useful for our paper. Let us denote by $\Phi\in\mathbb{R}^{d_R\times XY}$ and $\Psi\in\mathbb{R}^{d_P\times XY}$ the reward and policy feature matrices with columns $\phi(x,y)$ and $\psi(x,y)$, respectively. 

\begin{lemma}\label{lem:loglinear_is_optimal}
    Assume that $r^*\in\calF$, $\pi^*\in\Pi$ and $\mu\in\Pi$, for some $\pi^*_{r^*}\in\arg\max_\pi\calV^\pi(\rho)$. Furthermore, assume that the columns space of $\Phi$ is a subspace of the column space of $\Psi$. Then, there exists $\theta^*\in\Theta$, for which $\pi_{\theta^*}$ maximizes the objective of \eqref{op_kl_regularized} and that can be represented in terms of the ground-truth reward function, i.e. $\pi^*_{r^*}(y|x) = \pi_{\theta^*}(y|x) \propto \mu(y|x) \exp(r^*(x,y)/\beta)$, for all $(x,y)$.
\end{lemma}
\begin{proof}
    Let $x\in\calX$. From Equation \eqref{eq:reward_to_policy_mapping} we have that
    \begin{align*}
        \pi^*_{r^*}(x,y) & \propto \mu(y|x)\exp\left(\frac{1}{\beta}r^*(x,y)\right)\\
        & \propto \exp\left(\theta_\mu^\top\psi(x,y) + \phi(x,y)^\top\omega^*\right)~,
    \end{align*}
    for some $\pi^*_{r^*}\in\arg\max_\pi\calV^\pi_{r^*}(\rho)$, where the second relation holds due to the assumptions on the policy class and reward class. Thus, if we can find a $\theta^*\in\mathbb{R}^{d_P}$ such that, for all $(x,y)$, we have
    \begin{align*}
        \exp\left(\psi(x,y)^\top\theta^*\right)=\exp\left(\theta_\mu^\top\psi(x,y) + \phi(x,y)^\top\omega^*\right)~,
    \end{align*}
    then we have shown that the optimal regularized policy belongs to the loglinear class. For the above to hold, we equivalently need
    \begin{align*}
        \Psi^\top\left(\theta^*-\theta_\mu\right) - \Phi^\top\omega^* = \mathbf{0}~.
    \end{align*}
    The above equation has a solution for $\theta^*$ whenever the column space of $\Phi$ is contained in the column space of $\Psi$. 
\end{proof}
\begin{lemma}\label{lem:loglinear_is_optimal_v2}
    Assume that $r^*\in\calF$, $d^\mu_\rho\in\Pi'$ and $d^{\pi^*_{r^*}}_\rho\in\Pi'$, for some optimal $d^{\pi^*_{r^*}}_\rho$. Furthermore, assume that the column space of $\Phi+\Phi_{\pi^*_{r^*}}$ is contained in the column space of $\Psi$, where $\Phi_{\pi^*_{r^*}}\in\mathbb{R}^{d_R\times XY}$ has columns
    \begin{align*}
        \gamma\expp\left[\sum_{t\geq 0}\gamma^t\phi(x_t,y_t)\Big|x_0=x,\pi^*_{r^*}\right] -  \expp\left[\sum_{t\geq 0}\gamma^t\phi(x_t,y_t)\Big|x_0=x,y_0=T(x,y),\pi^*_{r^*}\right]~.
    \end{align*}
    Then, for finite MDPs with deterministic transitions, there exists $\theta^*$ such that $d^{\pi^*_{r^*}}_\rho(x,y)$.
\end{lemma}
\begin{proof}
    From Equation \eqref{eq:optimal_occupancy_measure_via_reward} we have
    \begin{align*}
        d^*_\rho(x,y) = \frac{1}{Z} d^\mu_\rho(x,y)\exp\left(\frac{1}{\beta}e_{\alpha^*}(x,y)\right)~,
    \end{align*}
    where
    \begin{align*}
        e_\alpha(x,y) = r^*(x,y) +\gamma \alpha^*(T(x,y)) - \alpha^*(x)~,
    \end{align*}
    is the advantage function when using $\alpha^*$ and $\alpha^*$ denote the optimal dual variables for Problem \eqref{op_dpo_for_mdp}. As shown in \citep{lee2021optidice}, these variables correspond to the optimal value function with respect to $r^*$. Thus, we have
    \begin{align*}
        e_{\alpha^*}(x,y) & = r^*(x,y) +\gamma\alpha^*(T(x,y))-\alpha^*(x) \\
            & = \phi(x,y)^\top\omega^* +\gamma \expp\left[\sum_{t\geq0}\gamma^t r^*(x,y)\Big| x_0=x,\pi^*_{r^*}\right] -   \expp\left[\sum_{t\geq0}\gamma^t r^*(x,y)\Big| x_0=x,y_0=T(x,y),\pi^*_{r^*}\right]\\
        & = \phi(x,y)^\top\omega^* +\gamma \expp\left[\sum_{t\geq0}\gamma^t \phi(x,y)^\top\omega^*\Big| x_0=x,\pi^*_{r^*}\right] -   \expp\left[\sum_{t\geq0}\gamma^t \phi(x,y)^\top\omega^*\Big| x_0=x,y_0=T(x,y),\pi^*_{r^*}\right]\\
            & = \left(\phi(x,y) - \gamma \expp\left[\sum_{t\geq0}\gamma^t \phi(x,y)\Big| x_0=x,\pi^*_{r^*}\right] -   \expp\left[\sum_{t\geq0}\gamma^t \phi(x,y)\Big| x_0=x,y_0=T(x,y),\pi^*_{r^*}\right]\right)^\top\omega^*~.
    \end{align*}
    Thus, we have
    \begin{align*}
        & d^*_\rho(x,y) = \frac{1}{Z} d^\mu_\rho(x,y)\exp\left(\frac{1}{\beta}e_{\alpha^*}(x,y)\right) \propto \exp\Bigg(\theta^\top_\mu\psi(x,y) \\ & \quad +\left. (\omega^*)^\top\left(\phi(x,y) + \gamma \expp\left[\sum_{t\geq0}\gamma^t \phi(x,y)\Big| x_0=x,\pi^*_{r^*}\right] -   \expp\left[\sum_{t\geq0}\gamma^t \phi(x,y)\Big| x_0=x,y_0=T(x,y),\pi^*_{r^*}\right]\right)\right)~.
    \end{align*}
    For the above to hold, we equivalently need
    \begin{align*}
         \Psi^\top\left(\theta^*-\theta_\mu\right) - \left(\Phi+\Phi_{\pi^*_{r^*}}\right)^\top\omega^* = \mathbf{0}~.
    \end{align*}
    The above has a solution whenever the column space of $\Phi+\Phi_{\pi^*_{r^*}}$ is contained in the column space of $\Psi$~.
\end{proof}

Next, we will prove a result that connects the suboptimality gap with the gap in terms of the KL-regularized objectives.
\begin{lemma}\label{lem:kl_difference_existence}
    For any $\theta$, we have 
    \begin{align*}
        \beta D_\textnormal{KL}\left(\pi_{\theta^*}||\mu\right)-\beta D_{KL}\left(\pi_{\theta}||\mu\right) \leq D(\pi_\theta) \leq \beta D_\textnormal{KL}\left(\pi^\textnormal{opt}_{r^*}||\mu\right)- \beta D_{KL}\left(\pi_\theta||\mu\right)~,
    \end{align*}
\end{lemma}
\begin{proof}
    Note that, by definition,
    \begin{align*}
        G\left(\pi_\theta\right) & = V^\textnormal{opt}_{r^*}(\rho) - V^{{\pi_{\theta}}}_{r^*}(\rho) \\
        & = \left(  V^\textnormal{opt}_{r^*}(\rho) -  \calV^{\pi_{\theta^*}}_{r^*}(\rho)\right) + \left( \calV^{\pi_{\theta^*}}_{r^*}(\rho) - \calV^{{{\pi_\theta}}}_{r^*}(\rho)\right)  + \left(  \calV^{{{\pi_\theta}}}_{r^*}(\rho) - V^{{{\pi_\theta}}}_{r^*}(\rho)\right)\\
        & \leq  \left(  V^\textnormal{opt}_{r^*}(\rho) -  \calV^{\pi^\textnormal{opt}_{r^*}}_{r^*}(\rho)\right) + \left( \calV^{\pi_{\theta^*}}_{r^*}(\rho) - \calV^{{{\pi_\theta}}}_{r^*}(\rho)\right)  + \left(  \calV^{{{\pi_\theta}}}_{r^*}(\rho) - V^{{{\pi_\theta}}}_{r^*}(\rho)\right)\\
        &  \leq \beta D_\textnormal{KL}\left( \pi^\textnormal{opt}_{r^*}||\mu\right) - \beta D_\textnormal{KL}\left( \pi_\theta||\mu\right) + \left( \calV^{\pi_{\theta^*}}_{r^*}(\rho) - \calV^{{\pi_{\theta}}}_{r^*}(\rho)\right)~.
    \end{align*}
    Similarly, 
    \begin{align*}
        G\left(\pi_\theta\right) & = V^\textnormal{opt}_{r^*}(\rho) - V^{{\pi_{\theta}}}_{r^*}(\rho) \\
        & = \left(  V^\textnormal{opt}_{r^*}(\rho) -  \calV^{\pi_{\theta^*}}_{r^*}(\rho)\right) + \left( \calV^{\pi_{\theta^*}}_{r^*}(\rho) - \calV^{{\pi_{\theta}}}_{r^*}(\rho)\right)  + \left(  \calV^{{\pi_{\theta}}}_{r^*}(\rho) - V^{{\pi_{\theta}}}_{r^*}(\rho)\right)\\
        & \geq \left(  V^{\pi_{\theta^*}}_{r^*}(\rho) -  \calV^{\pi_{\theta^*}}_{r^*}(\rho)\right) + \left( \calV^{\pi_{\theta^*}}_{r^*}(\rho) - \calV^{{\pi_{\theta}}}_{r^*}(\rho)\right)  + \left(  \calV^{{\pi_{\theta}}}_{r^*}(\rho) - V^{{\pi_{\theta}}}_{r^*}(\rho)\right)\\
        & \geq \beta D_\textnormal{KL}\left( \pi_{\theta^*}||\mu\right) - \beta D_\textnormal{KL}\left( \pi_\theta||\mu\right) + \left( \calV^{\pi_{\theta^*}}_{r^*}(\rho) - \calV^{{\pi_{\theta}}}_{r^*}(\rho)\right)~.
    \end{align*}
    The result follows.
\end{proof}

Next, we will control the quantity $D(\pi_\theta)$ for the DPO setting.
\begin{lemma}
    With probability at least $1-\delta$, we have 
    \begin{align*}
        D(\pi_{\widetilde{\theta}}) & \leq \beta D_\textnormal{KL}(\pi^\textnormal{opt}_{r^*},\pi_{\theta^*}) +  \frac{\Lambda_P U d_P}{S_P n}\sqrt{\frac{\log(4/\delta)}{2n}}\\
        & = \beta \left( D_\textnormal{KL}\left(\pi^\textnormal{opt}_{r^*}||\mu\right) - D_\textnormal{KL}\left(\pi_{\theta^*}||\mu\right) \right)  + O\left( \frac{d_P}{ n^{3/2}}\right)~.
    \end{align*}
\end{lemma}
\begin{proof}
    Recall that, for any $\theta$, we have defined
    \begin{align*}
        D_\textnormal{KL}\left(\pi_\theta||\mu\right) =\sum_{x}\rho(x) \sum_{x,y}\pi_\theta(y|x)\log\frac{\pi_\theta(y|x)}{\mu(y|x)}~.
    \end{align*}
    First, note that
    \begin{align*}
        D_\textnormal{KL}\left( \pi_{\theta^*}||\mu\right) & - D_\textnormal{KL}\left( \pi_{\theta}||\mu\right) =\Bigg( D_\textnormal{KL}\left( \pi_{\theta^*}||\mu\right) - \frac{1}{n}\sum_{x\in\calD}D_\textnormal{KL}\left(\pi_{\theta^*_{\calD_n}}||\mu\right)\Bigg) + \Bigg( \frac{1}{n}\sum_{x\in\calD}D_\textnormal{KL}\left(\pi_{\theta^*_{\calD_n}}||\mu\right) \\ & - \frac{1}{n}\sum_{x\in\calD}D_\textnormal{KL}\left(\pi_{\widetilde{\theta}}||\mu\right)\Bigg) + \Bigg( \frac{1}{n}\sum_{x\in\calD}D_\textnormal{KL}\left(\pi_{\widetilde{\theta}}||\mu\right) - D_\textnormal{KL}\left( \pi_{\theta}||\mu\right) \Bigg) \\
            & = \Bigg( D_\textnormal{KL}\left( \pi_{\theta^*}||\mu\right) - \frac{1}{n}\sum_{x\in\calD}D_\textnormal{KL}\left(\pi_{\theta^*_{\calD_n}}||\mu\right)\Bigg) +  \Bigg( \frac{1}{n}\sum_{x\in\calD}D_\textnormal{KL}\left(\pi_{\widetilde{\theta}}||\mu\right) - D_\textnormal{KL}\left( \pi_{\theta}||\mu\right) \Bigg)~, 
    \end{align*}
    where the second equality follows from the exact optimization assumption. Note that the two summands above are deviations from means. If we can show that each individual quantity is bounded, then we can apply Hoeffding bounds.
    To that end, first, note that
    \begin{align*}
         D_\textnormal{KL}\left( \pi_{\theta^*}||\mu\right) & = \sum_x \rho(x) \sum_{y}\pi_{\theta^*}(y|x)\log\frac{\pi_{\theta^*}(y|x)}{\mu(y|x)} = \sum_{x}\rho(x) \sum_{y}\pi_{\theta^*}(y|x)\log\frac{\mu(y|x)\exp\left(\frac{1}{\beta}r^*(x,y)\right)}{\mu(y|x)} \\
            & = \frac{1}{\beta}\sum_{x}\rho(x) \sum_{y}\pi_{\theta^*}(y|x) r^*(x,y) \leq \frac{1}{\beta}~,
    \end{align*}
    since the reward cannot be more than $1$. Similarly, for every $x\in\calD$, we have
    \begin{align*}
        D_\textnormal{KL}\left(\pi_{\widetilde{\theta}}||\mu\right) & = D_\textnormal{KL}\left(\pi_{\theta^*_{\calD_n}}||\mu\right) = \sum_{x}\rho(x) \sum_{y}\pi_{\theta^*_{\calD_n}}(y|x)\log\frac{\mu(y|x)\exp\left(\frac{1}{\beta}r^*(x,y)\right)}{\mu(y|x)} \\
            & = \frac{1}{\beta}\sum_{x}\rho(x) \sum_{y}\pi_{\theta^*}(y|x) r^*(x,y) \leq \frac{1}{\beta}~.
    \end{align*}
    For other contexts $x\not\in \calD$, such a relation does not hold. Thus, we take another approach. We show that, for such points, the KL divergence between the learned policy and the sampling policy cannot be too far away from that between the optimal policy and the sampling policy. 
    \begin{align*}
        \left|D_\textnormal{KL}\left( \pi_{\theta^*}||\mu\right) - D_\textnormal{KL}\left(\pi_{\theta^*_{\calD_n}}||\mu\right)\right| & =  \left|\sum_{x}\rho(x) \sum_{y}\pi_{\theta^*}(y|x)\log\frac{\pi_{\theta^*}(y|x)}{\mu(y|x)} -  \sum_{x,y}\pi_{\theta^*_{\calD_n}}(y|x)\log\frac{\pi_{\theta^*_{\calD_n}}(y|x)}{\mu(y|x)}\right| \\
        & = \Bigg|\sum_{x}\rho(x) \sum_{y}\pi_{\theta^*}(y|x)\left( \log\frac{\pi_{\theta^*}(y|x)}{\pi_{\theta^*}(y|x)} + \log\exp\left(\frac{1}{\beta}r^*(x,y)\right)\right) \\ & \quad - \sum_{x,y}\pi_{\theta^*_{\calD_n}}(y|x)\left( \log\frac{\pi_{\theta^*_{\calD_n}}(y|x)}{\pi_{\theta^*}(y|x)} + \log\exp\left(\frac{1}{\beta}r^*(x,y)\right)\right) \Bigg|\\
        & \leq\frac{1}{\beta}\left|\left(V^{\pi_{\theta^*}}_{r^*}(\rho) - V^{\pi_{\theta^*_{\calD_n}}}(\rho)\right) - D_\textnormal{KL}\left( \pi_{\theta^*}||\pi_{\theta^*_{\calD_n}}\right)\right|\\
        & \leq \frac{1}{\beta} + D_\textnormal{KL}\left( \pi_{\theta^*}||\pi_{\theta^*_{\calD_n}}\right)~.
    \end{align*}
    Now, for the last term of the right-hand side, we have
    \begin{align*}
        D_\textnormal{KL}\left( \pi_{\theta^*}||\pi_{\theta^*_{\calD_n}}\right) & =\sum_{x}\rho(x) \sum_{y}\pi_{\theta^*}(y|x)\left( \log \pi_{\theta^*}(y|x) -\log \pi_{\theta^*_{\calD_n}}(y|x)\right) \\
        & = \sum_{x}\rho(x) \sum_{y}\pi_{\theta^*}(y|x)\left( \left\langle \psi(x,y),\theta^*-\theta^*_{\calD_n}\right\rangle + \log\frac{\sum_{x',y'}\exp\left(\psi(x',y')^\top\theta^*_{\calD_n}\right)}{\sum_{x',y'}\exp\left(\psi(x',y')^\top\theta^*\right)}\right) \\
        & \leq \frac{\Lambda_P U d_P}{S_P\beta n}~,
    \end{align*}
    where the last inequality follows from the same arguments as in the proof of Theorem \ref{thm:main_dpo_theorem}. 
    Going back to the original expression, note that, for any given $x\in\calX$, we have 
    \begin{align*}
        0 \leq D_\textnormal{KL}\left( \pi_{\theta^*}||\mu\right) \leq \frac{1}{\beta}~,\;\;\; \text{and}\;\;\; 0 \leq D_\textnormal{KL}\left( \pi_{\theta^*_{\calD_n}}||\mu\right) \leq \frac{1}{\beta} + \frac{\Lambda_P U d_P}{S_P\beta n}~.
    \end{align*}
    Thus, by Hoeffding's inequality, for any $\delta\geq 0$, with probability at least $1-\delta$, we have
    \begin{align*}
        \left| D_\textnormal{KL}\left( \pi_{\theta^*}||\mu\right) - \frac{1}{n}\sum_{x\in\calD}D_\textnormal{KL}\left(\pi_{\theta^*_{\calD_n}}||\mu\right)\right| \leq\frac{1}{\beta} \sqrt{\frac{\log(4/\delta)}{2n}}~,
    \end{align*}
    and
    \begin{align*}
        \left| \frac{1}{n}\sum_{x\in\calD}D_\textnormal{KL}\left(\pi_{\widetilde{\theta}}||\mu\right) - D_\textnormal{KL}\left( \pi_{\theta}||\mu\right)\right| \leq \left(\frac{1}{\beta} + \frac{\Lambda_P U d_P}{S_P\beta n}\right)\sqrt{\frac{\log(4/\delta)}{2n}}~,
    \end{align*}
    which implies that 
    \begin{align*}
        - \left(\frac{2}{\beta} + \frac{\Lambda_P U d_P}{S_P\beta n}\right)\sqrt{\frac{\log(4/\delta)}{2n}}\leq D_\textnormal{KL}\left( \pi_{\theta^*}||\mu\right) & - D_\textnormal{KL}\left( \pi_{\theta}||\mu\right) \leq \left(\frac{2}{\beta} + \frac{\Lambda_P U d_P}{S_P\beta n}\right)\sqrt{\frac{\log(4/\delta)}{2n}}~.
    \end{align*}
    On the other hand, note that 
    \begin{align*}
        D_\textnormal{KL}\left(\pi^\textnormal{opt}_{r^*}||\mu\right) - D_\textnormal{KL}\left(\pi_{\widetilde{\theta}}||\mu\right) & =\left( D_\textnormal{KL}\left(\pi^\textnormal{opt}_{r^*}||\mu\right) - D_\textnormal{KL}\left(\pi_{\theta^*}||\mu\right) \right)+ \left( D_\textnormal{KL}\left(\pi_{\theta^*}||\mu\right) - D_\textnormal{KL}\left(\pi_{\widetilde{\theta}}||\mu\right)\right) \\
        & \leq \left( D_\textnormal{KL}\left(\pi^\textnormal{opt}_{r^*}||\mu\right) - D_\textnormal{KL}\left(\pi_{\theta^*}||\mu\right) \right) + \left(\frac{2}{\beta} + \frac{\Lambda_P U d_P}{S_P\beta n}\right)\sqrt{\frac{\log(4/\delta)}{2n}}~.
    \end{align*}
\end{proof}
Next, we prove some useful properties of the log-exp-sum function. 
\begin{lemma}
    The function defined as
    \begin{align*}
        A(\theta) = \sum_x \rho(x)\log \sum_{x,y}\exp\left(\theta^\top\psi(x,y)\right)~.
    \end{align*}
    is $1$-Lipschitz and $2$-smooth. Moreover, if the features are sampled from a $0$-mean distribution and span $R^{d_P}$, then there exists $\kappa >0$, such that $A(\theta)$ is $\kappa$-strongly convex.
\end{lemma}
\begin{proof}
    Let $\theta\in\mathbb{R}^{d_P}$. Note that
    \begin{align*}
        \nabla_\theta A(\theta) & = \sum_x \rho(x)\frac{\sum_{y}\exp(\psi(x,y)^\top\theta)}{\sum_{y'}\exp(\psi(x,y')^\top\theta)}\psi(x,y')\\
        & = \sum_x \rho(x)\sum_{y}\pi_\theta(y|x) \psi(x,y)\\
        & \leq \max_{x,y}\norm{\psi(x,y)}_2 \\
        &\leq 1~.
    \end{align*}
    On the other hand, the Hessian of $A(\theta)$ is
    \begin{align*}
        \nabla^2_\theta A(\theta) & = \sum_x \rho(x)\sum_{y} \nabla_\theta \pi_\theta(y|x) \psi(x,y) \\
        & = \sum_x \rho(x) \sum_{y}\pi_\theta(y|x) \left( \psi(x,y) - \expp_{y'\sim \pi_\theta(\cdot|x)}[\psi(x,y')]\right)\psi(x,y)^\top \\
        & = \expp_{x\sim\rho, y\sim \pi_\theta(\cdot|x)}\left[ \psi(x,y)\psi(x,y)^\top \right]  - \expp_{x\sim,\rho, y\sim \pi_\theta(\cdot|x)}[\psi(x,y)] \expp_{x\sim\rho, y\sim \pi_\theta(\cdot|x)}[\psi(x,y)]^\top \\
        & = \expp_{x\sim\rho, y\sim \pi_\theta(y|x)} \left[ \left(\psi(x,y) - \expp_\theta\left[\psi(x,y)\right]\right)\left(\psi(x,y) - \expp_\theta\left[\psi(x,y)\right]\right)^\top\right]~.
    \end{align*}
    By assumption on the feature mapping, we have that
    \begin{align*}
        \norm{\nabla^2_\theta A(\theta)}_2 & \leq \max_{x,y}
         \norm{\left(\psi(x,y) - \expp_\theta\left[\psi(x,y)\right]\right)\left(\psi(x,y) - \expp_\theta\left[\psi(x,y)\right]\right)^\top}_2\\
        & \leq \max_{x,y}\norm{\psi(x,y)-\expp_\theta[\psi(x,y)]}_2\\
        & \leq 2\max_{x,y}\norm{\psi(x,y)}_2 = 2~.
    \end{align*}
    Therefore, the function $A(\theta)$ is $2$-smooth in $\theta$.
    For strong convexity, let $\psi$ be sampled from a $0$-mean bounded distribution. Note that, for any non-zero vector in $\mathbb{R}^{d_P}$, we have
    \begin{align*}
        z^\top \nabla^2_\theta A(\theta)z & = \expp_{x\sim\rho, y\sim \pi_\theta(\cdot|x)}\left[ z^\top\psi(x,y)\psi(x,y)^\top z \right]\\
            & \geq \min_{\theta,x, y} \pi_\theta(\cdot|x) \sum_{x,y} (\psi(x,y)^\top z)^2  \\
            & \geq C_3 \sum_{x,y} (\psi(x,y)^\top z)^2~,
    \end{align*}
    for a positive $C_3$, since $\pi_\theta$ is in the loglinear class, for every $\theta$, and using Lemma \ref{lem:lower_bound_pi}. Now, note that, if $z$ can be expressed as a linear combination of $\{ \psi(x,y)\}_{x,y}$, the summation cannot be zero for non-zero $z$. Thus, if $\{ \psi(x,y)\}_{x,y}$ spans $\mathbb{R}^{d_P}$, that is, the feature matrix is full rank, then there exists an absolute positive constant $\kappa$, such that we have
    \begin{align*}
        \norm{\nabla^2_\theta A(\theta)}_2 \geq \kappa > 0~.
    \end{align*}
    Thus, the function $A(\theta)$ is $\kappa$-strongly convex.
\end{proof}

\begin{lemma}\label{lem:tabular_dpo}
    In general, the norms of the gradient and Hessian for the loss of tabular DPO are unbounded from above. 
\end{lemma}
\begin{proof}
    Observe that, given policy $\pi$ and $(x,y^w)\in\calD$, we have
    \begin{align*}
        \nabla_{\pi(y^w|x)}\calL_\calD(\pi) = \frac{\beta}{n}\left( 1- \sigma\left(\beta\log\frac{\pi(y^w|x)}{\mu(y^w|x)} - \beta\log\frac{\pi(y^l|x)}{\mu(y^l|x)}\right)\right) \frac{1}{\pi(y^w|x)}~.
    \end{align*}
    On the other hand, for the second derivative with respect to $\pi(y^w|x)$, we have the following. First, let 
    \begin{align*}
        f(\pi(y^w|x)) = \beta\log\frac{\pi(y^w|x)}{\mu(y^w|x)} - \beta\log\frac{\pi(y^l|x)}{\mu(y^l|x)}~.
    \end{align*}
    We have that $\nabla_{\pi(y^w|x)}f(\pi(y^w|x)) = \beta / \pi(y^w|x)$. Now, observe that
    \begin{align*}
        & \nabla^2_{\pi(y^w|x)}\calL_\calD(\pi) = \frac{\beta}{n}\nabla_{\pi(y^w|x)} \frac{\exp(f(\pi(y^w|x)))}{\pi(y^w|x)(1+\exp(f(\pi(y^w|x))))}\\
            & = \frac{\beta}{n}\left( \frac{\frac{\beta}{\pi(y^w|x)}\exp\left( f(\pi(y^w|x))\right) \pi(y^w|x)(1+\exp(f(\pi(y^w|x))))}{\left( \pi(y^w|x)(1+\exp(f(\pi(y^w|x))))\right)^2}\right) \\ &\quad \quad  -\frac{\beta}{n}\left(\frac{\exp\left( f(\pi(y^w|x))\right) \left( (1+\exp(f(\pi(y^w|x)))) + \pi(y^w|x) \frac{\beta}{\pi(y^w|x)}\exp\left( f(\pi(y^w|x))\right)\right)}{\left( \pi(y^w|x)(1+\exp(f(\pi(y^w|x))))\right)^2}\right)\\
        & = \frac{\beta\left( (\beta-1)\exp(f(\pi(y^w|x)))(1+\exp(f(\pi(y^w|x)))) -\beta \exp(f(\pi(y^w|x)))^2\right) }{n \left( \pi(y^w|x)(1+\exp(f(\pi(y^w|x))))\right)^2}~.
    \end{align*}
    The above numerator is not always non-negative, as solving for $\exp(f(\pi(y^w|x)))$ will show. Moreover, neither the norm of the gradient nor the operator norm of the Hessian can be upper-bounded in general, due to the presence of $\pi(y^w|x)$ in the denominator. 
\end{proof}

\begin{theorem}[Theorem 1.(c) of \cite{shah2016estimation}]\label{thm:btl_mle_minimax}
    For the BT preference model, $B$-bounded weight vector and sample size $n\geq O\left( tr(\Sigma^\dagger)/\beta^2 B^2\right)$, where $\Sigma$ denotes the Laplacian with respect to features, the maximum likelihood estimator satisfies the minimax bounds 
    \begin{align*}
        \Omega\left( \frac{d}{\beta n}\right) \leq \norm{\widetilde{\theta}-\theta^*}^2_{\Sigma} \leq O\left( \frac{ d}{\beta n}  \right)~.
    \end{align*}
\end{theorem}

\begin{lemma}[Lemma 27 of \cite{mei2020on}]\label{lem:kl_and_parameter_lemma}
    Let $\pi_\theta =softmax(\Psi\theta)$ and $\pi_{\theta'}=softmax(\Psi\theta')$. Then, for any constant $c$, we have
    \begin{align*}
        D_\textnormal{KL}(\pi_\theta ||\pi_{\theta'}) \leq \frac{1}{2}\norm{\Psi\theta -\Psi\theta' - c^\top\mathbf{1}}^2~.
    \end{align*}
\end{lemma}

\end{document}